\documentclass{article}

\usepackage[preprint]{neurips_2025}

\usepackage{pgfplots}
\usetikzlibrary{external}
\usepackage{bigints}
\usepackage[utf8]{inputenc} 
\usepackage[T1]{fontenc}    
\usepackage{hyperref}       
\usepackage{url}            
\usepackage{booktabs}       
\usepackage{amsfonts}       
\usepackage{algorithm} 
\usepackage{algorithmic}
\usepackage{nicefrac}  
\usepackage{microtype}      
\usepackage{xcolor}         
\usepackage{amsmath}
\usepackage{amssymb}
\usepackage{mathtools}
\usepackage{amsthm}
\usepackage{subfig}
\usepackage{wrapfig}
\usepackage[square,numbers]{natbib}
\usepackage{mymacros}
\usepackage{tikz}
\usetikzlibrary{arrows.meta} 
\usepackage[capitalize,noabbrev]{cleveref}

\usepackage{thmtools} 
\usepackage{thm-restate}
\usepackage{enumitem}

\pgfplotsset{compat=newest}
\usepgfplotslibrary{groupplots}
\usepackage{comment}

\usepackage{enumitem}

\title{Thompson Sampling-like Algorithms\\ for Stochastic Rising Bandits}

\author{%
  Marco Fiandri 
  \qquad 
  Alberto Maria Metelli
  \qquad
  Francesco Trovò
  \\
  DEIB, Politecnico di Milano \\
  Piazza Leonardo Da Vinci, 32, Milan, Italy \\ \texttt{marco.fiandri@mail.polimi.it} \\ \texttt{\{albertomaria.metelli,francesco1.trovo\}@polimi.it}
}

\allowdisplaybreaks[4]

\begin{document}
\setlength{\abovedisplayskip}{4pt}
\setlength{\belowdisplayskip}{4pt}
\setlength{\textfloatsep}{12pt}

\maketitle

\begin{abstract}
 \emph{Stochastic rising rested bandit} (SRRB) is a setting where the arms' expected rewards increase as they are pulled. It models scenarios in which the performances of the different options grow as an effect of an underlying learning process (e.g., online model selection). Even if the bandit literature provides specifically crafted algorithms based on upper-confidence bounds for such a setting, no study about \emph{Thompson sampling} (TS)-like algorithms has been performed so far. The strong regularity of the expected rewards in the SRRB setting suggests that specific instances may be tackled effectively using adapted and sliding-window TS approaches. This work provides novel regret analyses for such algorithms in SRRBs, highlighting the challenges and providing new technical tools of independent interest. Our results allow us to identify under which assumptions TS-like algorithms succeed in achieving sublinear regret and which properties of the environment govern the complexity of the regret minimization problem when approached with TS.
Furthermore, we provide a regret lower bound based on a complexity index we introduce. Finally, we conduct numerical simulations comparing TS-like algorithms with state-of-the-art approaches for SRRBs in synthetic and real-world settings.
\end{abstract}

\section{Introduction}

Traditional \emph{multi-armed bandits} (MABs) models~\citet{bubeck2012regret,lattimore2020bandit} consider \emph{static} environments, assuming arms with expected rewards that do not change during the learning process. However, many real-world applications present a more challenging scenario where the rewards associated with each arm \emph{dynamically} evolve depending on time and/or on the number of times an arm has been played.
In particular, two scenarios have been analyzed in the literature: \emph{restless} and \emph{rested} bandits.
While the restless MAB setting assumes that the arms' expected reward changes as an effect of {nature}, in the \emph{rested} MAB setting~\cite{Tekin2012rested}, the arms' evolution is triggered by their pull.
A large variety of \emph{restless} settings has been analyzed in the past under the name of  \emph{non-stationary} bandits, the most studied being the abruptly changing~\cite{garivier2011upper} and smoothly changing~\citep{trovo2020sliding}.
In contrast, the \emph{rested} setting has raised the attention of the bandit community more recently~\citep{heidari2016tight,seznec2019rotting,metelli2022stochastic}.

This paper focuses on the \emph{stochastic rising rested bandits}~\citep[SRRBs,][]{metelli2022stochastic} which reflect situations where the arms' expected rewards \emph{increase}, modeling growing trends (e.g., learning processes).
Many real-world challenges can be modeled as  SRRBs, examples of such settings are represented by the so-called \emph{combined algorithm selection and hyperparameter optimization}~\citep[CASH,][]{thornton2013auto,li2020efficient}, whose goal is to identify the best learning algorithm and the best hyperparameter configuration for a given machine learning task, representing one of the most fundamental problems in \emph{automatic machine learning}~\citep[AutoML,][]{waring2020automated}. As a further regularity condition, \citet{heidari2016tight} and \citet{metelli2022stochastic} consider \emph{concave} evolution of the expected rewards that are suitable to model satiation effects in recommendations~\citep{clerici2023linear,xu2023online}. Problems with similar characteristics arise when selecting an algorithm from a predefined set to optimize a given stochastic function, i.e., \emph{online model selection}~\citep{metelli2022stochastic}.

The seminal work by~\citet{metelli2022stochastic} approached the regret minimization problem in SRRBs by designing a sliding-window algorithm based on upper-confidence bounds able to provide a worst-case regret of the order of $\widetilde{O}(T^{\frac{2}{3}} + \Upsilon(T))$,  where $T$ is the learning horizon and $\Upsilon(T)$ a problem-dependent quantity which characterizes the growth rate of the arm expected rewards.\footnote{With the $\widetilde{O}(\cdot)$ notation we disregard logarithmic factors w.r.t.~the learning horizon $T$.}
However, to the best of our knowledge, there has been no analysis to investigate whether \emph{Thompson sampling}-like algorithms~\citep[TS,][]{kaufmann2012thompson,agrawal2017near} can provide good regret guarantees in the SRRB scenario. The interest in studying TS-like algorithms lies in the fact that they are the most widely used bandit algorithms in real-world applications, both easy to implement and with good empirical performance \cite{scott2010ts,chapelle2011empirical}. This paper aims to address the following three questions: ($i)$ What modifications should be introduced to the TS algorithms to tackle the SRRB setting? ($ii$) Which properties of the environment govern their performance? ($iii$) What assumptions are necessary to ensure the no-regret property?

\textbf{Original Contributions.}~~The contributions of the paper are summarized as follows.
\begin{itemize}[leftmargin=*,noitemsep,topsep=-2pt]
    \item In Section~\ref{sec:problem}, we redefine the notion of suboptimality gap $\overline{\Delta}(n,n')$ (Equation~\ref{eq:deltas}) for the SRRB scenario. Based on it, we introduce a novel complexity index $\sigma_{\bm{\mu}}(T)$ (Equation~\ref{eq:sigmaDef}) representing the number of pulls of the optimal arm needed to allow the TS-like algorithm to detect its optimality.
    \item In Section~\ref{sec:algs}, we present and describe our TS-like algorithms, characterized by a sliding-window with size $\tau$, preceded by an optional forced exploration phase in which each arm is pulled $\Gamma$ times. They are instanced with Beta priors, \texttt{ET-Beta-SWTS} (Explore Then-Beta-Sliding Window Thompson Sampling, Algorithm~\ref{alg:betats}), and Gaussian priors, \texttt{$\gamma$-ET-SWGTS} (Explore Then-Gaussian Sliding Window Thompson Sampling, Algorithm~\ref{alg:gts}).
    \item In Section~\ref{sec:preliminari}, to face the additional complexities given by a dynamic scenario, we introduce novel tools to conduct the \emph{frequentist} regret analysis of TS-like algorithms that can be of independent interest. Specifically, we generalize the decomposition of the expected number of pulls provided by Theorem 36.2 of \citet{lattimore2020bandit} to \emph{general dynamic environments and sliding-window TS-like algorithms} (Lemma~\ref{lem:RD}) and we the problem of \emph{the underestimation of the optimal arm in the presence of Bernoulli rewards} with completely novel techniques based on log-convexity and stochastic dominance (Lemma \ref{lemma:techlemma}). 
   \item In Section~\ref{sec:AnalysisBetaTS}, we provide {frequentist} regret bounds for our TS-like algorithms with no sliding window (i.e., $\tau=T$) for both Beta and Gaussian priors (Theorems~\ref{thm:ts} and \ref{thm:gts}) depending on our novel index of complexity $\sigma_{\bm{\mu}}(T)$ and on the total variation divergence between specifically defined distributions characterizing the instances. Our bounds are tight for the case of stationary bandits, and, differently from what has been done in \citet{agrawal2017near}, they apply to distributions with expected rewards that are not constrained between $[0,1]$. 
   We analyze the sliding-window approaches (i.e., $\tau<T$) providing analogous frequentist regret bounds in Appendix~\ref{sec:SWapproach}.
   {Additionally, we conceive a lower bound on the regret dependent on the complexity index $\sigma_{\bm{\mu}}(T)$ we have introduced, and we compare it with the upper bounds showing the tightness on this term.}
    \item In Section~\ref{sec:Experiments}, we validate the theoretical findings with numerical simulations to compare the performances of the proposed algorithms w.r.t.~the ones designed for the SRRB setting in terms of cumulative regret on both synthetic scenarios and with real-world data.
\end{itemize}
The proofs of the results are reported in Appendix~\ref{apx:proofs}.

\section{Related Literature}
We revise the related literature with a focus on TS algorithms, restless, rising, and rotting bandits.

\textbf{TS Algorithms.}~~\emph{Thompson sampling} has a long history, originally designed as a heuristic for sequential decision-making~\citep{thompson1933ts}. Only in the past decade has it been analyzed~\cite{kaufmann2012thompson,agrawal2017near}, providing finite-time logarithmic \emph{frequentist} regret in the stochastic stationary bandit setting. Although the original frequentist analyses are rather involved, general approaches for conducting such analyses are now available with a more concrete understanding of the intimate functioning of TS methods~\cite{lattimore2020bandit,baudry2023general}.


\textbf{Restless Bandits.}~~
Algorithms designed for stationary bandits (including {UCB} and {TS}) have often been extended to \emph{non-stationary} (or \emph{restless}) bandits. The main idea is to forget past observations. Two approaches are available: passive and active. The former iteratively discards the information coming from the far past (either with a \emph{sliding window} or a \emph{discounted} estimator). Examples of such a family of algorithms are \texttt{DUCB}~\citep{garivier2011upper}, \texttt{Discounted TS}~\citep{raj2017taming,qi2023discounted}, \texttt{SW-UCB}~\citep{garivier2011upper}, and \texttt{SW-TS}~\citep{trovo2020sliding}. Instead, the latter class uses \emph{change-detection} techniques~\cite{basseville1993detection} to decide when it is the case to discard old samples. This occurs when a sufficiently evident change affects the arms' expected rewards. Among the active approaches we mention \texttt{CUSUM-UCB}~\citep{liu2018change}, \texttt{REXP3}~\citep{besbes2014stochastic}, \texttt{GLR-klUCB}~\citep{besson2019generalized}, and \texttt{BR-MAB}~\citep{re2021exploiting}. 

\textbf{Rising Bandits.}~~
The regret minimization problem of SRRB in the deterministic case has been proposed by~\citet{heidari2016tight}. They designed algorithms with provably optimal policy regret bounds, and, when the expected rewards are increasing and concave (i.e., rising), they provide an algorithm with sublinear regret.
The stochastic version of SRRB has been studied by~\citep{metelli2022stochastic} where the authors provide regret bounds of order $\widetilde{O}(T^{\frac{2}{3}}+\Upsilon(T))$, where $\Upsilon(T)$ is a problem-dependent quantity which characterizes the growth rate of the arm expected rewards. This bound is attained by specifically designed optimistic algorithms that make use of a sliding window approach, namely \texttt{R-ed-UCB} and \texttt{R-less-UCB}. More recently, approaches presented in \citet{metelli2022stochastic} have been extended to the novel graph-triggered setting~\cite{GenaltiM0RCM24}, interpolating between the rested and restless cases. Furthermore, \citet{amichay2025rising} provide specific algorithms for the case in which the growing trend is linear. The recent preprint by \citet{fiandri2024rising} has provided the first lower bounds depending on the complexity term $\Upsilon(T)$. The SRRB setting has also been studied in a BAI framework by~\citet{mussi2023best}, where the authors propose the \texttt{R-UCBE} and \texttt{R-SR} algorithm, a UCB-inspired and successive elimination approaches, respectively, that provide guarantees for the fixed budget case. Similar results have been obtained in \citet{CellaPG21} and \citet{TakemoriUG24} under a more restrictive parametric model for the expected reward functions. A comparison between the guarantees of the algorithms presented in this paper and the results from \citet{metelli2022stochastic} is provided in Appendix~\ref{sec:Comparison}.

\textbf{Rotting Bandits.}~~Another setting related to SRRB is the rotting bandits~\citep{seznec2019rotting,levine2017rotting}, where the expected rewards decrease either over time or depending on the arm pulls. The authors propose specifically crafted algorithms to address rested and restless rotting bandits, obtaining sublinear regret guarantees. However, as highlighted in~\citet{metelli2022stochastic}, these algorithms cannot be applied to the SRRB setting.

\section{Problem Formulation}\label{sec:problem}

We consider an SRRB instance with stochastic rewards. Let $K \in \Nat$ be the number of arms. Every arm $i \in \dsb{K}$\footnote{For two integers $a, b \in \mathbb{N}, a < b$, we denote with $\dsb{a}$ the set $\{1, \ldots, a\}$ and $\dsb{a,b}$ the set $\{a, \ldots, b\}$.} is associated with an expected reward $\mu_i : \Nat \rightarrow \mathbb{R}$, where $\mu_i(n)$ defines the expected reward of arm $i$ when pulled for the $n$-th time with $n \in \Nat$. {We denote an SRRB instance as $\bm{\mu} = (\mu_i)_{i \in \dsb{K}}$.} In a rising bandit, the expected reward function $\mu_i(n)$ is non-decreasing.\footnote{Differently from previous works~\cite{metelli2022stochastic,mussi2023best,CellaPG21}, we will not enforce \emph{concavity} since our algorithms will enjoy regret guarantees depending on a novel complexity index (see Equation~\ref{eq:sigmaDef}), allowing for sublinear instance-dependent regret guarantees even for non-concave expected reward functions.}

\begin{ass}[Non-Decreasing]\label{ass:nonDecreasing}
For every arm $i \in \dsb{K}$ and number of pulls $n \in \Nat$, let $\gamma_i(n) \coloneqq \mu_i(n+1) - \mu_i(n)$ be the increment function, it holds that: $\gamma_i(n) \ge 0$.
\end{ass}



The learning process occurs over $T \in \Nat$ rounds, where $T$ is the learning horizon. At every round $t \in \dsb{T}$, the agent pulls an arm $I_t \in \dsb{K}$ and observes a random reward $X_{I_t,t} \sim \nu_{I_t}(N_{I_t,t})$, where for every arm $i \in \dsb{K}$, we have that $\nu_i(N_{i,t})$ is a probability distribution
depending on the current number of pulls up to round $t$, i.e., $N_{i,t} \coloneqq \sum_{l=1}^t \mathds{1}\{I_l = i\}$ whose expected value is given by $\mu_i(N_{i,t})$. For every arm $i \in \dsb{K}$ and round $t \in\dsb{T}$, we define the \emph{average expected reward} as: $
    \overline{\mu}_{i}(t) \coloneqq \frac{1}{t} \sum_{l = 1}^t \mu_i(l)
$.
The optimal policy constantly plays the arm with the maximum average expected reward at the end of the learning horizon $T$. We denote with $i^*(T) \coloneqq \argmax_{i \in \dsb{K}} \overline{\mu}_{i}(T)$ the optimal arm (highlighting the dependence on $T$), that we assume w.l.o.g. unique~\cite{heidari2016tight}.

\textbf{Suboptimality Gaps.}~~For analysis purposes, we introduce, for every suboptimal arm $i \neq i^*(T)$ and every number of pulls $n, n' \in \dsb{T}$, the following \emph{novel} notions of suboptimality gaps defined in terms of the expected reward $\Delta_i(n,n')$ and of average expected reward $\overline{\Delta}_i(n,n')$, formally:
\begin{align}
 \Delta_i(n,n') \coloneqq \max\{0,\mu_{i^*(T)}(n) - \mu_i(n')\}, \qquad \overline{\Delta}_i(n,n') \coloneqq \max\{0,\overline{\mu}_{i^*(T)}(n) - \overline{\mu}_i(n')\}.\label{eq:deltas}
\end{align}
Notice that, for specific rounds, the gaps $\overline{\Delta}_i(n,n')$ may be negative, however, at the end of the learning horizon $T$, we are guaranteed that $\overline{\mu}_{i^*(T)}(T) > \overline{\mu}_i(n')$ for all $n' \in \dsb{T}$.
This property allows defining, for every suboptimal arm $i \neq i^*(T)$, the minimum number of pulls $\sigma_i(T)$ of the optimal arm needed so that the average expected reward suboptimality gap of arm $i$ is positive and its maximum $\sigma_{\bm{\mu}}(T)$ over the suboptimal arms, formally:
\begin{align}\label{eq:sigmaDef} 
    \sigma_i(T) \coloneqq \min \{ l \in \dsb{T} : \overline{\mu}_{i^*(T)}(l) > \overline{\mu}_i(T)\}, \qquad \sigma_{\bm{\mu}}(T) \coloneqq \max_{i \neq  i^*(T)} \sigma_i(T). 
\end{align}
After $\sigma_{\bm{\mu}}(T)$ pulls of $i^*(T)$, we can identify $i^*(T)$ as the optimal arm, treating the problem as a stationary bandit since no suboptimal arm can exceed the performance of $i^*(T)$ beyond this point.

\textbf{Regret.}~~Given a SRRB instance $\bm{\mu}$, the goal of a regret minimization algorithm $\mathfrak{A}$ in an SRRB is to minimize the \emph{expected cumulative regret}~\cite{heidari2016tight,metelli2022stochastic} defined as follows:
\begin{align}
    R_{\bm{\mu}}(\mathfrak{A},T) \coloneqq T \ \overline{\mu}_{ i^*(T)}(T) -  \E \bigg[ \sum_{t = 1}^T \mu_{I_t}(N_{I_t,t}) \bigg],
\end{align}
where the expectation is w.r.t.~the randomness of the rewards and the possible randomness of $\mathfrak{A}$.

\textbf{Reward Distribution.}~~
We perform the regret analysis under the following reward distributions.
\begin{ass}[Bernoulli Rewards]\label{ass:bernoulli}
    For every $t \in \dsb{T}$, $n \in \dsb{T}$ and $i \in \dsb{K}$, the reward $X_{i,t} \sim \nu_i(n)$ is Bernoulli distributed.
\end{ass}
\begin{ass}[Non-negative Subgaussian rewards\protect\footnote{For the sake of presentation, we assume positive support. W.l.o.g.~one might also consider realizations bounded from below by a given constant (see Section~\ref{sec:AnalysisBetaTS} for details).}]\label{ass:nonNeg}
	For every $t \in \dsb{T}$, $n \in \dsb{T}$, and $i \in \dsb{K}$, the reward $X_{i,t} \sim \nu_i(n)$ is: ($i$) non-negative almost surely; ($ii$) $\lambda^2$-subgaussian.\footnote{A zero-mean random variable $X$ is $\lambda^2$-subgaussian if for every $s \in \mathbb{R}$ it holds that $\E[e^{sX}] \le e^{s^2\lambda^2/2}$.}
\end{ass}
\section{Algorithms}\label{sec:algs}

In this section, we introduce \texttt{ET-Beta-SWTS} ({Explore Then-Beta-Sliding Window Thompson Sampling}, Algorithm~\ref{alg:betats}) and \texttt{$\gamma$-ET-SWGTS} ({$\gamma$-Explore Then-Sliding Window Gaussian Thompson Sampling}, Algorithm~\ref{alg:gts}), TS-like algorithms designed to deal with Bernoulli (Assumption~\ref{ass:bernoulli}) and subgaussian non-negative (Assumption~\ref{ass:nonNeg}) rewards, respectively. For the sake of generality, the algorithms are presented in their \emph{sliding window} versions, with a window of size $\tau$, since the non-windowed versions can be retrieved by setting $\tau=T$.

In the first $K\Gamma$ rounds, both algorithms perform a \emph{forced exploration}. They play each arm $\Gamma$ times, collecting $S_{i,K\Gamma+1,\tau}$ and $N_{i,K\Gamma+1,\tau}$, where $N_{i,t,\tau} \coloneqq \sum_{
s = \max{\{t-\tau,1}\}}^{t-1} \mathds{1}{\{I_s = i\}}$, is the number of times arm $i \in \dsb{K}$ has been selected in the last $\min{ \{t, \tau \}}$ rounds with $t \in \dsb{T}$, and $S_{i,t,\tau} \coloneqq \sum_{ s = \max\{{t-\tau,1}\}}^{t-1} X_{i,s} \mathds{1}{\{I_s = i\}}$ is the cumulative reward collected by arm $i$ in the last $\min{\{t, \tau}\}$ rounds. 
Then, in the subsequent rounds, both algorithms run a TS-like routine. In particular, for \texttt{ET-Beta-SWTS}, in each round $t\in \dsb{K\Gamma+1,T}$, the posterior distribution of the expected reward of arm $i$ at round $t$ is a \emph{Beta distribution} defined as $\eta_{i,t} \coloneqq \text{Beta}(S_{i,t,\tau} +
1, N_{i,t,\tau} - S_{i,t,\tau} + 1)$, where $ \text{Beta}(\alpha,\beta)$ denotes a Beta distribution with parameters $\alpha>0$ and $\beta>0$. Once the distributions $\eta_{i,t}$ have been computed, for each arm $i \in \dsb{K}$, the algorithm draws a sample $\theta_{i,t,\tau}$ from $\eta_{i,t}$, a.k.a.~Thompson sample, and plays the arm whose sample has the highest sample value. Then, the values for $S_{i,t+1,\tau}$ and $N_{i,t+1,\tau}$ are updated and, therefore also the posterior distribution $\eta_{i,t+1}$. 

\texttt{$\gamma$-ET-SWGTS} shares the principles of \texttt{ET-Beta-SWTS} with some differences. In particular, after the forced exploration, \texttt{$\gamma$-ET-SWGTS} at every round $t\in\dsb{K\Gamma+1,T}$, the posterior distribution is a \emph{Gaussian distribution} computed as $\eta_{i,t} \coloneqq \mathcal{N} \left( \frac{S_{i,t,\tau}}{N_{i,t,\tau}},\frac{1}{\gamma N_{i,t,\tau}} \right)$, where $\mathcal{N}(m, s^2)$ is a Gaussian distribution with mean $m$ and variance $s^2$. Notice how the parameter $\gamma >0$, whose value will be specified later, scales the variance of the posterior. If there exists an arm $i \in \dsb{K}$ from which we do not have samples, i.e., $N_{i, t, \tau} = 0$, the algorithm pulls it, so that the posterior distribution is always well defined. Otherwise, analogously to \texttt{ET-Beta-SWTS}, the algorithm draws a random sample $\theta_{i,t,\tau}$ from $\eta_{i,t}$ and plays the arm whose Thompson sample is the largest. Finally, the algorithm updates $S_{i,t+1,\tau}$, $N_{i,t+1,\tau}$, and  $\eta_{i,t+1}$.
It is worth noting that by removing the forced exploration, i.e., setting $\Gamma=0$, from \texttt{ET-Beta-SWTS}, we recover \texttt{Beta-SWTS} introduced by \citet{trovo2020sliding}; while by setting $\Gamma=1$ in \texttt{$\gamma$-ET-SWGTS}, we retrieve a new algorithm, that is a natural extension of \texttt{Beta-SWTS} in settings with subgaussian rewards, that we call \texttt{$\gamma$-SWGTS}.\footnote{In Appendix \ref{apx:cc}, we show that the presence of a sliding window does not increase the per-round computational complexity which remains $O(K)$ as in standard TS algorithms with no window.}

\begin{figure}[t]
\begin{minipage}{0.4\textwidth}
\begin{algorithm}[H]
\caption{\texttt{ET-Beta-SWTS}.} \label{alg:betats}
\small
{
\begin{algorithmic}[1]
    \STATE \textbf{Input:} Number of arms $K$, time horizon $T$, time window $\tau$, forced exploration $\Gamma$
    \STATE Pull each arm $\Gamma$ times
    \STATE Set $\eta_{i,K\Gamma+1} \gets \text{Beta}(S_{i,K\Gamma+1,\tau}+1,N_{i,K\Gamma+1,\tau}-S_{i,K\Gamma+1,\tau}+1)$ for each $i \in \dsb{K}$
    \FOR{$t \in \dsb{K\Gamma+1,T}$}
        \STATE {\thinmuskip=1mu
  \medmuskip=1mu \thickmuskip=1mu Sample $\theta_{i,t,\tau} \sim \eta_{i,t}$ for each $i \in \dsb{K}$ }\label{line:sample2}
        \STATE Select $I_t \in \arg \max_{i \in \dsb{K}} \theta_{i,t,\tau}$ \label{line:selectionts2}
        \STATE Pull arm $I_t$ and collect reward $X_{I_t,t}$
        \STATE Update $S_{i,t+1,\tau}$ and $N_{i,t+1,\tau}$ for each $i \in \dsb{K}$
        \STATE Update $\eta_{i,t+1} \gets \text{Beta}(1+S_{i,t+1,\tau},1+(N_{i,t+1,\tau}-S_{i,t+1,\tau}))$ for each $i \in \dsb{K}$ \label{line:updatets2}
    \ENDFOR
\end{algorithmic}
}
\end{algorithm}
\end{minipage}%
\hfill
\begin{minipage}{0.59\textwidth}
\begin{algorithm}[H]
\caption{\texttt{$\gamma$-ET-SWGTS}.} \label{alg:gts}
\small
{\thinmuskip=1mu
  \medmuskip=1mu \thickmuskip=1mu
\begin{algorithmic}[1]
 \STATE \textbf{Input:} Number of arms $K$, time horizon $T$, parameter $\gamma$, time window $\tau$, forced exploration $\Gamma$
    \STATE Pull each arm $\Gamma$ times
    \STATE Set $\eta_{i,K\Gamma+1} \gets \mathcal{N}\big(\frac{S_{i,K\Gamma+1,\tau}}{N_{i,K\Gamma+1,\tau}},\frac{1}{\gamma N_{i,K\Gamma+1,\tau}}\big)$ for each $i \in \dsb{K}$ \label{line:prior1}
    \FOR{$t \in \dsb{K\Gamma+1,T}$}
         \IF{ $\exists i \in \dsb{K}$ such that $N_{i,t,\tau}=0$}
            \STATE Select $I_t$ as any arm $i$ such that $N_{i,t,\tau}=0$ \label{line:forI0}
        \ELSE
            \STATE Sample $\theta_{i,t, \tau} \sim \eta_{i,t}$ for each $i \in \dsb{K}$ \label{line:sample11}
            \STATE Select $I_t \in \arg \max_{i \in \dsb{K}} \theta_{i,t,\tau}$ \label{line:selectionts11}
        \ENDIF
        \STATE Pull arm $I_t$ and collect reward $X_{I_t,t}$
            \STATE Update  $S_{i,t+1,\tau}$ and $N_{i,t+1,\tau}$  for each $i \in \dsb{K}$ 
            \STATE Update $\eta_{i,t+1} \gets \mathcal{N}\label{line:updatets11}\big( \frac{S_{i,t+1,\tau}}{N_{i,t+1,\tau}},\frac{1}{\gamma N_{i,t+1,\tau}} \big) $ for each $i \in \dsb{K} $
    \ENDFOR
    \vspace{.07cm}
\end{algorithmic}
}
\end{algorithm}
\end{minipage}
\end{figure}

\section{Challenges and New Technical Tools} \label{sec:preliminari}
The complexity of the scenario where the expected rewards vary throughout the learning process presents additional technical challenges and necessitates the development of new tools. In this section, we introduce them highlighting their generality beyond the scope of this paper. To appreciate the content of this section, the interested reader is advised to have a general understanding of the \emph{frequentist} TS analysis of \citet{agrawal2017near}.
The reader not interested in these technical aspects can freely skip this section and proceed to Section~\ref{sec:AnalysisBetaTS}, which presents the regret bounds.

Let us start by introducing the probability that a Thompson sample is larger than a given threshold.
\begin{restatable}[]{definition}{probs}\label{def:prob}
Let $i,i' \in \dsb{K}$ be two arms, $t \in \dsb{T}$ be a round, $\tau \in \dsb{T}$ be the window, and $y_{i',t}\in (0,1)$ be a threshold, we define: $
    p_{i,t,\tau}^{i'} \coloneqq \mathbb{P}\left(\theta_{i,t,\tau}>y_{i',t}\right|\mathcal{F}_{t-1})$,
    where $\mathcal{F}_t$ is the filtration induced by the sequence of arms played and observed rewards up to round $t$.
\end{restatable}
The TS analysis of \citet{agrawal2017near} is based on bounding the expected number of pulls $\E_{\bm{\mu}}[N_{i,T}]$ of every arm $i \in \dsb{K}$ through a decomposition that has been formalized in Theorem 36.2 of \citet{lattimore2020bandit} for stationary bandits. We now generalize it for a generic dynamic setting (including restless and rested) and for sliding window approaches. To this end, let $i^*(t) \in \dsb{K}$ be the optimal arm at time $t$,\footnote{For the SRRB setting, $i^*(t) = i^*(T)$ for every $t \in \dsb{T}$.} for the \texttt{ET-Beta-SWTS} algorithm, the following holds.
\begin{restatable}[Expected Number of Pulls Bound for \texttt{ET-Beta-SWTS}]{lemma}{rd} \label{lem:RD}
Let $T \in \mathbb{N}$ be the learning horizon, $\tau \in \dsb{T}$ the window size, $\Gamma \in \dsb{0,T}$ the forced exploration parameter, for the \texttt{ET-Beta-SWTS} algorithm it holds for every $i \neq i^*(t)$ and free parameter $\omega \in \dsb{0,T}$ that:
    \begin{align*}
        \mathbb{E}_{\bm{\mu}}  [N_{i,T}] & \leq 1 +\Gamma+\frac{\omega T}{\tau} +\mathbb{E}_{\bm{\mu}} \bigg[ \sum_{t=K\Gamma+1}^{T} \mathds{1} \bigg\{ p_{i,t,\tau}^i > \frac{1}{T}, \ I_t=i, \ N_{i,t,\tau} \ge \omega \bigg\} \bigg] \\ &\qquad +\mathbb{E}_{\bm{\mu}} \bigg[ \sum_{t=K\Gamma+1}^{T} \bigg( \frac{1}{p_{i^*(t),t,\tau}^i}-1 \bigg) \mathds{1} \left\{ I_t=i^*(t) \right\} \bigg].
    \end{align*}
\end{restatable}
In Appendix \ref{apx:swgts}, we present an analogous result for \texttt{$\gamma$-ET-SWGTS} (Lemma~\ref{lem:RDg}). First, we observe that by setting $\omega=\Gamma=0$ and $\tau=T$, we retrieve Theorem 36.2 of \citet{lattimore2020bandit}. 
As highlighted by \cite{agrawal2012analysis}, \citet{agrawal2017near,kaufmann2012thompson,jin2023thompson}, the main difficulty in the frequentist analysis of TS-like algorithms lies in the evaluation of the error caused by the underestimation of the optimal arm, measured by the term $\mathbb{E}_{\bm{\mu}}[1/{p_{i^*(t),t,\tau}^i}]$.
More recently, \citet{baudry2023general} proved that controlling this term is essential to design optimal TS-like algorithms. Lemma \ref{lem:RD} confirms that this is true even in dynamic settings (including restless and rested). An additional term $\Gamma$ arises from the forced exploration and $\omega > 0$ is a free parameter whose value can be chosen to tighten the bound when a sliding window size $\tau$ is used.

In what follows, we omit the dependence on $\tau$, as our result does not depend on the fact that $N_{i^*(t),t}$ and $S_{i^*(t),t}$ might be constrained within a window.
The central challenge when bounding the underestimation term $\mathbb{E}_{\bm{\mu}}[1/{p_{i^*(t),t}^i}]$ lies in characterizing the \emph{distribution of the parameters of the posterior distribution} of the optimal arm $i^*(t)$, which, in turn, depend on the cumulative reward $S_{i^*(t),t}$. In a stationary bandit,  $S_{i^*(t),t}$ is a sum of $N_{i^*(t),t} = j$ \emph{identically distributed} rewards, but, in a dynamic scenario, rewards are no longer identically distributed. While this does not pose significant issues for Gaussian priors, it prevents us from applying the technique of \citet{agrawal2017near} for Beta priors and Bernoulli rewards. Indeed, such an analysis heavily relies on the fact that $S_{i^*(t),t}$ (i.e., the number of successes) is a \emph{Binomial distribution}, being the sum of identically distributed Bernoulli distributions with parameter $\overline{\mu}_{i^*(t),j}$. Instead, in the dynamic setting, $S_{i^*(t),t}$ is a sum of non-identically distributed Bernoulli distributions (with parameters $\mu_{i^*(t),1},\dots,\mu_{i^*(t),j}$) since the arm expected reward changes throughout the learning process. Instead, it can be proved that $S_{i^*(t),t}$ is distributed as a \emph{Poisson-Binomial}~\citep{wang1993poissobinomial}, whose analytical treatment is far more challenging than the binomial counterpart.
Let us first instance Definition~\eqref{def:prob} for Beta priors: $
    p_{i^*(t),t}^i= \mathbb{P} (\text{Beta} (S_{i^*(t),t}+1, F_{i^*(t),t}+1) > y_{i,t} | \mathcal{F}_{t-1})$, 
where $F_{i^*(t),t} = N_{i^*(t),t} - S_{i^*(t),t}$ (i.e., the number of failures). Our technical innovation consists in Lemma~\ref{lemma:techlemma}, presented below, in which we show that, surprisingly, the worst case in bounding the underestimation term is attained when all rewards are identically distributed, i.e., when $S_{i^*(t),t}$ has a Binomial distribution. 
\begin{restatable}[PB-Bin Stochastic Dominance]{lemma}{mf}\label{lemma:techlemma}
Let $j \in \Nat$, $\mathrm{PB}(\underline{\mu}_{i^*(t)}(j))$ be a Poisson-Binomial distribution with parameters $\underline{\mu}_{i^*(t)}(j) = (\mu_{i^*(t),1},\dots,\mu_{i^*(t),j})$, and $\mathrm{Bin}(j, x)$ be a binomial distribution of $j$ trials and success probability $0 \le x \leq \frac{1}{j} \sum_{l=1}^j \mu_{i^*(t),l} =\overline{\mu}_{i^*(t),j}$. Then, it holds that:
\begin{align*}
	 \E_{S_{i^*(t),t} \sim  \mathrm{PB}(\underline{\mu}_{i^*(t)}(j))}  \bigg[\frac{1}{p_{i^*(t),t}^i} \bigg| N_{i^*(t),t}=j\bigg]  
     &\leq \E_{S_{i^*(t),t}\sim \mathrm{Bin}(j, x)} \bigg[\frac{1}{p_{i^*(t),t}^i} \bigg| N_{i^*(t),t}=j \bigg].
\end{align*}
\end{restatable}
The result is derived by: ($i$) showing the \emph{discrete log-convexity} of $1/p_{i^*(t),t}^i$ as a function of the number of successes~\cite{HOGGAR1974concavity,Johnson_2006,Hill1999AdvancesIS} to pass from the expectation w.r.t.~the Poisson-Binomial to that w.r.t.~the binomial and obtain the first inequality and ($ii$) relying on \emph{stochastic dominance} \cite{boland2002stochastic,boland2004sstochasticordertesting,marshall2011inequalities} arguments to show the monotonicity of the expectation w.r.t.~the parameter of the binomial distribution and obtain the second inequality. 
A similar result is proven in the context of Gaussian priors, resorting, however, to standard arguments based on concentration inequalities (i.e., Chernoff bounds). 
Lemma~\ref{lemma:techlemma} represents our main technical novelty and can be of independent interest.
Indeed, it applies to: ($i$) \emph{any dynamic environment} in which the optimal arm $i^*(t)$ may change across rounds and ($ii$) \emph{sliding window} approaches. Thus, it provides a novel tool for the {frequentist} analysis of TS-like algorithms in dynamic settings even beyond the scope of the paper.
\section{Regret Analysis}\label{sec:AnalysisBetaTS}
We start the section by providing a general regret analysis of \texttt{ET-Beta-TS}  and \texttt{$\gamma$-ET-GTS} algorithms in the SRRBs setting, i.e., the algorithms presented in Section~\ref{sec:algs} with no sliding window, i.e., $\tau=T$ (Section~\ref{sec:upperBounds}). The regret bounds of the corresponding windowed versions, i.e.,  \texttt{ET-BETA-SWTS}  and \texttt{$\gamma$-ET-SWGTS}, is reported in Appendix~\ref{sec:SWapproach}. We conclude the section by deriving a regret lower bound and discuss the tightness of the upper bounds of the ET-TS algorithms (Section~\ref{sec:lowerBounds}).

\subsection{ Regret Upper Bounds of ET-TS Algorithms for SRRBs}\label{sec:upperBounds}
Before presenting the results, we introduce the following lemma relating the expected cumulative regret to the expected number of pulls of the suboptimal arms.

\begin{restatable}[Wald's Inequality for Rising Bandits]{lemma}{wald}\label{lemma:wald}
	Let $\mathfrak{A}$ be an algorithm and $T \in \Nat$ be a learning horizon, it holds that: $
		R_{\bm{\mu}}(\mathfrak{A},T) \le \sum_{i\neq i^*(T)} \Delta_i(T,1)\E_{\bm{\mu}}[N_{i,T}]$.
\end{restatable}
It is worth noting that Lemma~\eqref{lemma:wald} holds with equality for stationary bandits. In the following, as it is typical in the literature of bandits in evolving environments \citep[e.g.,][]{garivier2011upper, trovo2020sliding, besson2019generalized}, we provide bounds on the expected value of plays of the suboptimal arms, under different assumptions for the distribution of the rewards (Assumptions~\ref{ass:bernoulli} and~\ref{ass:nonNeg}). From these results, the regret bound can be easily computed via Lemma~\ref{lemma:wald}. We start with the \texttt{ET-Beta-TS} algorithm. 

\begin{restatable}[\texttt{ET-Beta-TS} Bound]{thr}{ts} \label{thm:ts}
Let $\sigma \in \dsb{\sigma_{\bm{\mu}}(T), T}$. Under Assumptions~\ref{ass:nonDecreasing} and~\ref{ass:bernoulli}, for the \texttt{ET-Beta-TS} algorithm, for every arm $i \in \dsb{K}\setminus\{i^*(T)\}$, it holds for every $\epsilon \in (0,1]$, that:\footnote{For the sake of presentation, with the Big-O notation, we are neglecting terms depending on $\mu_i(T)$ and $\mu_{i^*(T)}(\sigma)$, but not explicitly on $T$. The full expression is reported in Equation~\eqref{eq:fullExpression} in the appendix.}
{\thinmuskip=2mu
  \medmuskip=2mu \thickmuskip=2mu\begin{align*} 
    \mathbb{E}_{\bm{\mu}}[N_{i,T}] \le  O \Bigg( \textcolor{vibrantRed}{\underbrace{\Gamma}_{\text{($i$)}}} + \textcolor{vibrantBlue}{\underbrace{   \frac{ (1+\epsilon)\log (T)}{d(\overline{\mu}_i(T),\overline{\mu}_1(\sigma))} + \frac{1}{\epsilon^2}}_{\text{($ii$)}}}
    + \textcolor{vibrantTeal}{\underbrace{\sum_{j=\Gamma}^{\sigma-1}\frac{\delta_{\mathrm{TV}}(\mathrm{Bin}(j,\overline{\mu}_{i^*(T)}(j)),\mathrm{Bin}(j,\overline{\mu}_{i^*(T)}(\sigma)) }{(1-\overline{\mu}_{i^*(T)}(\sigma))^{j+1}}}_{\text{($iii$)}}}
      \Bigg),
\end{align*}}%
where $d(x,y) = x \log \frac{x}{y} + (1-x) \log \frac{1-x}{1-y}$ is the KL divergence between Bernoulli distributions of parameters $x,y \in [0,1]$, 
$\delta_{\mathrm{TV}}(P, Q) \coloneqq \sup_{A \in \mathcal{F}}{\left| P(A) - Q(A) \right|}$ is the total variation divergence between $P$ and $Q$, and $\mathrm{Bin}(j,x)$ is the binomial distribution with $j$ trials and parameter $x$. 
\end{restatable}

Let us now move to the \texttt{$\gamma$-ET-GTS} algorithm.
\begin{restatable}[\texttt{$\gamma$-ET-GTS} Bound]{thr}{gts} \label{thm:gts}
Let $\sigma \in \dsb{\sigma_{\bm{\mu}}(T), T}$. Under Assumptions~\ref{ass:nonDecreasing} and~\ref{ass:nonNeg}, setting $\gamma \leq \min \left\{ \frac{1}{4\lambda^2}, 1\right\}$, 
for the \texttt{$\gamma$-ET-GTS} algorithm, for every arm $i \in \dsb{K}\setminus\{i^*(T)\}$, it holds that:
\begin{align*}
    \mathbb{E}_{\bm{\mu}}[N_{i,T}]\leq O \Bigg( \textcolor{vibrantRed}{\underbrace{\Gamma}_{\text{($i$)}}} &+ \textcolor{vibrantBlue}{\underbrace{\frac{\log (T\overline{\Delta}_i(\sigma,T)^2+e^6)}{\gamma\overline{\Delta}_i(\sigma,T)^2}}_{\text{($ii$)}}}+\textcolor{vibrantTeal}{\underbrace{\sum_{j=\Gamma}^{{\sigma}-1}\frac{\delta_{\mathrm{TV}}(\mathbb{P}_j,\mathbb{Q}_j(\overline{\mu}_{i^*(T)}(\sigma))) }{\mathrm{erfc}(\sqrt{{{\gamma j}}/{2}}(\overline{\mu}_{i^*(T)}(\sigma)))}}_{\text{($iii$)}}}\Bigg),
\end{align*}
where $\mathrm{erfc}(\cdot)$ is the complementary error function, $\mathbb{P}_j$ is the distribution of the sample mean of the first $j$ samples collected from arm $i^*(T)$, while $\mathbb{Q}_j(y)$ is the distribution of the sample mean of $j$ samples collected from \emph{any} $\lambda^2$-subgaussian distribution with mean $y$.
\end{restatable}
Both bounds on the expected number of pulls are characterized by three terms: ($i$) corresponds to the number of pulls $\Gamma$ performed during the \emph{\textcolor{vibrantRed}{forced exploration }}; ($ii$) is the \emph{\textcolor{vibrantBlue}{expected number of pulls in a stationary bandit}} with expected rewards $\overline{\mu}_i(T)$ and $\overline{\mu}_{i^*(T)}(\sigma)$; ($iii$) is a dissimilarity index represented by the cumulative \emph{\textcolor{vibrantTeal}{total variation (TV) distance}} measuring the dissimilarity between the real distribution of the optimal arm's rewards and that of a stationary bandit with expected reward $\overline{\mu}_{i^*(T)}(\sigma)$ which originates from a \emph{change of measure} argument.\footnote{The state-of-the-art bounds for the total variation distance are discussed in Lemmas \ref{lemma:delta} and \ref{lem:gil}.} In detail, for \texttt{ET-Beta-TS}, term ($iii$) is related to the TV between the distribution of the reward at the $j$-th pull $\mathrm{Bin}(j,\overline{\mu}_{i^*(T)}(j))$ and that evaluated the reference number of pulls $\mathrm{Bin}(j,\overline{\mu}_{i^*(T)}(\sigma))$. Similarly, for \texttt{$\gamma$-ET-GTS}, ($iii$) is related to the TV divergence between the distribution $\mathbb{P}_j$ of the sample mean of the first $j$ rewards sampled from the optimal arm $i^*(T)$ and the distribution $\mathbb{Q}_j$ of the sample mean of $j$ samples obtained from \emph{an arbitrarily chosen} subgaussian distribution with mean $\overline{\mu}_{i^*(T)}(\sigma)$ and same subgaussian parameter $\lambda^2$. These terms vanish in a stationary bandit. Indeed, in a stationary bandit, by setting $\Gamma=0$, we retrieve the state-of-the-art bounds for TS with Bernoulli and Gaussian priors (see Theorem 1.1 and 1.3 from~\citet{agrawal2017near}).
The free parameter $\sigma \in \dsb{\sigma_{\bm{\mu}}(T),T}$ can be chosen to tighten the bounds, exercising the tradeoff between term ($ii$), that decreases with $\sigma$ since $\overline{\mu}_{i^*(T)}(\sigma)$ moves away from $\overline{\mu}_i(T)$ and term ($iii$), that increases with $\sigma$ being the summation made of $\sigma - \Gamma$ terms. Whenever $\sigma$ can be selected as a constant independent from $T$ (i.e., when $\sigma_{\bm{\mu}}(T)$ is independent of $T$), term ($iii$) is a constant of order $O(\sigma)$, leading to a regret bound that matches that of stationary bandits. In these cases, we can freely remove the forced exploration by setting $\Gamma = 0$. However, whenever $\sigma$ depends on $T$, the bounds suggest that {enforcing} the exploration through $\Gamma$ can be beneficial since term ($i$) can be smaller than term ($iii$) when the TV divergence is close to $1$ (due to the denominators in the summations of term ($iii$)). Finally, we note that in \texttt{$\gamma$-ET-GTS}, the non-negativity of the reward, enforced by Assumption~\ref{ass:nonNeg}, is needed to bound the denominator of the addenda in terms ($iii$).\footnote{If we consider rewards $X \ge -b$ for some $b \ge 0$ a.s., we would replace the denominator in term ($iii$) of Theorem~\ref{thm:gts} with $\mathrm{erfc}(\sqrt{{{\gamma j}}/{2}}(\overline{\mu}_{i^*(T)}(\sigma)+b))$ (see Equation~\ref{eq:nonNegChange} in the appendix).} We highlight that Theorems~\ref{thm:ts} and~\ref{thm:gts} hold the for every SRRB under the assumption that the expected reward functions are non-decreasing (Assumption~\ref{ass:nonDecreasing}) only, with no need of enforcing the \emph{concavity} assumption. {We remark that our bounds are \emph{instance-dependent} (w.r.t. $\overline{\Delta}_i$ and $\sigma_{\bm{\mu}}$) and, 
as shown in previous works~\cite{metelli2022stochastic,fiandri2024rising}, the worst-case regret over the class of all rising bandits (i.e., the \emph{minimax} regret) degenerates to linear if no further structure is enforced, making the problem unlearnable in the minimax sense.}

\subsection{Explicit Regret Upper and Lower Bounds}\label{sec:lowerBounds}
In this section, we first particularize Theorems~\ref{thm:ts} and~\ref{thm:gts} to obtain more explicit bounds, then, we derive e regret lower bound, and we discuss their tightness.
We introduce a subset of SRRBs parametrized by a bound $\overline{\sigma} \ge 0$ to the complexity index $\sigma_{\bm{\mu}}(T)$: $
    \mathcal{M}_{\overline{\sigma}} \coloneqq \{ \bm{\mu} \text{ SRRB } \,:\,  \sigma_{\bm{\mu}}(T) \le \overline{\sigma} \}$. We denote with $\mathcal{M}_{\overline{\sigma}}^{\mathrm{det}}$ the subset of the \emph{deterministic} SRRBs in $\mathcal{M}_{\overline{\sigma}}$.
Intuitively, $\overline{\sigma}$ controls the complexity of the SRRB instances, as it corresponds to the number of pulls needed to allow the algorithm to distinguish the optimal arm $i^*(T)$ from the suboptimal ones in the worst instance of $ \mathcal{M}_{\overline{\sigma}}$. In particular, if $\overline{\sigma}' \ge \overline{\sigma}$, we have that $\mathcal{M}_{\overline{\sigma}} \subseteq \mathcal{M}_{\overline{\sigma}'}$, and $\mathcal{M}_{T}$ is the set of all SRRBs. This allows
deriving a more explicit regret bound for our ET-TS algorithms.

\begin{restatable}[Explicit \texttt{Beta-TS} and \texttt{$\gamma$-GTS} Bound]{coroll}{cor}\label{thm:noregret}
     Let $\overline{\sigma} \ge 0$. Under the same assumptions of Theorems~\ref{thm:ts} and~\ref{thm:gts}, setting $\Gamma = \alpha\overline{\sigma}$, with $\alpha \ge 1$,  for both \texttt{Beta-TS} and \texttt{$\gamma$-GTS}, for every $\bm{\mu} \in\mathcal{M}_{\overline{\sigma}}$ and  for every arm $i \in \dsb{K} \setminus \{i^*(T)\}$, it holds that:
\begin{equation*} 
    \mathbb{E}_{\bm{\mu}}[N_{i,T}] \le  O \bigg( {\overline{\sigma}} + { \frac{ \log (T)}{\overline{\Delta}_i(\overline{\sigma},T)^2} }
      \bigg).\footnote{Details about the choice of the exploration parameter are provided in Appendix~\ref{apx: explore}.}
\end{equation*}
\end{restatable}

Thus, we identify two components: ($i$) the complexity index $\overline{\sigma}$ and ($ii$) the usual expected number of pulls $O ( \frac{\log(T)}{\overline{\Delta}_i^2} )$ unavoidable even in any stationary stochastic bandit~\cite{lattimore2020bandit}. The following regret lower bound shows that the dependence on $\overline{\sigma}$ is also unavoidable even for the deterministic SRRBs.

\begin{restatable}[Lower Bound]{thr}{lb} \label{thm:lb}
    Let $T \in \mathbb{N} $ and $\overline{\sigma} \in \dsb{2,\frac{T-1}{2}}$. For every algorithm $\mathfrak{A}$, it holds that:
    \begin{equation}
        \sup_{\bm{\mu} \in \mathcal{M}^{\mathrm{det}}_{\overline{\sigma}}} R_{\bm{\mu}}(\mathfrak{A},T)\ge \frac{K}{64}(\overline{\sigma}-2).
    \end{equation}
\end{restatable}


First of all, we note that the lower bound is derived considering deterministic instances. Indeed, the complexity index $\sigma_{\bm{\mu}}(T)$ is not affected by the possible reward stochasticity.
Thus, 
\emph{any algorithm} that wishes to be no-regret (with no additional information), cannot avoid playing every arm a number of times proportional to $\overline{\sigma}$, as it is always possible to design an instance that needs at least $\overline{\sigma}$ pulls to differentiate $i^*(T)$ from the suboptimal arms. Finally, by recalling that the logarithmic regret component in Corollary~\ref{thm:noregret} comes from the reward stochasticity (and it is tight for standard bandits), Theorem~\ref{thm:lb} shows that our bound of Corollary~\ref{thm:noregret} is tight in the dependence on the complexity index bound $\overline{\sigma}$ for the class of SRRBs $\mathcal{M}_{\overline{\sigma}}$.

\section{Numerical Simulations}\label{sec:Experiments}
In this section, we validate our algorithms on both synthetic and real-world environments. 

\subsection{Synthetic Experiments}
We consider the same 15-arms environment of~\cite{metelli2022stochastic} (Figure~\ref{fig:15armexp100k}) and compare against \texttt{R-ed-UCB}~\cite{metelli2022stochastic}.\footnote{We also compare with the other baselines considered in~\citet{metelli2022stochastic} in Appendix~\ref{apx:comparison222}.} The parameters, complying with the recommendation of each algorithm, and the parameters defining the environment are provided in Appendix~\ref{app:ex}. We evaluate empirical cumulative regret $\hat{R}(\mathfrak{A}, t)$ averaged over $50$ independent runs (with the corresponding standard deviations), over a time horizon of $T = 100,000$ rounds. We run the sliding-window algorithms 
against the worst-case misspecification for the forced exploration parameter, i.e., $\Gamma=0$ for $\texttt{Beta-SWTS}$ and $\Gamma=1$ for $\texttt{$\gamma$-SWGTS}$.\footnote{Our analysis suggests that a sliding-window can mitigate a too aggressive (small) choice of $\Gamma$ (Appendix~\ref{sec:SWapproach}).}
For the non-windowed algorithms (\texttt{ET-Beta-TS} and \texttt{$\gamma$-ET-GTS}), we consider $\Gamma=2000$. We also include their standard versions (designed for the stationary setting) without the forced exploration  (\texttt{Beta-TS} and \texttt{$\gamma$-GTS}), to assess whether the proposed modifications behave as suggested by our analysis. 

\textbf{Results.}~~In Figure \ref{fig:15armexp100k}, we observe that \texttt{Beta-TS}, $\gamma$-\texttt{SWGTS}, and \texttt{Beta-SWTS} suffer smaller regret than \texttt{R-ed-UCB} over the entire time horizon, whereas \texttt{ET-Beta-TS} and \texttt{$\gamma$-ET-GTS}, coherently with the theoretical analysis, suffer linear regret in the forced exploration phase and then catch up to \texttt{R-ed-UCB} for $t \gtrsim 50,000$. Indeed, there is statistical evidence of the superiority of our TS-like from that point on. Moreover, all the algorithms, except \texttt{Beta-TS}, \texttt{$\gamma$-GTS} and \texttt{$\gamma$-SWGTS}, display a flattening regret curve for $t \gtrsim  30,000$. This is explained since \texttt{Beta-TS} and \texttt{$\gamma$-GTS} are designed for stationary bandits, while, as we shall see, the performance of \texttt{$\gamma$-SWGTS} depends on the window size.

\begin{figure}[t!]
\centering
\begin{minipage}{.24\textwidth}
  \centering
  \vspace{-.3cm}
  \includegraphics[width=.98\textwidth]{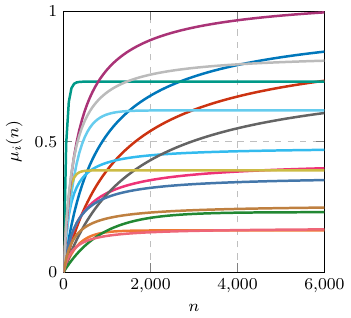}
  \captionof{figure}{15-arm setting: expected reward functions of the arms.}
  \label{fig:15arms}
\end{minipage}%
\hfill
\begin{minipage}{.24\textwidth}
  \centering
  \includegraphics[width=1.05\linewidth]{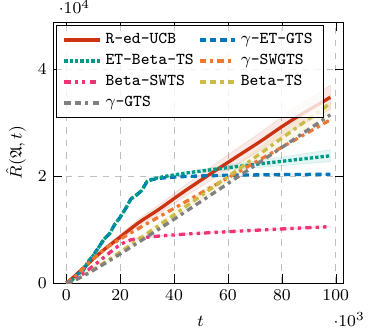}
  \captionof{figure}{Average cumulative regret in the 15-arm setting ($50$ runs $\pm$ std).}
  \label{fig:15armexp100k}
\end{minipage}%
\hfill
\begin{minipage}{.48\textwidth}
    \subfloat[][Sensitivity on $\tau = T^\alpha$.]{\includegraphics[width=.49\textwidth]{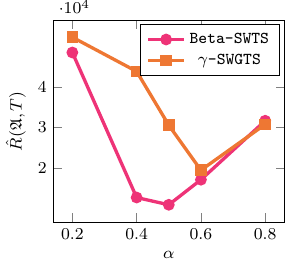}} \label{fig:sensitivity1}
    \subfloat[][Sensitivity on $\Gamma$.]{\includegraphics[width=.49\textwidth]{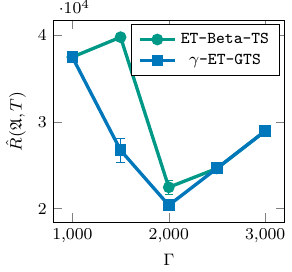}} \label{fig:sensitivity2}
\caption{Average cumulative regret of (a) \texttt{Beta-SWTS} and   \texttt{$\gamma$-SWGTS} for different window sizes $\tau=T^{\alpha}$ (b)  \texttt{ET-Beta-SWTS} and   \texttt{$\gamma$-ET-SWGTS} for different forced exploration $\Gamma$.}
\label{fig:sens15100k}
\end{minipage}
\end{figure}

\textbf{Sensitivity Analysis.}~~
We run a sensitivity analysis on the sliding window $\tau$ for \texttt{Beta-SWTS} and \texttt{$\gamma$-SWGTS} and on the $\Gamma$ parameter for \texttt{ET-Beta-TS} and \texttt{$\gamma$-ET-GTS}. We provide the value of their average regret $\hat{R}(\mathfrak{A}, T)$ at the end of the time horizon $T$, where the average has been taken over $50$ runs of the algorithms on the same 15-arms setting. The results are reported in Figure \ref{fig:sens15100k}.\footnote{We report the standard deviation bars that, sometimes, are not clearly visible due to their limited size.} 
Coherently with our analysis, whenever the forced exploration parameter $\Gamma$ or the window size $\tau$ is large enough, our algorithms are able to identify the best arm consistently. A too-small value of $\Gamma$, instead, leads to a poor exploration that makes the regret grow fast. Similarly, employing a too-small window size $\tau$, introduces a large variance resulting in a poorer performance.

\begin{wrapfigure}{r}{0.3\textwidth}
\vspace{-.5cm}
    \resizebox{\linewidth}{!}{\includegraphics{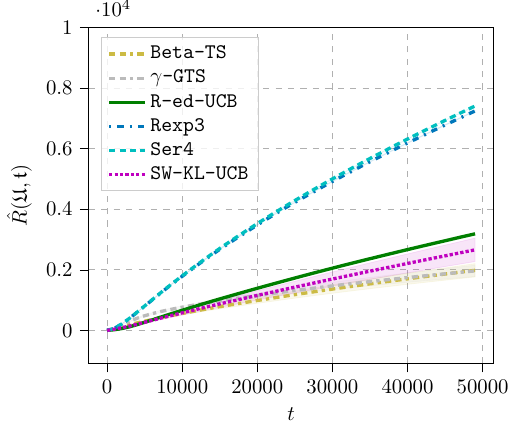}}
    \captionof{figure}{Average cumulative regret in the IMBD setting ($30$ runs $\pm$ std).}\label{fig:imdb}
    \vspace{-.5cm}
\end{wrapfigure}
\subsection{IMDB Experiment}
We consider the same \emph{online model selection} task on the IMDB dataset proposed in \cite{metelli2022stochastic}, where each arm corresponds to a different online-trained classifier and the learner gets reward $1$ when the predicted label is correct, $0$ otherwise. Whenever a classifier is pulled, one step of training is performed. We compare the non-windowed approaches (\texttt{Beta-TS} and \texttt{$\gamma$-GTS}), over a time horizon of $T=50,000$ considering the empirical cumulative regret over 30 runs, against several baselines considered in \cite{metelli2022stochastic} with the values of the hyperparameters defined there.

\textbf{Results.}~~ The results are reported in Figure~\ref{fig:imdb}. We observe that, as we discussed, TS-like algorithms can be still effective whenever the complexity term 
$\sigma_{\bm{\mu}}(T)$ is low enough. In fact \texttt{Beta-TS} and standard $\gamma$\texttt{-GTS} outperform the baselines, even those designed explicitly for SRRBs bandits, i.e., \texttt{R-ed-UCB}.
Additional experiments are reported in Appendix~\ref{apx:imdb}.

\section{Conclusions}

In this paper, we investigated the properties of TS-like algorithms for regret minimization in the setting of SRRBs. We analyzed the \texttt{TS} algorithms with Beta and Gaussian priors. In both cases, we derived a general analysis that highlights the challenges of the setting, which makes use of novel technical tools of independent interest. We showed that in the SRRB setting, the classical logarithmic regret is increased by a term that depends on the total variation distance between pairs of suitably defined distributions. We also specialize these results for a parametric subset of SRRBs showing the tightness of the bounds on the regret, by inferring the complexity of the problem via a lower bound, dependent on the complexity index that we have introduced. 
Finally, we validated our theoretical findings with numerical experiments succeeding in showing the advantages of the approaches we proposed on both synthetic experiments and real-world data.
\newpage
\bibliographystyle{abbrvnat}
\bibliography{biblio}

\newpage


\appendix


\section{Comparison with the \texttt{R-ed-UCB} Algorithm} \label{sec:Comparison}

\begin{figure}[H]
\centering
\subfloat[]{
	\begin{tikzpicture}[scale=.7]
		\draw[->] (1,0) -- (7,0);
		\draw[->] (1,0) -- (1,5) node[above] {$\mu(n)$};
		\draw[dashed](1,4) -- (5,4);
		\node[right] at (7,0) {$n$};
		\draw[domain=1:7, smooth, variable=\t, blue, ultra thick] plot({\t},{4*(1-2^(-\t))});
		\draw[domain=1:7,smooth, variable=\t,red, ultra thick] plot({\t},{4*(1-2^(-2*\t+2))});
		\node[above] at (7,4) {{\color{blue} $\mu_1(n)$}};
		\node[above] at (7,2.5) {{\color{red} $\mu_2(n)$}};
		\node[above] at (0,2.7) {$1$};
		
	\end{tikzpicture} \label{fig:env1}
}
\subfloat[]{
    \begin{tikzpicture}[scale=.7]
	\draw[->] (0,0) -- (6,0);
	\draw[->] (0,0) -- (0,5) node[above] {$\mu(n)$};

    \node[right] at (6,0) {$n$};
    \draw[dashed](0,2.5) -- (4.5,2.5);
	
	\draw[domain=0:6, smooth, variable=\t, blue, ultra thick] plot ({\t},{5*(1-1/(2+\t))});
 
	\draw[domain=0:6, smooth, variable=\t, red, ultra thick] plot ({\t},{5*(0.5-1/(2+\t))});
	\node[above] at (6,5) {{\color{blue} $\mu_1(n)$}};
	\node[above] at (6,2) {{\color{red} $\mu_2(n)$}};
\end{tikzpicture} \label{fig:env2}
}
\caption{Different environments for the SRRB problem.}
\label{fig:env}
\end{figure}
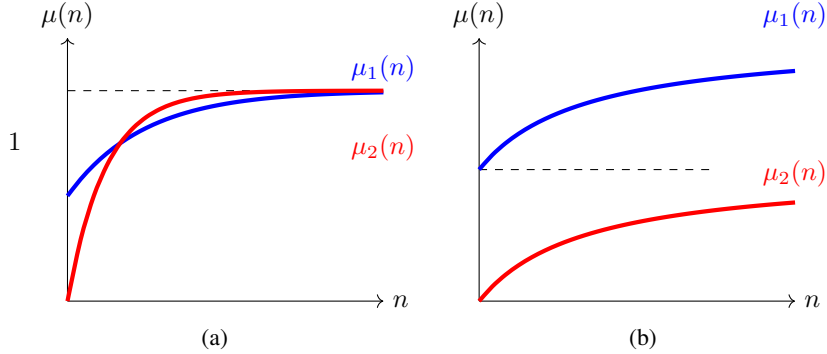

In this appendix, provide and analyze two SRRB instances to highlight the advantages and disadvantages of the proposed algorithms when compared with the optimistic algorithm \texttt{R-ed-UCB}~\cite{metelli2022stochastic} designed for SRRB settings. We recall the regret result provided by~\cite{metelli2022stochastic} for \texttt{R-ed-UCB}: 

\begin{restatable}[Theorem 4.4, \cite{metelli2022stochastic}]{thr}{metelli} \label{thm:metelli}
 \texttt{R-ed-UCB} with a suitable exploration index (see \cite{metelli2022stochastic}) $\alpha > 2$, and $\epsilon \in (0,1/2)$ suffers an expected regret for every $q \in [0,1]$ bounded as:
\begin{equation}
	R(\texttt{R-ed-UCB}, T) \leq O \left( \frac{K}{\epsilon}T^{\frac{2}{3}}(\alpha \log T)^{\frac{1}{3}} + \frac{K T^q}{1-2 \epsilon} \Upsilon \left(\Bigl \lceil (1-2 \epsilon) \frac{T}{K} \Bigr \rceil, q\right)\right),
\end{equation}
where $\Upsilon(M,q) := \sum_{l=1}^{M-1} \max_{i \in \dsb{K}} \{ \gamma_i(l)^q \}$ is a complexity index depending on the expected rewards.
\end{restatable}

\textbf{First Instance.}~~We start with an instance in which \texttt{R-ed-UCB} succeeds in delivering a sublinear regret, while our algorithms may fail. We define the expected reward functions as follows: $\mu_1(n) = 1 - 2^{-n}$ and $\mu_2(n) = 1 - 2^{-2n+2}$ (Figure~\ref{fig:env1}). It is not possible to find a 
 $\sigma$, so that $\overline{\Delta}(\sigma,T) \ge \overline{\Delta}$, with $\overline{\Delta}>0$ and not vanishing with $T$. This is visible since $\max_{\sigma \le T} \overline{\Delta}(\sigma,T) \le \frac{5}{6T}$. Therefore, we cannot guarantee that our algorithms provide a sublinear regret.
Conversely, by using the definition of $\Upsilon$, for $q \in [0,1]$, we have:
\begin{align*}
	\Upsilon \left( \Bigl \lceil (1-2 \epsilon) \frac{T}{K} \Bigr \rceil, q \right) = \sum_{n=1}^{\lceil (1-2 \epsilon) \frac{T}{K}  \rceil} \max_{y \in \{1,2\}}\{e^{- y \lambda n} - e^{-y \lambda (n+1)}\}^{q} \le \frac{3}{4} + \frac{1}{q \log 2},
\end{align*}
which implies that Theorem~\ref{thm:metelli} provides a regret of order $O(T^{2/3}+T^{q}/q)$ for the \texttt{R-ed-UCB} algorithm which is sublinear for every $q < 1$ and, selecting $q = 1/\log T$ we obtain the best rate $O(T^{2/3} + \log T)$.

\textbf{Second Instance.}~~The second instance is designed so that our algorithms provide sublinear regret, while \texttt{R-ed-UCB} fails.
We define the expected reward functions as: $\mu_1(n) = 1 - \frac{2^{\lambda-1}}{(n+1)^\lambda}$, and $\mu_2(n) =  \frac{1}{2} - \frac{2^{\lambda-1}}{(n+1)^\lambda}$, where $\lambda \in [0,1]$ is an arbitrary parameter (Figure~\ref{fig:env2}). With $\sigma = 2$ and $\overline{\Delta}(2,T)= \overline{\mu}_1(2)-\frac{1}{2}$, that is independent from $T$, since $\overline{\mu}_1(2) > \frac{1}{2}\ge\overline{\mu}_2(T)$ for all possible time horizons $T$. The $\Upsilon$ factor of this setting for every $q \in [0,1]$ is given by:
\begin{align*}
   \Upsilon \left( \Bigl \lceil (1-2 \epsilon) \frac{T}{K} \Bigr \rceil, q \right) = \hspace{-0.4 cm} \sum_{n=1}^{   \lceil (1-2 \epsilon) \frac{T}{K}  \rceil } \left( \frac{2^{\lambda-1}}{(n+1)^\lambda} - \frac{2^{\lambda-1}}{(n+2)^\lambda} \right)^q \ge O 
   \begin{cases}
        \lambda^q & \text{if } q(\lambda+1)>1 \\
        \lambda^\frac{1}{\lambda+1} \log T & \text{if } q(\lambda+1)=1 \\
        \lambda^q T^{1-q(\lambda+1)} & \text{otherwise}
   \end{cases}.
\end{align*}
We prove in Appendix~\ref{apx:detailsInstances}, that the optimal choice of $q$ is $1/(\lambda+1)$, according to Theorem~\ref{thm:metelli}, this leads to a regret of order $O(T^{2/3}+(\lambda T)^{\frac{1}{\lambda+1}})$ for \texttt{R-ed-UCB}. Thus, for instance, choosing $\lambda = 1/3$, \texttt{R-ed-UCB} attains a regret bound of order $O(T^{3/4})$ while our approaches succeed to achieve an instance-dependent $O(\log(T))$ regret.

\subsection{Detailed Computations for the Instances of Appendix~\ref{sec:Comparison}}\label{apx:detailsInstances}

\textbf{First Instance.} Let us upper bound the complexity index:
\begin{align*}
	\Upsilon \left( \Bigl \lceil (1-2 \epsilon) \frac{T}{K} \Bigr \rceil, q \right) & = \sum_{n=1}^{\lceil (1-2 \epsilon) \frac{T}{K}  \rceil} \max\{2^{-n}-2^{-(n+1)}, 2^{-2n+2}-2^{-2(n+1)+2}\}^q \\
     & \le \frac{3}{4} + \sum_{n=2}^{\lceil (1-2 \epsilon) \frac{T}{K}  \rceil} (2^{-n}-2^{-(n+1)} + 2^{-2n+2}-2^{-2(n+1)+2})^q \\
     & = \frac{3}{4} + \sum_{n=2}^{\lceil (1-2 \epsilon) \frac{T}{K}  \rceil} (2^{-1-2n}(6+2^n))^q\\
     & \le \frac{3}{4} + \sum_{n=2}^{\lceil (1-2 \epsilon) \frac{T}{K}  \rceil} \left(\frac{5}{4} 2^{-n}\right)^q\\
    & \le \frac{3}{4} + \int_{n=2}^{+\infty} (2^{(-1-2n)}(6+2^n))^q \de n \\
    & = \frac{3}{4} + \frac{1}{q \log 2} \left(\frac{5}{8}\right)^q\le \frac{3}{4} + \frac{1}{q \log 2},
\end{align*}
where we bounded the maximum with the sum, performed some algebraic bounds, and bounded the summation with the integral. For this instance, we compute the average expected rewards, for every $n \in \dsb{N}$:
\begin{align}
    & \overline{\mu}_1(n) = 1 - \frac{1}{n} \left( 1 - 2^{-n} \right), \\
    & \overline{\mu}_1(n) = 1 - \frac{4}{3n} \left( 1 - 2^{-2n} \right).
\end{align}
We will show that:
\begin{align}
    \inf_{T \in \mathbb{N}} \max_{\sigma \in \dsb{T}} \overline{\Delta}(\sigma,T) = 0.
\end{align}
Let us fix $T \ge 1$, we have:
\begin{align}
    \max_{\sigma \le T} \overline{\Delta}(\sigma,T) & = \max_{\sigma \le T} \left\{  1 - \frac{1}{\sigma} \left( 1 - 2^{-\sigma} \right) -  1 + \frac{4}{3T} \left( 1 - 2^{-2T} \right)\right\} \\
    & = 1 - \frac{1}{T} \left( 1 - 2^{-T} \right) -  1 + \frac{4}{3T} \left( 1 - 2^{-2T} \right) \\
    & \le \frac{1}{3T} + \frac{2^{-T}}{T} \le \frac{5}{6T},
\end{align}
being the function $- \frac{1}{\sigma} \left( 1 - 2^{-\sigma} \right)$ non-decreasing in $\sigma$. 

\textbf{Second Instance.} Let us lower bound the complexity index:
\begin{align*}
     \Upsilon \left( \Bigl \lceil (1-2 \epsilon) \frac{T}{K} \Bigr \rceil, q \right) &=  \sum_{n=1}^{   \lceil (1-2 \epsilon) \frac{T}{K}  \rceil } \left(\frac{2^{\lambda-1}}{(n+1)^\lambda} - \frac{2^{\lambda-1}}{(n+2)^\lambda}\right)^q \\
     & \ge \sum_{n=1}^{   \lceil (1-2 \epsilon) \frac{T}{K}  \rceil } \left( \frac{2^{\lambda-1}\lambda}{(n+2)^{\lambda+1}} \right)^q, \\
\end{align*}
where we used $\frac{2^{\lambda-1}}{(n+1)^\lambda} - \frac{2^{\lambda-1}}{(n+2)^\lambda} \ge \min_{x \in [n,n+1]}\frac{\partial}{\partial x} \left(1 - \frac{2^{\lambda-1}}{(x+1)^\lambda}\right) = \frac{2^{\lambda-1}\lambda}{(n+2)^{\lambda+1}}$. For $q(\lambda+1)>1$, we proceed as follows:\footnote{We use big-O notation to highlight the dependences on $\lambda \rightarrow 0$ and $T \rightarrow +\infty$.}
\begin{align*}
     \sum_{n=1}^{   \lceil (1-2 \epsilon) \frac{T}{K}  \rceil } \left( \frac{2^{\lambda-1}\lambda}{(n+2)^{\lambda+1}} \right)^q \ge 2^{q(\lambda-1)}3^{-q (\lambda+1)}\lambda^q = O(\lambda^q).
\end{align*}
Instead, for $q(\lambda+1)=1$, we bound the summation with the integral:
\begin{align*}
    \sum_{n=1}^{   \lceil (1-2 \epsilon) \frac{T}{K}  \rceil } \left( \frac{2^{q(\lambda-1)}\lambda^q}{n+2} \right)^q  & \ge    \int_{n=1}^{  (1-2 \epsilon) \frac{T}{K} } \left( \frac{2^{q(\lambda-1)}\lambda^q}{n+2} \right)  \mathrm{d} n\\
    & \ge 2^{q(\lambda-1)}\lambda^q \log \left( (1-2 \epsilon) \frac{T}{K}  - \frac{2}{3} \right) = O (\lambda^q \log T).
\end{align*}
Finally, for $q(\lambda+1)<1$, we still bound the summation with the integral:
\begin{align*}
    \sum_{n=1}^{   \lceil (1-2 \epsilon) \frac{T}{K}  \rceil } \left( \frac{2^{\lambda-1}\lambda}{(n+2)^{\lambda+1}} \right)^q & \ge    \int_{n=1}^{  (1-2 \epsilon) \frac{T}{K} } \left( \frac{2^{\lambda-1}\lambda}{(n+2)^{\lambda+1}} \right)^q  \mathrm{d} n\\
    & \ge \frac{2^{q(\lambda-1)}\lambda^q}{1-q(\lambda+1)} \left(\left( (1-2 \epsilon) \frac{T}{K} \right)^{1-q(\lambda+1)} - 3^{1-q(\lambda+1)} \right) \\&
    = O(\lambda^q T^{1-q(\lambda+1)}).
\end{align*}
Now, recalling that the instance-dependent component of the regret of Theorem~\ref{thm:metelli} is in the order of $T^q  \Upsilon \left( \Bigl \lceil (1-2 \epsilon) \frac{T}{K} \Bigr \rceil, q \right) $, we have for the three cases the optimal choice of $q$ that minimizes the regret:
\begin{align*}
    & q(\lambda+1)>1 \implies q \downarrow \frac{1}{\lambda+1} \implies O\left( \lambda^{\frac{1}{\lambda+1}} T^{\frac{1}{\lambda+1}} \right); \\
    & {q(\lambda+1)=1} \implies q=\frac{1}{\lambda+1} \implies O\left( \lambda^{\frac{1}{\lambda+1}} T^{\frac{\lambda}{\lambda+1}} \log T \right); \\
    & {q(\lambda+1)<1} \implies q \uparrow \frac{1}{\lambda+1} \implies O\left( \lambda^{\frac{1}{\lambda+1}} T^{\frac{1}{\lambda+1}} \right). 
\end{align*}
Thus, we have that the bound of the instance-dependent component of the regret is at least $O\left( T^{\frac{1}{\lambda+1}} \right)$.

\clearpage
\section{Regret Analysis of Sliding Window Approaches} \label{sec:SWapproach}
The main drawback of the previously presented approach is that they use \emph{all the samples from the beginning of learning} for estimating the average expected reward. However, in some cases, it might be convenient to forget the past and focus on the most recent samples only. 

\textbf{Preliminaries.}~~We extend the definitions of Section~\ref{sec:problem} to account for a sliding window. For every arm $i \in \dsb{K}$, round $t \in \dsb{T}$, and window size $\tau \in \dsb{T}$, we define the \emph{windowed average expected reward} as $\overline{\mu}_{i}(t;\tau) \coloneqq \frac{1}{\tau} \sum_{l=t-\tau+1}^t \mu_i(l)$.
Furthermore, we define the minimum number of pulls needed so that the optimal arm $i^*(T)$ can be identified as optimal in a window of size $\tau$:
\begin{align*}
    &\sigma_i'(T;\tau) \hspace{-0.1 cm}\coloneqq \hspace{-0.1 cm}\min\{ \{l\hspace{-0.05 cm} \in \hspace{-0.05 cm}\dsb{T} : \overline{\mu}_{i^*(T)}(l;\tau) > \mu_i(T)\} \hspace{-0.09 cm}\cup \hspace{-0.09 cm} \{+\infty\} \}, \nonumber\\ &\sigma'(T;\tau):=\max_{i \neq i^*(T)} \sigma_i'(T;\tau). 
\end{align*}
These definitions resemble those of Equation~\eqref{eq:sigmaDef}. However, here they involve the windowed average expected reward of the optimal arm $\overline{\mu}_{i^*(T)}(l;\tau)$ compared against the expected reward (not averaged) of the other arms at the end of the learning horizon $\mu_i(T)$. Furthermore, for some values of $\tau$, a number of pulls $l$ so that $\overline{\mu}_{i^*(T)}(l;\tau) > \mu_i(T)$ might not exist. In such a case, we set $ \sigma_i'(T;\tau)$ (and, thus, $\sigma'(T;\tau)$) to $+\infty$. Nevertheless, as visible in Figure~\ref{fig:spiegone}, in some cases $\sigma'(T;\tau) \ll \sigma(T)$. This justifies the introduction of the complexity index $\sigma'(T;\tau)$ that will appear in the regret bounds presented in this section. Finally, we introduce a new definition of suboptimality gaps: $\Delta_i'(T;\tau) \coloneqq \overline{\mu}_{i^*(T)}(\sigma'(T;\tau);\tau) - \mu_i(T)$ for every arm $i \in \dsb{K}$.
\begin{figure}[th!]
\begin{center}
\includegraphics{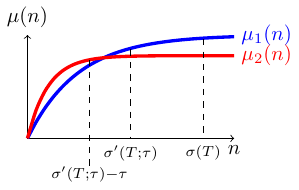}
\end{center}
    \caption{Visual representation of $\sigma'(T;\tau)$, a point in which $\overline{\mu}_{i^*(T)}(\sigma'(T;\tau), \tau) > \mu_i(T)$.}
    \label{fig:spiegone}
\end{figure}

\textbf{Regret.}~~In this section, we analyze the sliding window versions of TS with Bernoulli priors, namely \texttt{ET-Beta-SWTS} (Algorithm~\ref{alg:betats}), and Gaussian priors, namely $\gamma$-\texttt{ET-SWGTS} (Algorithm~\ref{alg:gts}).
The following results provide the regret upper bounds achieved by these algorithms as a function $\tau$. 




\begin{restatable}[\texttt{ET-Beta-SWTS} Bound]{thr}{swbeta}\label{thr:swbeta}
Under Assumption~\ref{ass:bernoulli}, for the \texttt{ET-Beta-SWTS} algorithm with a window of size $\tau \in \dsb{T}$, for every arm $i \in \dsb{K}\setminus\{i^*(T)\}$, it holds that:
\begin{align*}
	\mathbb{E}[N_{i,T}]\leq O &\Bigg(\textcolor{vibrantRed}{\underbrace{\Gamma}_{\text{($i$)}}} + \textcolor{vibrantBlue}{\underbrace{\frac{T\log(T)}{\tau(\Delta_i'(T;\tau))^{3}}}_{\text{($ii$)}}}+\textcolor{vibrantTeal}{\underbrace{\frac{\sigma'(T;\tau)-\Gamma}{(1-\overline{\mu}_{i^*(T)}(\sigma'(T;\tau); \tau))^{\tau+1}}}_{\text{($iii$)}}}\Bigg).
\end{align*}
\end{restatable}

\begin{restatable}[$\gamma$-\texttt{ET-SWGTS} Bound]{thr}{gtssw}\label{thr:gtssw}
Under Assumption~\ref{ass:nonNeg}, setting $\gamma \leq \min \left\{ \frac{1}{4\lambda^2}, 1\right\}$, 
for the \texttt{$\gamma$-ET-SWGTS} algorithm with a window of size $\tau \in \dsb{T}$, for every arm $i \in \dsb{K}\setminus\{i^*(T)\}$, it holds that:
\begin{align*}
	\mathbb{E}[N_{i,T}]\leq &O \Bigg( \textcolor{vibrantRed}{\underbrace{\Gamma}_{\text{($i$)}}} + \textcolor{vibrantBlue}{\underbrace{\frac{T\log(T(\Delta_i'(T;\tau))^2+e^6)}{\gamma\tau(\Delta_i'(T;\tau))^{2}}+\frac{T}{\tau}}_{\text{($ii$)}}}+\textcolor{vibrantTeal}{\underbrace{\frac{\sigma'(T;\tau)-\Gamma}{\mathrm{erfc}(\sqrt{{\gamma\tau}/{2}}(\overline{\mu}_{i^*(T)}(\sigma'(T;\tau);\tau))}}_{\text{($iii$)}}} \Bigg).
\end{align*} 
\end{restatable}

We recognize the terms related to ($i$) the number of pulls $\Gamma$ performed during the \emph{\textcolor{vibrantRed}{forced exploration }}; ($ii$) the \emph{\textcolor{vibrantBlue}{expected number of pulls for a windowed algorithm in standard bandits}}; ($iii$) the cumulative \emph{\textcolor{vibrantTeal}{total variation (TV) distance}}.  First, the regret bounds are presented for a generic choice of the window size $\tau$, whose optimal choice depends on the instance-dependent quantities $\sigma'(T;\tau)$ and $\Delta_i'(T;\tau)$. Second, if we compare these regret bounds with those of the corresponding non-windowed versions, we observe that ($iii$) has a magnitude proportional $O(\sigma) = O(\sigma(T))$ in Theorems~\ref{thm:ts} and~\ref{thm:gts} and becomes $O(\sigma'(T;\tau))$ in Theorems~\ref{thr:swbeta} and~\ref{thr:gtssw}. This quantifies the advantage of the sliding window algorithms in the cases in which $\sigma'(T;\tau) \ll \sigma(T)$. Thus, the windowed algorithms can be more robust to inaccurate choices of $\Gamma$ (i.e., $\Gamma$ too small). Furthermore, since we can have at most $\tau$ samples, we are able to control the bias. This is encoded by the denominator of terms ($iii$). However, for the \texttt{ET-Beta-SWTS}, we pay a heavier dependence on $\Delta_i'(T;\tau)^3$ compared to KL divergence term of its non-windowed counterpart (Theorem~\ref{thm:ts}). The dependence on $T$ and $\tau$ is in line with the state-of-the-art results for windowed algorithms \citep[][Theorem 5.1]{combes2014unimodal}.
 Finally, we remark that these regret bounds become vacuous when $\sigma'(T;\tau)=+\infty$. As it is common in the Sliding-Window literature \cite{trovo2020sliding,garivier2011upper}, the best choice for the window-lenght is instance dependent. However, the analysis highlights that \emph{any} choice of the sliding window reduces the impact of the bias term. Furthermore, the analysis suggests a useful criterion for the choice of the window length: whenever the policy maker thinks the instance itself is hard to learn should set a small window length, conversely, for easy to learn problem, can employ a larger window length.

\clearpage
\section{Proofs and Derivations}\label{apx:proofs}
In this appendix, we provide the complete proofs and derivations we have omitted in the main paper.

\subsection{Proofs of Section~\ref{sec:preliminari}}
\rd*
\begin{proof}
We define the event $E_i(t)\coloneqq \{\theta_{i,t,\tau}\leq y_{i,t}\}$. Thus, the following holds:
\begin{equation}
    \mathbb{E}[N_{i,T}]=\sum_{t=1}^{T}\Prob(I_t=i)\leq \Gamma+\underbrace{\sum_{t=K\Gamma+1}^T \Prob(I_t=i, E_i^\complement(t))}_{\text{(A)}}+\underbrace{\sum_{t=K\Gamma+1}^T \Prob(I_t=i, E_i(t))}_{\text{(B)}}.
\end{equation}
Let us first face term (A): 
\begin{align}
    \text{(A)} &\leq\sum_{t=K\Gamma+1}^T \Prob(I_t=i, E_i^\complement(t), N_{i,t,\tau}\leq \omega)+\sum_{t=K\Gamma+1}^T\Prob(I_t=i, E_i^\complement(t), N_{i,t,\tau}\ge \omega)\\
    &\leq \sum_{t=K\Gamma+1}^T\Prob(I_t=i, N_{i,t,\tau}\leq \omega)+\sum_{t=K\Gamma+1}^T\Prob(I_t=i, E_i^\complement(t), N_{i,t,\tau}\ge \omega)\\
    &= \mathbb{E}\left[\sum_{t=K\Gamma+1}^T\mathds{1}\left\{I_t=i,N_{i,t,\tau}\leq\omega\right\}\right]+\sum_{t=K\Gamma+1}^T\Prob(I_t=i, E_i^\complement(t), N_{i,t,\tau}\ge \omega) \\
&\leq\mathbb{E}\left[\underbrace{\sum_{t=1}^{T}\mathds{1}\left\{I_t=i, N_{i,t,\tau}\leq \omega\right\}}_{\text{(C)}}\right]+\sum_{t=K\Gamma+1}^T\Prob(I_t=i, E_i^\complement(t), N_{i,t,\tau}\ge \omega).
\end{align}
Observe that (C) can be bounded by Lemma \ref{lemma:window}. Thus, the above inequality can be rewritten as:
\begin{align}
    \text{(A)} &\leq \frac{\omega T}{\tau}+\underbrace{\sum_{t=K\Gamma+1}^{T}\Prob(I_t=i, E_i^\complement(t), N_{i,t,\tau}\ge \omega)}_{\text{(D)}}.
\end{align}
We now focus on the term (D). Defining $\mathcal{T}:=\{t \in \dsb{K\Gamma+1,T}:1-\Prob(\theta_{i,t,\tau}\leq y_{i,t}\mid \mathcal{F}_{t-1})> \frac{1}{T}, N_{i,t,\tau}\ge \omega \}$ and $\mathcal{T}':=\{t \in \dsb{K\Gamma+1,T}:1-\Prob(\theta_{i,t,\tau}\leq y_{i,t}\mid \mathcal{F}_{t-1})\le \frac{1}{T}, N_{i,t,\tau}\ge \omega \}$ we obtain:
\begin{align}
    \sum_{t=K\Gamma+1}^{T} & \Prob(I_t=i, E_i^\complement(t), N_{i,t,\tau}\ge \omega) = \mathbb{E}\left[\sum_{t=K\Gamma+1}^T\mathds{1}\{I_t=i, E_i^\complement(t), N_{i,t,\tau}\ge \omega\}\right]\\
    &=\mathbb{E}\left[\sum_{t \in \mathcal{T}}\mathds{1}\{I_t=i,E_i(t)^\complement\}\right]+\mathbb{E}\left[\sum_{t \in \mathcal{T}'}\mathds{1}\{I_t=i,E_i(t)^\complement\}\right]\\
    &\leq \mathbb{E}\left[\sum_{t \in \mathcal{T}}\mathds{1}\{I_t=i\}\right]+\mathbb{E}\left[\sum_{t \in \mathcal{T}'}\mathds{1}\{E_i(t)^\complement\}\right]\\
    &\leq \mathbb{E}\left[\sum_{t=K\Gamma+1}^T\mathds{1}\left\{1-\Prob(\theta_{i,t,\tau}\leq y_{i,t}\mid \mathcal{F}_{t-1})> \frac{1}{T}, N_{i,t,\tau}\ge \omega, I_t=i\right\}\right]+\sum_{t=1}^T\frac{1}{T}.
\end{align}
Now we focus on term (B). We have:
\begin{align}
  \sum_{t=K\Gamma+1}^T \Prob(I_t=i, E_i(t))=\sum_{t=K\Gamma+1}^T \mathbb{E}\left[\underbrace{\Prob(I_t=i, E_i(t)\mid \mathcal{F}_{t-1})}_{\text{(E)}}\right].
\end{align}
In order to bound (B) we need to bound (E).
Let $i_t^{\prime}=\operatorname{argmax}_{i \neq i^*(t)} \theta_{i,t,\tau}$. Then, we have:
$$
\begin{aligned}
\mathbb{P}\left(I_t=i^*(t), E_i(t) \mid \mathcal{F}_{t-1}\right) & \geq \mathbb{P}\left(i_t^{\prime}=i, E_i(t), \theta_{i^*(t),t,\tau} > y_{i,t} \mid \mathcal{F}_{t-1}\right) \\
& =\mathbb{P}\left(\theta_{i^*(t),t,\tau} > y_{i,t} \mid \mathcal{F}_{t-1}\right) \mathbb{P}\left(i_t^{\prime}=i, E_i(t) \mid \mathcal{F}_{t-1}\right) \\
& \geq \frac{p_{i^*(t),t,\tau}^i}{1-p_{i^*(t),t,\tau}^i} \mathbb{P}\left(I_t=i, E_i(t) \mid \mathcal{F}_{t-1}\right),
\end{aligned}
$$
where in the first equality we used the fact that $\theta_{i^*(t),t,\tau}$ is conditionally independent of $i_t^{\prime}$ and $E_i(t)$ given $\mathcal{F}_{t-1}$. In the second inequality, we used the fact that:
$$
\mathbb{P}\left(I_t=i, E_i(t) \mid \mathcal{F}_{t-1}\right) \leq\left(1-\mathbb{P}\left(\theta_{i^*(t),t,\tau}>y_{i,t} \mid \mathcal{F}_{t-1}\right)\right) \mathbb{P}\left(i_t^{\prime}=i, E_i(t) \mid \mathcal{F}_{t-1}\right)
$$
which is true since $\left\{I_t=i\right\} \cap E_i(t)  \subseteq\left\{i_t^{\prime}=i\right\} \cap E_i(t) \cap\left\{\theta_{i^*(t),t,\tau}\leq y_{i,t}\right\}$, and the two intersected events are conditionally independent given $\mathcal{F}_{t-1}$. Therefore, we have:
$$
\begin{aligned}
\mathbb{P}\left(I_t=i, E_i(t) \mid \mathcal{F}_{t-1}\right) & \leq\left(\frac{1}{p_{i^*(t),t,\tau}^i}-1\right) \mathbb{P}\left(I_t=i^*(t), E_i(t) \mid \mathcal{F}_{t-1}\right) \\
& \leq\left(\frac{1}{p_{i^*(t),t,\tau}^i}-1\right) \mathbb{P}\left(I_t=i^*(t) \mid \mathcal{F}_{t-1}\right),
\end{aligned}
$$
substituting, we obtain:
\begin{align}
    \sum_{t=K\Gamma+1}^T\mathbb{E}[\Prob(I_t=i,E_i(t)\mid\mathcal{F}_{t-1})]&\leq \mathbb{E}\left[\sum_{t=K\Gamma+1}^T\left(\frac{1}{p_{i^*(t),t,\tau}^i}-1\right) \mathbb{P}\left(I_t=i^*(t) \mid \mathcal{F}_{t-1}\right)\right]\\
    &=\mathbb{E}\left[\mathbb{E}\left[\sum_{t=K\Gamma+1}^T\left(\frac{1}{p_{i^*(t),t,\tau}^i}-1\right)\mathds{1}\left\{I_t=i^*(t)\right\}\mid\mathcal{F}_{t-1}\right]\right]\\
    &=\mathbb{E}\left[\sum_{t=K\Gamma+1}^{T} \left(\frac{1}{p_{i^*(t),t,\tau}^i}-1\right)\mathds{1}\left\{I_t=i^*(t)\right\}\right].
\end{align}
The statement follows by summing all the terms.
\end{proof}
\mf*
\begin{proof}
 Let us assume, without loss of generality, the arm $1$ is the optimal arm at round $t \in \dsb{T}$. Let $N_{1,t} = j$, $S_{1,t}= s$. Then,
as shown by~\citet{agrawal2012analysis}, $p_{1,t}^i$ can be written as:
    $$p_{1,t}^i = \Prob(\theta_{1,t} > y_{i,t}) = F^B_{j+1,y_{i,t}}(s),$$
    where $F^B_{j+1,y_{i,t}}(s)$ is the cumulative distribution function of a Binomial random variable after $j+1$ Bernoulli trials each with probability of success $y_{i,t}$ evaluated in $s$.
    Our goal is to prove the inequality:
\begin{equation}\label{eq:goal}
	\E_{X' \sim \text{PB($\underline{\mu}_1(j)$)}}\left[\frac{1}{F^B_{j+1,y_{i,t}} (X')} \right]\leq \E_{X \sim \text{Bin}(j,x)}\left[\frac{1}{F^B_{j+1,y_{i,t}} (X)} \right].
\end{equation}  
    We notice that the probability mass function of a binomial distribution is discrete log-concave (Lemma~\ref{lemma:con}). Thus, let $Y\sim \text{Bin}(j+1,y_{i,t})$, we have that for $i \in \dsb{0,j-1}$ it holds  that:
\begin{equation}\label{eq;conbin}
	p_Y(i+1)^2 \geq p_Y(i) p_Y(i+2).
\end{equation}
By Lemma~\ref{lemma:logcon} used with $\alpha = 1$ and $r = +\infty$, and $q$ being the probability mass function of the binomial distribution (precisely, $q(x) = p(-x)$ in Lemma~\ref{lemma:logcon}), we find that the CDF of the binomial is discrete log-concave on $\mathbb{Z}$ too. Indeed, according to Lemma~\ref{lemma:logcon}, if the probability mass function of an integer-valued random variable is discrete log-concave as a function on $\mathbb{Z}$, then the corresponding CDF ($F^B_{j+1,y_{i,t}}$ in our notation) is also discrete log-concave as a function on $\mathbb{Z}$.
Thus, omitting superscripts and subscripts,  $1 / F$
is discrete log-convex on the set $ x \in S\coloneqq \dsb{0,j+1}$, i.e.:
\begin{equation}
	\left(\frac{1}{F(x+1)}\right)^2\leq\frac{1}{F(x+2)}\frac{1}{F(x)},
\end{equation}
or, equivalently:
\begin{equation}
	\frac{1}{F(x+1)}\leq \left(\frac{1}{F(x+2)}\right)^{\frac{1}{2}}\left(\frac{1}{F(x)}\right)^{\frac{1}{2}}.
\end{equation}
Using the AM-GM inequality, we obtain:
\begin{align}
	\frac{1}{F(x+1)} <  \frac{1}{2}\left(\frac{1}{F(x+2)}\right)+ \frac{1}{2}\left(\frac{1}{F(x)}\right).
\end{align}
Notice the inequality is strict since the AM-GM inequality holds with equality only when all elements are equal and this is not our case since  $\frac{1}{F(x+2)} < \frac{1}{F(x)}$ for $x\in\dsb{0,j-1}$, and hence $\frac{1}{F(x+2)} \neq \frac{1}{F(x)}$. 
Thus, we have proved that $1 / F$ is strictly discrete convex on $S$. Therefore, using Lemma~\ref{lemma:hoeffpomp} we obtain that for every $j$, being the number of trials for both the Poisson-binomial and the binomial distributions, the expected value of the term of our interest for a Poisson-binomial distribution with a certain mean of the probabilities of the success at each trial, namely $\overline{\mu}_{1,j}=\frac{1}{j} {\sum_{l=1}^j\mu_{1,l}}$, is always smaller than the one of a binomial distribution where each Bernoulli trial has a probability of success equal to $\overline{\mu}_{1,j}$. More formally:
\begin{equation}
	\E_{X' \sim \text{PB($\underline{\mu}_1(j)$)}}\left[\frac{1}{F^B_{j+1,y_{i,t}}(X')} \right]\leq \E_{X'' \sim \text{Bin}(j,\overline{\mu}_{1,j})}\left[\frac{1}{F^B_{j+1,y_{i,t}}(X'')} \right].
\end{equation}  
To show Equation~\eqref{eq:goal} for any $j$ such that $\overline{\mu}_{1,j}\ge x$, we need to prove that the expected value of $1/{F_{j+1, y_{i,t}}^B}$ considered for a binomial process with mean $\overline{\mu}_{1,j}$ is smaller than the expected value of $1/{F_{j+1, y_{i,t}}^B}$ for a binomial distribution with mean $x\leq\overline{\mu}_{1,j}$. 
We apply Lemma~\ref{lemma:change} stating that for a non-negative random variable (like ours $1/{F_{j+1, y_{i,t}}^B}$), the expected value can be computed as:
\begin{align}
	\mathbb{E}_{\text{Bin}(j,\overline{\mu}_{1,j})}\left[\frac{1}{F_{j+1, y_{i,t}}^B}\right]=\int_{0}^{+\infty}\Prob\left(\frac{1}{F_{j+1, y_{i,t}}^B}> y\right) \de y.
\end{align}
Let $X''\sim \text{Bin}(j, \overline{\mu}_{1,j})$. Thus, we have:
\begin{align}
	\Prob\left(\frac{1}{F_{j+1, y_{i,t}}^B} > y \right)& = \Prob(X''=0)+\Prob(X''=1)+\ldots+\Prob\left(X''=\left(\frac{1}{F_{j+1,y_{i,t}}^B}\right)^{-1}(y) -1\right)\\
	& =\Prob\left(X'' \leq \underbrace{\left(\frac{1}{F_{j+1,y_{i,t}}^B}\right)^{-1}(y)-1}_{\text{$\eqqcolon k_j(y)$}} \right),
\end{align}
and the same goes for $X\sim \text{Bin}(j,x)$:
\begin{align}
	\Prob\left(\frac{1}{F_{j+1, y_i}^B} > y \right)& = \Prob(X=0)+\Prob(X=1)+\ldots+\Prob\left(X=\left(\frac{1}{F_{j+1,y_{i,t}}^B}\right)^{-1}(y)-1 \right)\\
	& =\Prob\left(X\leq \underbrace{\left(\frac{1}{F_{j+1,y_{i,t}}^B}\right)^{-1}(y)-1}_{\text{$\eqqcolon k_j(y)$}} \right),
\end{align}
where the inverse is formally defined as follows:
\begin{align}
	{\left(\frac{1}{F_{j+1,y_{i,t}}^B}\right)^{-1}(y)} \coloneqq \min\left\{ s \in \dsb{0,j}\, :\, y \ge \frac{1}{F_{j+1,y_{i,t}}^B(s)} \right\}.
\end{align}
Notice, that whenever such $s$ in the above definition does not exist, we will have that $\Prob\left(X<{\left(\frac{1}{F_{j+1,y_{i,t}}^B}\right)^{-1}(y)}\right)=0$, so that those values does not contribute to the integral.
Thus, we need to prove:
\begin{align}\label{eq:condtion2}
	\mathbb{E}_{\text{Bin}(j,\overline{\mu}_{1,j})}\left[\frac{1}{F_{j+1, y_{i,t}}^B}\right]= \int_{0}^{+\infty}\Prob(X''\leq k_j(y))\de y&\leq \int_{0}^{+\infty}\Prob(X\leq k_j(y))\de y = \mathbb{E}_{\text{Bin}(j,x)}\left[\frac{1}{F_{j+1, y_{i,t}}^B}\right].
\end{align}
A sufficient condition to ensure that the condition in Equation~\eqref{eq:condtion2} holds is that:
\begin{equation} \label{eq:condition3}
	\Prob(X''> m)\ge \Prob(X> m), \quad \forall m \in \mathbb{R}.
\end{equation}
Indeed, we have
\begin{align}
    &\int_{0}^{+\infty}\Prob(X''\leq k_j(y))\de y\leq \int_{0}^{+\infty}\Prob(X\leq k_j(y))\de y\\
    &\iff \int_{0}^{+\infty}(1-\Prob(X''> k_j(y)))\de y\leq \int_{0}^{+\infty}(1-\Prob(X > k_j(y)))\de y\\
    &\iff \int_{0}^{+\infty}\Prob(X''> k_j(y))\de y\ge \int_{0}^{+\infty}\Prob(X> k_j(y))\de y\\
    & \iff  \int_{0}^{+\infty}\big(\Prob(X''>k_j(y))-\Prob(X > k_j(y))\big)\de y \ge 0.
\end{align}
Let us recall the concept of stochastic order (\citet{boland2002stochastic,boland2004sstochasticordertesting,marshall2011inequalities}) that is often useful in comparing random variables. For two random variables $U$ and $V$, we say that $U$ is \emph{greater} than $V$ in the usual stochastic order (or $U$  \emph{stochastically dominates} $V$), and we denote it with $U\ge_{\text{st}}V$, when $\Prob(U> m) \ge \Prob(V >m)$ for every $ m \in \mathbb{R}$.
Thus, if we have that $X'' \ge_{\text{st}} X$ we would have that also Equation~\eqref{eq:condition3} holds too. It has been shown by ~\citet{boland2002stochastic} (Lemma \ref{lemma:stochastic}) that the condition for that to happen when $X''$ and $X$ are binomial distributions  with means $\mu''$ and $\mu$ is that $\mu''\ge \mu$.
Thus, we have showed that for any $j$ such that $\overline{\mu}_{1,j}\ge x$:
\begin{equation}
	\mathbb{E}_{X' \sim\text{PB($\underline{\mu}_1(j)$)}}\left[\frac{1}{F_{j+1,y_{i,t}}(X')}\right]\leq \mathbb{E}_{X'' \sim\text{Bin}(j,\overline{\mu}_{1,j})}\left[\frac{1}{F_{j+1,y_{i,t}}(X'')}\right]\leq \mathbb{E}_{X \sim \text{Bin($j,x$)}}\left[\frac{1}{F_{j+1,y_{i,t}}(X)}\right].
\end{equation}  
This concludes the proof.
\end{proof}

\subsection{Additional Lemmas}\label{apx:swgts}
\begin{restatable}[Expected Number of Pulls Bound for  \texttt{$\gamma$-ET-SWGTS}]{lemma}{rdg} \label{lem:RDg}
   Let $T \in \mathbb{N}$ be the learning horizon, $\tau \in \dsb{T}$ be the window size, $\Gamma \in \dsb{T}$ be the forced exploration parameter, for the  \texttt{$\gamma$-ET-SWGTS} algorithm the following holds for every $i \neq i^*(t)$ and free parameters $\omega \in \dsb{T}$ and $\epsilon > 0$:
     \begin{align*}
        \mathbb{E}[N_{i,T}] & \leq \Gamma+\frac{1}{\epsilon}+\frac{T}{\tau}+\frac{\omega T}{\tau}+{\mathbb{E}\left[\sum_{t=K\Gamma+1}^{T} \mathds{1}\left \{ p_{i,t,\tau}^i>\frac{1}{T\epsilon}, \ I_t=i, \ N_{i,t,\tau}\ge \omega,\right\}\right]} \\
        & \quad +{\mathbb{E}\left[\sum_{t=K\Gamma+1}^{T}\left(\frac{1}{p_{i^*(t),t,\tau}^i}-1\right)\mathds{1}\left\{I_t=i^*(t) \right\}\right]}.
    \end{align*}
\end{restatable}
\begin{proof}
    We define the event $E_i(t)\coloneqq \{\theta_{i,t,\tau}\leq y_{i,t}\}$. Thus, the following holds:
\begin{equation}
    \mathbb{E}_{\tau}[N_{i,T}]=\sum_{t=1}^{T}\Prob(I_t=i)\leq \Gamma+\underbrace{\frac{T}{\tau}}_{\text{(X)}}+\underbrace{\sum_{t=K\Gamma+1}^T \Prob(I_t=i, E_i^\complement(t))}_{\text{(A)}}+\underbrace{\sum_{t=K\Gamma+1}^T \Prob(I_t=i, E_i(t))}_{\text{(B)}},
\end{equation}
where (X) is the term arising given by the forced play whenever $N_{i,t,\tau}=0$. 
Let us first face term (A): 
\begin{align}
    \text{(A)} &\leq\sum_{t=K\Gamma+1}^T \Prob(I_t=i, E_i^\complement(t), N_{i,t,\tau}\leq \omega)+\sum_{t=K\Gamma+1}^T\Prob(I_t=i, E_i^\complement(t), N_{i,t,\tau}\ge \omega)\\
    &\leq \sum_{t=K\Gamma+1}^T\Prob(I_t=i, N_{i,t,\tau}\leq \omega)+\sum_{t=K\Gamma+1}^T\Prob(I_t=i, E_i^\complement(t), N_{i,t,\tau}\ge \omega)\\
    &\leq \mathbb{E}\left[\sum_{t=K\Gamma+1}^{T}\mathds{1}\left\{I_t=i, N_{i,t,\tau}\leq \omega\right\}\right]+\sum_{t=K\Gamma+1}^T\Prob(I_t=i, E_i^\complement(t), N_{i,t,\tau}\ge \omega)\\
&\leq\mathbb{E}\left[\underbrace{\sum_{t=1}^{T}\mathds{1}\left\{I_t=i, N_{i,t,\tau}\leq \omega\right\}}_{\text{(C)}}\right]+\sum_{t=K\Gamma+1}^T\Prob(I_t=i, E_i^\complement(t), N_{i,t,\tau}\ge \omega).
\end{align}
Observe that (C) can be bounded by Lemma \ref{lemma:window}. Thus, the above inequality can be rewritten as:
\begin{align}
    \text{(A)} &\leq \frac{\omega T}{\tau}+\underbrace{\sum_{t=1}^{T}\Prob(I_t=i, E_i^\complement(t), N_{i,t,\tau}\ge \omega)}_{\text{(D)}}.
\end{align}
We now focus on the term (D). Defining $\mathcal{T}:=\{t \in \dsb{K\Gamma+1,T}:1-\Prob(\theta_{i,t,\tau}\leq y_{i,t}\mid \mathcal{F}_{t-1})> \frac{1}{T\epsilon}, N_{i,t,\tau}\ge \omega \}$ and $\mathcal{T}':=\{t \in \dsb{K\Gamma+1,T}:1-\Prob(\theta_{i,t,\tau}\leq y_{i,t}\mid \mathcal{F}_{t-1})\le \frac{1}{T\epsilon}, N_{i,t,\tau}\ge \omega \}$ we obtain:
\begin{align}
    \sum_{t=K\Gamma+1}^{T}&\Prob(I_t=i, E_i^\complement(t), N_{i,t,\tau}\ge \omega) = \mathbb{E}\left[\sum_{t=K\Gamma+1}^T\mathds{1}\{I_t=i, E_i^\complement(t), N_{i,t,\tau}\ge \omega\}\right]\\
    &=\mathbb{E}\left[\sum_{t \in \mathcal{T}}\mathds{1}\{I_t=i,E_i(t)^\complement\}\right]+\mathbb{E}\left[\sum_{t \in \mathcal{T}'}\mathds{1}\{I_t=i,E_i(t)^\complement\}\right]\\
    &\leq \mathbb{E}\left[\sum_{t \in \mathcal{T}}\mathds{1}\{I_t=i\}\right]+\mathbb{E}\left[\sum_{t \in \mathcal{T}'}\mathds{1}\{E_i(t)^\complement\}\right]\\
    &\leq \mathbb{E}\left[\sum_{t=K\Gamma+1}^T\mathds{1}\left\{1-\Prob(\theta_{i,t,\tau}\leq y_{i,t}\mid \mathcal{F}_{t-1})> \frac{1}{T\epsilon}, N_{i,t,\tau}\ge \omega, I_t=i \right\}\right]+\sum_{t=1}^T\frac{1}{T\epsilon}.
\end{align}
Term (B) is bounded exactly as in the proof of Lemma~\ref{lem:RD}. The statement follows by summing all the terms.
\end{proof}

\subsection{Proofs of Section~\ref{sec:AnalysisBetaTS}}
\wald*
\begin{proof}
	We start with the usual definition of regret and proceed as follows:
	\begin{align}
		R(\mathfrak{A},T) & = T \overline{\mu}_{i^*(T)}(T) - \E \left[ \sum_{t=1}^T  \mu_{I_t}(N_{I_t,t}) \right] \\
		& =  \E \left[ \sum_{t=1}^T \left(  \mu_{i^*(T)}(t) - \mu_{I_t}(N_{I_t,t}) \right) \right] \\
		& = \E \left[ \sum_{t=1}^T  \mu_{i^*(T)}(t) - \sum_{j=1}^{N_{i^*(T),T}} \mu_{i^*(T)}(j) - \sum_{i\neq i^*(T)} \sum_{j=1}^{N_{i,T}} \mu_{i}(j) \right] \\
		& = \E \left[ \sum_{j=N_{i^*(T),T}+1}^T  \mu_{i^*(T)}(j) - \sum_{i\neq i^*(T)} \sum_{j=1}^{N_{i,T}} \mu_{i}(j) \right] \\
		& \le \E \left[ \sum_{i\neq i^*(T)} \sum_{j=1}^{N_{i,T}} \left( \mu_{i^*(T)}(T) -  \mu_{i}(j)  \right) \right] \\
		& \le \sum_{i \neq i^*(T)} \left( \mu_{i^*(T)}(T) - \mu_i(1) \right) \E \left[ \sum_{j=1}^{N_{i,T}} 1  \right].
	\end{align}
    The result follows from the definition of $\Delta_i(T,1)$.
\end{proof}

\ts*

\begin{proof}
In order to bound the expected value of the pulls of the suboptimal arm $i \neq i^*(T)$ at the time horizon $T$, we use Lemma \ref{lem:RD}, considering the forced exploration and imposing the window length $\tau=T$. Formally, for \texttt{Beta-ET-TS}, we have:
\begin{align}
    \mathbb{E}[N_{i,T}] & \leq \Gamma+1+\frac{\omega T}{T}+\underbrace{\mathbb{E}\left[\sum_{t=K\Gamma+1}^T \mathds{1}\left\{p_{i,t,T}^i>\frac{1}{T}, \ I_t=i, \ N_{i,t}\ge \omega\right\}\right]}_{\text{(A)}} \nonumber\\
    & \quad +\underbrace{\mathbb{E}\left[\sum_{t=K\Gamma+1}^{T} \left(\frac{1}{p_{i^*(t),t,T}^i}-1\right)\mathds{1}\left\{I_t=i^*(t)\right\}\right]}_{\text{(B)}}.
\end{align}
In the following, we consider, without loss of generality, the first arm to be the best, i.e., for every round $t \in \dsb{T}$ we have $i^*(t)=i^*(T)=1$ and introduce the following quantities, defined for some $\epsilon \in (0,1]$: $x_{i,t}=x_i \in (\overline{\mu}_i(T),\overline{\mu}_1(\sigma))$ such that $d(x_i, \overline{\mu}_1(\sigma)) = d(\overline{\mu}_i(T),\overline{\mu}_1(\sigma))/(1+\epsilon)$, $y_{i,t}=y_i \in (x_i,\overline{\mu}_1(\sigma))$ such that $d(x_i, y_i) = d(x_i,\overline{\mu}_1(\sigma))/(1+\epsilon) = d(\overline{\mu}_i(T),\overline{\mu}_1(\sigma))/(1+\epsilon)^2$, $\omega=\frac{\log(T)}{d(x_i,y_i)} = (1+\epsilon)^2 \frac{\log(T)}{d(\overline{\mu}_i(T),\overline{\mu}_1(\sigma))}$. We define $\hat{\mu}_{i,t,T}=\frac{S_{i,t}}{N_{i,t}}$.

We further decompose term $\mathcal{A}$ in two contributions:
\begin{align}
\text{(A)} & =\underbrace{\mathbb{E}\left[\sum_{t=K\Gamma+1}^T\mathds{1}\left\{{p_{i,t,T}^i>\frac{1}{T},\hat{\mu}_{i,t,T}\leq x_i, I_t=i,\ N_{i,t}\ge \omega}\right\}\right]}_{\text{(A1)}}\nonumber\\
& \quad +\underbrace{\mathbb{E}\left[\sum_{t=K\Gamma+1}^T\mathds{1}\left\{p_{i,t,T}^i>\frac{1}{T},\hat{\mu}_{i,t,T}> x_i, I_t=i,\ N_{i,t}\ge \omega\right\}\right]}_{\text{(A2)}}
\end{align}
\paragraph{Term (A2)} Focusing on (A2), let $\tau_j^i$ denote the round in which arm $i$ is played for the $j$-th time:
\begin{align}
   \text{(A2)} &\leq \mathbb{E}\left[\sum_{t=K\Gamma+1}^T \mathds{1}\left\{\hat{\mu}_{i,t,T}> x_i ,  I_t=i,  N_{i,t}\ge \omega\right\}\right]\\
    & \le \mathbb{E}\left[\sum_{j=K\Gamma}^T \mathds{1}\left\{\hat{\mu}_{i,\tau_j^i+1,T}> x_i\right\} \sum_{t=\tau_j^i+1}^{\tau_{j+1}^i} \mathds{1} \left\{ I_t=i \right\}\right]\\
    &\le\sum_{j=\Gamma}^T\Prob(\hat{\mu}_{i,\tau_j^i+1,T}> x_i\big)\\
    &=\sum_{j=\Gamma}^T\Prob(\hat{\mu}_{i,\tau_j^i+1,T}-\mathbb{E}[\hat{\mu}_{i,\tau_j^i+1,T}]> x_i-\mathbb{E}[\hat{\mu}_{i,\tau_j^i+1,T}])\\
    &\leq \sum_{j=\Gamma}^T \exp{(-j d ( x_i,\mathbb{E}[\hat{\mu}_{i,\tau_{j}^i+1,T}]))} \\
    &\leq \sum_{j=0}^T \exp{(-j d ( x_i,\overline{\mu}_i(T)))},\label{eq:tsexp}
\end{align}
where in the last but one inequality, we have exploited the Chernoff-Hoeffding bound (Lemma \ref{lemma:chernoff} with $\lambda = x_i - \mathbb{E}[\hat{\mu}_{i,\tau_j^i+1,T}]$), and in the last one, we observed that $\mathbb{E}[\hat{\mu}_{i,\tau_j^i+1,T}] \le \overline{\mu}_i(T)$. Thus, we have:
\begin{align}
\sum_{j=0}^T \exp{(-j d ( x_i,\overline{\mu}_i(T)))} & \le 1 + \int_{j=0}^{+\infty} \exp{(-j d ( x_i,\overline{\mu}_i(T)))} \de j \\
& = 1 + \frac{1}{d ( x_i,\overline{\mu}_i(T))} \\
& \le 1 + \frac{1}{2(x_i-\overline{\mu}_i(T))^2},
\end{align} 
having used Pinsker's inequality.
Following \cite{agrawal2017near}, we have the inequality:
\begin{align}
    x_i - \overline{\mu}_i(T) \ge \frac{\epsilon}{1+\epsilon} \cdot \frac{d(\overline{\mu}_i(T),\overline{\mu}_1(\sigma))}{\log \frac{\overline{\mu}_1(\sigma)(1-\overline{\mu}_i(T))}{\overline{\mu}_i(T)(1-\overline{\mu}_1(\sigma))}}.
\end{align}
From this, we have:
\begin{align}
    \sum_{s=0}^T \exp{(-s d ( x_i,\overline{\mu}_i(T)))} \le 1 + \left(\log \frac{\overline{\mu}_1(\sigma)(1-\overline{\mu}_i(T))}{\overline{\mu}_i(T)(1-\overline{\mu}_1(\sigma))} \right)^2  \cdot \frac{(1+\epsilon)^2}{2\epsilon^2} \cdot \frac{1}{d(\overline{\mu}_i(T),\overline{\mu}_1(\sigma))^2}.
\end{align}
\paragraph{Term (A1)}
Let us focus on the term (A1) of the regret. 
\begin{align}
   \text{(A1)} \leq \mathbb{E}\left[\sum_{t=K\Gamma+1}^T\mathds{1}\left\{\underbrace{ p_{i,t,T}^i>\frac{1}{T}, N_{i,t}\ge \omega,\hat{\mu}_{i,t,T}\leq x_i}_{\text{(C)}} \right\}\right].
\end{align}
We wish to evaluate if the condition (C) ever occurs.
To this end, let $(\mathcal{F}_{t-1})_{t \in \dsb{T}}$ be the canonical filtration. We have:
\begin{align}
	&\Prob(\theta_{i,t,T}>y_{i}|\hat{\mu}_{i,t,T} \leq x_{i}, N_{i,t}\ge \omega,\mathcal{F}_{t-1})=
	\nonumber\\ & = \Prob \left( \text{Beta} \left( \hat{\mu}_{i,t,T} N_{i,t} + 1, (1 - \hat{\mu}_{i,t,T}) N_{i,t} + 1 \right) > y_{i} | \hat{\mu}_{i,t,T} \leq x_{i},N_{i,t}\ge \omega \right) \label{line:line}\\
	& \leq \Prob \left( \text{Beta} \left( x_{i} N_{i,t} + 1, (1 - x_{i}) N_{i,t} + 1 \right) > y_{i} |N_{i,t}\ge \omega\right)\label{eq:bbi} \\
	& \leq F^{B}_{N_{i,t}+1,y_{i}}\big(x_{i}N_{i,t}|N_{i,t}\ge \omega\big) \leq F^{B}_{N_{i,t},y_{i}}\big(x_{i}N_{i,t}|N_{i,t}\ge \omega\big)\label{eq:bborder}\\
    &\leq\exp \left( - (N_{i,t}) d(x_{i}, y_{i})|N_{i,t}\ge \omega \right)\label{eq:ch}\\
    &\leq \exp \left( - \omega  d(x_{i}, y_{i}) \right),\label{eq:tscondfin}
\end{align}
where \eqref{eq:ch} follows from the generalized Chernoff-Hoeffding bounds  (Lemma~\ref{lemma:chernoff}) and \eqref{eq:bbi} from the Beta-Binomial identity (Fact~\ref{lem:betabin}). Equation~\eqref{line:line} was derived by exploiting the fact that on the event $\hat{\mu}_{i,t,\tau}\leq x_i$ a sample from $\text{Beta} \left( x_{i} N_{i,t} + 1, (1 - x_{i}) N_{i,t} + 1 \right) $ is likely to be as large as a sample from $\text{Beta}( \hat{\mu}_{i,t,T} N_{i,t} + 1, (1 - \hat{\mu}_{i,t,T})N_{i,t} + 1 )$, as a $\text{Beta}(\alpha, \beta)$ random variable
is stochastically dominated by $\text{Beta}(\alpha', \beta')$ if $\alpha'\ge \alpha$ and $\beta'\leq \beta$ (Lemma \ref{lem:nbord}), and the second inequality in \eqref{eq:bborder} derives from Lemma \ref{lem:bborder}.
Therefore, since $\omega=\frac{\log(T)}{d(x_i,y_i)}$, we have:
\begin{equation}
	\Prob(\theta_{i,t,T}>y_{i}|\hat{\mu}_{i,t,T} \leq x_{i}, N_{i,t}\ge \omega,\mathbb{F}_{t-1})\leq\frac{1}{T}, \label{eq:tscond}
\end{equation}
in contradiction with condition (C), implying that (A1)$=0$.
\paragraph{Term (B)}
For this term, We have:
\begin{align}
   \text{(B)}= \mathbb{E}\left[\sum_{t=K\Gamma+1}^{T} \left(\frac{1}{p_{1,t,T}^i}-1\right)\mathds{1}\left\{I_t=1\right\}\right]
\end{align}
Let $\tau_j$ denote the time step at which arm $1$ is played for the $j$-th time:
\begin{align}
	  \text{(B)} &\leq \sum_{j=\Gamma}^{T-1}\mathbb{E}\bigg[\frac{1-p_{1,\tau_{j}+1,T}^i}{p_{1,\tau_{j}+1,T}^i}\sum_{t=\tau_j+1}^{\tau_{j+1}}\mathds{1}\{I_t=1\}\bigg]\\
	&\leq\sum_{j=\Gamma}^{T-1}\mathbb{E}\bigg[\frac{1-p_{1,\tau_{j}+1,T}^i}{p_{1,\tau_{j}+1,T}^i}\bigg], \label{eq:fixed}
\end{align}
where the inequality in Equation~\eqref{eq:fixed} uses the fact that $p_{1,t,T}^i$ is fixed, given $\mathcal{F}_{t-1}$. Then, we observe that $p_{1,t,T}^i = \Prob(\theta_{1,t,T} > y_i |\mathcal{F}_{t-1})$
changes only when the distribution of $\theta_{1,t,T}$ changes, that is, only on the time step after each play of the first arm. Thus, $p_{1,t,T}^i$ is the same at all time steps $t \in \{\tau_j+1, \dots , \tau_{j+1}\}$, for every $j$. Finally, bounding the probability of selecting the optimal arm by $1$ we have:
\begin{equation}
	\text{(B)} \leq \sum_{k=\Gamma}^{T-1}\mathbb{E} \left[ \frac{1}{p_{1,\tau_j+1,T}^i} - 1 \right].
\end{equation}

By definition $N_{1,\tau_j+1} = j$, and let $S_{1,t}= s$. Thus, we have:
$$p_{1,\tau_j+1,T}^i = \Prob(\theta_{1,\tau_j+1,T} > y_i) = F^B_{j+1,y_i}(s)$$
due to the relation that links the Beta and the binomial distributions (Fact 3 of~\cite{agrawal2017near}). Let $\tau_j + 1$ denote the time step after the $j$-th play of the optimal arm. Then, $N_{1,\tau_j + 1} =j$. We do notice a sensible difference with respect to the stationary case. Indeed, the number of successes after $j$ trial is not distributed anymore as a binomial distribution. Instead, it can be described by a Poisson-Binomial distribution $\text{PB}(\underline{\mu}_{1}(j))$ where the vector $\underline{\mu}_{1}(j)=(\mu_{1}(1),\ldots,\mu_{1}(j))$, and $\mu_{1}(m)$ represents the probability of success of the best arm at the $m$-th trial. The probability of having $s$ successful trials out of a total of $j$ trials can be written as follows~\citep{wang1993poissobinomial,poisson2019english}:
\begin{align}
	f_{j,\underline{\mu}_{1}(j)}(s)=\sum_{A \in F_s} \prod_{m \in A}\mu_{1}(m)\prod_{m' \in A^c}(1-\mu_{1}(m')),   
\end{align}
where $F_{s}$ is the set of all subsets of $s$ integers that can be selected from $\dsb{j}$. $F_s$ by definition will contain $\frac{j!}{(j-s)!s!}$ elements, the sum over which is infeasible to compute in practice unless the number of trials $j$ is small. A useful property of $f$ is that it is invariant to the order of the elements in $\underline{\mu}_{1}(j)$. Moreover, we define density function of the binomial of $j$ trials and mean $\overline{\mu}_{1}(j)$, i.e., $\text{Bin}(j,\overline{\mu}_{1}(j))$, as:
\begin{align}
	f_{j,\overline{\mu}_{1}(j)}(s) = \binom{j}{s} \overline{\mu}_{1}(j)^s (1-\overline{\mu}_{1}(j))^{j-s}.
\end{align}
Using the first inequality of Lemma \ref{lemma:techlemma}, we obtain:
\begin{equation}
    \sum_{s=0}^{j} \frac{f_{j,\underline{\mu}_{1}(j)}(s)}{F_{j+1, y_i}^B(s)}\leq \sum_{s=0}^{j} \frac{f_{j,\overline{\mu}_{1}(j)}(s)}{F_{j+1, y_i}^B(s)}
\end{equation}
For the case $j \ge \sigma$, we apply the second inequality of Lemma~\ref{lemma:techlemma}, to obtain:
\begin{align}
     \sum_{s=0}^{j} \frac{f_{j,\overline{\mu}_{1}(j)}(s)}{F_{j+1, y_i}^B(s)} \le  \sum_{s=0}^{j} \frac{f_{j,\overline{\mu}_{1}(\sigma)}(s)}{F_{j+1, y_i}^B(s)}.
\end{align}
Instead, for the case $j < \sigma$, by applying the change of measure argument of Lemma~\ref{lemma:changeMeasure}, we have:
\begin{align}
	\sum_{s=0}^{j} \frac{f_{j,\underline{\mu}_{1}(j)}(s)}{F_{j+1, y_i}^B(s)} \leq {\sum_{s=0}^{j} \frac{f_{j,\overline{\mu}_{1}(j)}(s)}{F_{j+1, y_i}^B(s)}} \leq \left(\frac{1}{(1-y_i)^{j+1}}-1\right)&\delta_{\text{TV}}\left(\text{Bin}(j,\overline{\mu}_{1}(j)),\text{Bin}(j,\overline{\mu}_{1}(\sigma))\right)+ \nonumber\\&+{\sum_{s=0}^{j} \frac{f_{j,\overline{\mu}_{1}(\sigma)}(s)}{F_{j+1, y_i}^B(s)}},
\end{align}
where $\delta_{TV}(P,Q) \coloneqq \sup_{A\in\mathcal{F}}{\big|P(A)-Q(A)\big|}$ is the total variation between the probability measures $P$ and $Q$, having observed that, using the notation of Lemma~\ref{lemma:changeMeasure}:
\begin{align}
	& b= \max_{s \in \dsb{0,j}} \frac{1}{F^B_{j+1,y_i}(s)} = \frac{1}{F^B_{j+1,y_i}(0)} =\frac{1}{\Prob(\text{Bin}(j+1,y_i)=0)} =  \frac{1}{(1-y_i)^{j+1}}, \\
	& a =\min_{s \in \dsb{0,j}} \frac{1}{F^B_{j+1,y_i}(s)} = \frac{1}{F^B_{j+1,y_i}(j)} =\frac{1}{\Prob(\text{Bin}(j+1,y_i)\le j)} \ge \frac{1}{\Prob(\text{Bin}(j+1,y_i)\le j+1)} = 1.  
\end{align}
Putting all together, we obtain:
\begin{align}
	&\sum_{s=0}^{j} \frac{f_{j,\underline{\mu}_{1}(j)}(s)}{F_{j+1, y_i}^B(s)}\leq \nonumber \\
    &\leq
	\begin{cases}
		\Big(\frac{1}{(1-y_i)^{j+1}}-1\Big)\delta_{\text{TV}}(\text{Bin}(j,\overline{\mu}_{1}(j)),\text{Bin}(j,\overline{\mu}_{1}(\sigma)))+ \sum_{s=0}^{j} \frac{f_{j,\overline{\mu}_{1}(\sigma)}(s)}{F_{j+1, y_i}^B(s)} & \text{if } 0 \le j<\sigma \\
		\sum_{s=0}^{j} \frac{f_{j,\overline{\mu}_{1}(\sigma)}(s)}{F_{j+1, y_i}^B(s)} & \text{if } j\ge\sigma \\
	\end{cases}.
\end{align}
Then, summing over $j$, we have:
\begin{align}
  &\sum_{j=\Gamma}^{T-1}\left(\sum_{s=0}^{j} \frac{f_{j,\underline{\mu}_{1}(j)}(s)}{F_{j+1, y_i}^B(s)}-1\right)  \leq \nonumber\\ &\leq\begin{cases}
      \sum_{j=\Gamma}^{\sigma-1}\Big(\frac{1}{(1-y_i)^{j+1}}-1\Big)\delta_{\text{TV}}(\text{Bin}(j,\overline{\mu}_{1}(j)),\text{Bin}(j,\overline{\mu}_{1}(\sigma)))+ \sum_{j=\Gamma}^{T-1}\left(\sum_{s=0}^{j} \frac{f_{j,\overline{\mu}_{1}(\sigma)}(s)}{F_{j+1, y_i}^B(s)}-1\right) & \text{if }  \Gamma<\sigma \\
      \\
      \sum_{j=\Gamma}^{T-1}\left(\sum_{s=0}^{j} \frac{f_{j,\overline{\mu}_{1}(\sigma)}(s)}{F_{j+1, y_i}^B(s)}-1\right) & \text{if }  \Gamma\ge\sigma 
  \end{cases},
\end{align}
where we recall that $y_i \le \overline{\mu}_1(\sigma)$.
From Lemma 2.9 by~\citet{agrawal2017near}, we have that:
\begin{align}
	\sum_{s=0}^{j} \frac{f_{j,\overline{\mu}_{1}(\sigma)}(s)}{F_{j+1, y_i}(s)}-1 \leq
	\begin{cases}
		\frac{3}{\Delta_i'} & \text{if } j<\frac{8}{\Delta_i'}\\
		\Theta\left(e^{-\frac{\Delta_i^{'2}j}{2}}+\frac{e^{-D_i j}}{(j+1)\Delta_i'^{2}}+\frac{1}{e^{\Delta_i'^{2}\frac{j}{4}}-1}\right) & \text{if } j\ge\frac{8}{\Delta_i'}
	\end{cases},\label{eq:lemma4}
\end{align}
where $\Delta_i'= \overline{\mu}_{1}(\sigma)-y_i$ and $D_i=d(y_i,\overline{\mu}_{1}(\sigma)) =y_i\log{\frac{y_i}{\overline{\mu}_{1}(\sigma)}}+(1-y_i)\log{\frac{1-y_i}{1-\overline{\mu}_{1}(\sigma)}}$. Thus, summing over $j$, we obtain:
\begin{align}
	\sum_{j=\Gamma}^{T-1}\left(\sum_{s=0}^{j}\frac{f_{j,\overline{\mu}_{1}(\sigma)}(s)}{F_{j+1, y_i}(s)}-1\right)& \leq \hspace{-0.4 cm} \sum_{j<\min\{\frac{8}{\Delta_i'}-\Gamma,0\}}\hspace{-0.2 cm}\frac{3}{\Delta_i'}+\hspace{-0.4 cm}\sum_{j\ge\max\{\Gamma,\frac{8}{\Delta_i'}\}}\hspace{-0.5 cm}\Theta \left(e^{-\frac{\Delta_i^{'2}j}{2}}+\frac{e^{-D_i j}}{(j+1)\Delta_i'^{2}}+\frac{1}{e^{\Delta_i'^{2}\frac{j}{4}}-1}\right)\\
    &\leq \sum_{j<\frac{8}{\Delta_i'}}\frac{3}{\Delta_i'}+\sum_{j\ge\frac{8}{\Delta_i'}}\Theta \left(e^{-\frac{\Delta_i^{'2}j}{2}}+\frac{e^{-D_i j}}{(j+1)\Delta_i'^{2}}+\frac{1}{e^{\Delta_i'^{2}\frac{j}{4}}-1}\right)\\
    &\leq\frac{24}{\Delta_i'^2}+\Theta\left(\frac{1}{\Delta_i'^2}+\frac{1}{\Delta_i'^2\sqrt{D_i}} + \frac{1}{\Delta_i'^3} \right),
\end{align}
having bounded the individual terms carefully:
\begin{align}
& \sum_{j\ge\frac{8}{\Delta_i'}} e^{-\frac{\Delta_i^{'2}j}{2}} \le \sum_{j=1}^{+\infty} e^{-\frac{\Delta_i^{'2}j}{2}} \le \int_{j=0}^{+\infty}  e^{-\frac{\Delta_i^{'2}j}{2}}  \de j = \frac{2}{\Delta_i^{'2}}, \\
& \sum_{j\ge\frac{8}{\Delta_i'}} \frac{e^{-D_i j}}{(j+1)\Delta_i'^{2}} \le  \frac{1}{8\Delta_i'} \sum_{j\ge\frac{8}{\Delta_i'}} e^{-D_i j}  \le \frac{1}{8\Delta_i'} \sum_{j=1}^{+\infty} e^{-D_i j} \le  \frac{1}{8\Delta_i'} \int_{j=0}^{+\infty} e^{-D_i j} \de j \le  \frac{1}{8\Delta_i'D_i},\\
& \sum_{j\ge\frac{8}{\Delta_i'}} \frac{1}{e^{\Delta_i'^{2}\frac{j}{4}}-1}  \le \int_{j=\frac{4}{\Delta_i'}}^{+\infty} \frac{1}{e^{\Delta_i'^{2}\frac{j}{4}}-1} \de j = \frac{4}{\Delta_i'^2} \log \frac{1}{1-e^{-\Delta_i'}} \le \frac{4}{\Delta_i'^2} \log \frac{1}{  \Delta_i'},
\end{align}
having bounded $1-e^{-x} \ge x$. First of all, let us observe that from Pinsker's inequality, we have $D_i \ge 2 \Delta_i'^2$. We now relate $\Delta_i'^2$ with $d(\overline{\mu}_i(T),\overline{\mu}_1(\sigma))$. From the mean value theorem:
\begin{align}
    \frac{\epsilon}{(1+\epsilon)^2} d(\overline{\mu}_i(T),\overline{\mu}_1(\sigma)) & = \frac{d(\overline{\mu}_i(T),\overline{\mu}_1(\sigma))}{1+\epsilon} - \frac{d(\overline{\mu}_i(T),\overline{\mu}_1(\sigma))}{(1+\epsilon)^2} \\
    & = d(x_i,\overline{\mu}_1(\sigma)) - d(x_i,y_i) \\
    & \le \max_{z \in [y_i,\overline{\mu}_1(\sigma)]} \frac{\partial d(x_i,z)}{\partial z} (\overline{\mu}_1(\sigma) - y_i) \\
    & \le \frac{\partial d(x_i,z)}{\partial z}\rvert_{z = \overline{\mu}_1(\sigma)} \Delta_i',
\end{align}
having exploited that the maximum derivative is attained in $z = \overline{\mu}_1(\sigma)$. Moreover, we have, from Pinsker's inequality:
\begin{align}
    \frac{\partial d(x_i,z)}{\partial z}\rvert_{z = \overline{\mu}_1(\sigma)} = \frac{\overline{\mu}_1(\sigma)-x_i}{\overline{\mu}_1(\sigma)(1-\overline{\mu}_1(\sigma))} \le  \frac{\sqrt{d(x_i,\overline{\mu}_1(\sigma))}}{\sqrt{2}\overline{\mu}_1(\sigma)(1-\overline{\mu}_1(\sigma))} = \frac{\sqrt{d(\overline{\mu}_1(T),\overline{\mu}_1(\sigma))}}{\sqrt{2(1+\epsilon)}\overline{\mu}_1(\sigma)(1-\overline{\mu}_1(\sigma))}.
\end{align}
Putting all together, we obtain:
\begin{align}
    \Delta_i'^2 \ge   \frac{2\epsilon^2}{(1+\epsilon)^3} (\overline{\mu}_1(\sigma)(1-\overline{\mu}_1(\sigma)))^2 d(\overline{\mu}_i(T),\overline{\mu}_1(\sigma)) = \frac{2\epsilon^2}{(1+\epsilon)^3} \widetilde{d}(\overline{\mu}_i(T),\overline{\mu}_1(\sigma)) .
\end{align}
Finally, we get the following bound:
\begin{align}
\sum_{j=\Gamma}^{T-1}\left(\sum_{s=0}^{j}\frac{f_{j,\overline{\mu}_{1}(\sigma)}(s)}{F_{j+1, y_i}(s)}-1\right) \le \Theta \left( \frac{1}{\epsilon^2 \widetilde{d}(\overline{\mu}_i(T),\overline{\mu}_1(\sigma))} \log \frac{1}{\epsilon^2 \widetilde{d}(\overline{\mu}_i(T),\overline{\mu}_1(\sigma))} \right).
\end{align}
Putting all together, we have:
\begin{equation}\label{eq:fullExpression}
\begin{aligned}
\E[N_{i,T}] & \le \Gamma + 2 + (1+\epsilon)^2 \frac{\log(T)}{d(\overline{\mu}_i(T),\overline{\mu}_1(\sigma))} \\
& +  \left(\log \frac{\overline{\mu}_1(\sigma)(1-\overline{\mu}_i(T))}{\overline{\mu}_i(T)(1-\overline{\mu}_1(\sigma))} \right)^2  \cdot \frac{(1+\epsilon)^2}{2\epsilon^2} \cdot \frac{1}{d(\overline{\mu}_i(T),\overline{\mu}_1(\sigma))^2} \\
& +  \Theta \left( \frac{1}{\epsilon^2 \widetilde{d}(\overline{\mu}_i(T),\overline{\mu}_1(\sigma))} \log \frac{1}{\epsilon^2 \widetilde{d}(\overline{\mu}_i(T),\overline{\mu}_1(\sigma))} \right) \\
& + \sum_{j=\Gamma}^{\sigma-1}\frac{\delta_{\mathrm{TV}}(\mathrm{Bin}(j,\overline{\mu}_{1}(j)),\mathrm{Bin}(j,\overline{\mu}_{1}(\sigma)) }{(1-\overline{\mu}_{1}(\sigma))^{j+1}}.
\end{aligned}
\end{equation}
By setting $\epsilon'=3\epsilon$, we get the result.
\end{proof}
\gts*
{
\begin{proof}
Using Lemma \ref{lem:RDg} to bound the expected value of pulls of the suboptimal arm when the window length is equal to the time horizon we have:
     \begin{align*}
        \mathbb{E}_{\tau}[N_{i,T}]& \leq \Gamma+\frac{1}{\epsilon_i}+\frac{T}{T}+\frac{\omega T}{T}\hspace{-0.12 cm} +\hspace{-0.12 cm}\underbrace{\mathbb{E}\left[\sum_{t=K\Gamma+1}^{T} \hspace{-0.3 cm}\mathds{1}\left \{ p_{i,t,T}^i>\frac{1}{T\epsilon_i}, I_t=i, N_{i,t}\ge \omega\right\}\right]}_{\text{(A)}}\\
        & \quad + \underbrace{\mathbb{E}\left[\sum_{t=K\Gamma+1}^{T}\left(\frac{1}{p_{i^*(t),t,T}^i}-1\right)\mathds{1}\left\{I_t=i^*(t) \right\}\right]}_{\text{(B)}}.
    \end{align*}
In the following, we consider, without loss of generality, the first arm to be the best, i.e., for each round $t \in \dsb{T}$ we will have $i^*(t)=i^*(T)=1$. Furthermore, we set $y_{i,t}=y_i=\overline{\mu}_1(\sigma)-\frac{\overline{\Delta}_i(\sigma,T)}{3}$, $x_{i,t}=x_i=\overline{\mu}_i(T)+\frac{\overline{\Delta}_i(\sigma,T)}{3}$, $\omega=\frac{32\log{(T\overline{\Delta}_i(\sigma,T)^2+e^6)}}{\gamma(y_i-x_i)^2}=\frac{288\log{(T\overline{\Delta}_i(\sigma,T)^2+e^6)}}{\gamma\overline{\Delta}_i(\sigma,T)^2}$, $\epsilon_i=\overline{\Delta}_i(\sigma,T)^2$. Finally, we define ${\overline{\hat{\mu}}_{i,t,T}}=\frac{S_{i,t}}{N_{i,t}}$. We rewrite term (A) as a contribution of two different terms.
\begin{align}
\text{(A)}& =\underbrace{\mathbb{E}\left[\sum_{t=K\Gamma+1}^{T} \hspace{-0.3 cm}\mathds{1}\left \{ p_{i,t,T}^i>\frac{1}{T\epsilon_i},\overline{\hat{\mu}}_{i,t,T}\leq x_i, I_t=i, N_{i,t}\ge \omega\right\}\right]}_{\text{(A1)}} \nonumber\\
& \quad +\underbrace{\mathbb{E}\left[\sum_{t=K\Gamma+1}^{T} \hspace{-0.3 cm}\mathds{1}\left \{ p_{i,t,T}^i>\frac{1}{T\epsilon_i},\overline{\hat{\mu}}_{i,t,T}>x_i, I_t=i, N_{i,t}\ge \omega\right\}\right]}_{\text{(A2)}}.
\end{align}
We focus first on term (A2).
\paragraph{Term (A2)}
  Focusing on (A2), let $\tau_j^i$ denote the round in which arm $i$ is played for the $j$-th time:
\begin{align}
   \text{(A2)} &\leq \mathbb{E}\left[\sum_{t=K\Gamma+1}^T \mathds{1}\left\{\hat{\mu}_{i,t,T}> x_i ,  I_t=i,  N_{i,t}\ge \omega\right\}\right]\\
    & \le \mathbb{E}\left[\sum_{j=\omega}^T \mathds{1}\left\{\hat{\mu}_{i,\tau_j^i+1,T} >x_i\right\} \sum_{t=\tau_j^i+1}^{\tau_{j+1}^i} \mathds{1} \left\{ I_t=i \right\}\right]\\
    &\le\sum_{j=\omega}^T\Prob(\hat{\mu}_{i,\tau_j^i+1,T}> x_i\big)\\
    &=\sum_{j=\omega}^T\Prob(\hat{\mu}_{i,\tau_j^i+1,T}-\mathbb{E}[\hat{\mu}_{i,\tau_j^i+1,T}]> x_i-\mathbb{E}[\hat{\mu}_{i,\tau_j^i+1,T}])\\
    &\leq \sum_{j=\omega}^T \exp{(-j d ( x_i,\mathbb{E}[\hat{\mu}_{i,\tau_{j}^i+1,T}]))} \\
     &\leq \sum_{j=\omega}^T \exp{(-2\omega   ( x_i-{\mu}_i(T))^2)},
\end{align}
where the last inequality follows from the Chernoff-Hoeffding bound (Lemma \ref{lemma:chernoff}). By definition, we will have that $x_i-\overline{\mu}_i(T)=y_i-x_i$. So, substituting in the last inequality, we obtain:
\begin{align}
    \sum_{j=\omega}^T \exp{(-2\omega(x_i-\overline{\mu}_i(T))^2)}&=\sum_{j=\omega}^T \exp{\left(-2\frac{32\log(T\overline{\Delta}_i(\sigma,T)^2+e^6)}{(y_i-x_i)^2}(x_i-\overline{\mu}_i(T))^2\right)}\\
    &\leq \frac{1}{\overline{\Delta}_i(\sigma,T)^2}.
\end{align}
\paragraph{Term (A1)} Focusing on term (A1), we have:
\begin{align}
    \text{(A1)} \leq \mathbb{E}\left[\sum_{t=K\Gamma+1}^{T} \hspace{-0.3 cm}\mathds{1}\left \{\underbrace{ p_{i,t,T}^i>\frac{1}{T\epsilon_i}, N_{i,t}\ge \omega, \overline{\hat{\mu}}_{i,t,T}\leq x_i}_{\text{(C)}}\right\}\right].
\end{align}
We now wish to evaluate if condition \text{(C)} ever occurs.
In order to do so, notice that $\theta_{i,t,T}$ is a normal distributed random variable, in particular distributed as $\mathcal{N}\left(\overline{\hat{\mu}}_{i,t,T}, \frac{1}{\gamma N_{i,t}}\right)$. An $\mathcal{N}\left(m, \sigma^2\right)$ distributed random variable is stochastically dominated by $\mathcal{N}\left(m^{\prime}, \sigma^2\right)$ distributed r.v. if $m^{\prime} \geq m$. Therefore, given $\overline{\hat{\mu}}_{i,t,T} \leq x_i$, the distribution of $\theta_{i,t,T}$ is stochastically dominated by $\mathcal{N}\left(x_i, \frac{1}{\gamma N_{i,t}}\right)$ (see Lemma \ref{lem:nbord}). This implies:
$$
\Prob\left(\theta_{i,t,T}>y_i \mid N_{i,t}\ge \omega, \overline{\hat{\mu}}_{i,t,T} \leq x_i, \mathbb{F}_{t-1}\right) \leq \Prob\left(\left.\mathcal{N}\left(x_i, \frac{1}{\gamma N_{i,t}}\right)>y_i \right\rvert\, \mathcal{F}_{t-1}, N_{i,t}>\omega\right).
$$
Using Lemma~\ref{lemma:Abramowitz2}, we have:
\begin{align}
    \Prob\left(\mathcal{N}\left(x_i, \frac{1}{\gamma N_{i,t}}\right)>y_i\right) & \leq \frac{1}{2} e^{-\frac{\left(\gamma N_{i,t}\right)\left(y_i-x_i\right)^2}{2}} \\ & \leq \frac{1}{2} e^{-\frac{\left(\gamma \omega\right)\left(y_i-x_i\right)^2}{2}},
\end{align}
which is smaller than $\frac{1}{T \overline{\Delta}_i(\sigma,T)^2}$ because $\omega \ge \frac{2 \log \left(T \overline{\Delta}_i(\sigma,T)^2\right)}{\gamma \left(y_i-x_i\right)^2}$. Substituting, we get,
\begin{equation}
   \Prob\left(\theta_{i,t,T}>y_i \mid N_{i,t}>\omega, \hat{\mu}_i(t) \leq x_i, \mathcal{F}_{t-1}\right) \leq \frac{1}{T\overline{\Delta}_i(\sigma,T)^2}. 
\end{equation}
In contradiction with the condition (C), then (A1) $=0$.
\paragraph{Term (B)}
Focusing on term (B):
\begin{align}    \text{(B)}=\mathbb{E}\left[\sum_{t=K\Gamma+1}^{T}\left(\frac{1}{p_{1,t,T}^i}-1\right)\mathds{1}\left\{I_t=1 \right\}\right]
\end{align}
Let $\tau_j$ denote the time step at which arm $1$ is played for the $j$-th time:
\begin{align}
    \text{(B)} &\leq \sum_{j=\Gamma}^{T-1}\mathbb{E}\bigg[\frac{1-p_{1,\tau_{j}+1,T}^i}{p_{1,\tau_{j}+1,T}^i}\sum_{t=\tau_j+1}^{\tau_{j+1}}\mathds{1}\{I_t=1\}\bigg]\\
    &\leq\sum_{j=\Gamma}^{T-1}\mathbb{E}\bigg[\frac{1-p_{1,\tau_{j}+1,T}^i}{p_{1,\tau_{j}+1,T}^i}\bigg], \label{eq:fixed2}
\end{align}
where the inequality in Equation~\eqref{eq:fixed2} uses the fact that $p_{1,t,T}^i$ is fixed, given $\mathcal{F}_{t-1}$. Then, we observe that $p_{1,t,T}^i = \Prob(\theta_{1,t,T} > y_i |\mathcal{F}_{t-1})$
changes only when the distribution of $\theta_{1,t,T}$ changes, that is, only on the time step after each play of the first arm. Thus, $p_{1,t,T}^i$ is the same at all time steps $t \in \{\tau_j+1, \dots , \tau_{j+1}\}$, for every $j$. Finally, bounding the probability of selecting the optimal arm by $1$ we have:
\begin{equation}
    \text{(B)} \leq \sum_{j=\Gamma}^{T-1}\mathbb{E} \left[ \frac{1}{p_{1,\tau_j+1,T}^i} - 1 \right].
\end{equation}
Now in order to face this term, let us consider the arbitrary trial $j$. Thanks to Lemma \ref{lemma:changeMeasure}, we can bound the difference between the real process and an analogous (same number of trials) virtual process with mean $\overline{\mu}_1(\sigma)$ (where by stationary we mean that all the trials of the virtual process will have a fixed mean):
\begin{align}
    \underbrace{\mathbb{E} \left[ \frac{1}{p_{1,\tau_j+1,T}^i}  \right]}_{\text{(D)}} \leq \frac{2\delta_{TV}(\mathbb{P}_j,\mathbb{Q}_j(\overline{\mu}_{1}(\sigma)))}{\mathrm{erfc}(\sqrt{\frac{\gamma j}{2}}\overline{\mu}_1(\sigma))}+\underbrace{\mathbb{E}_{\overline{\mu}_1(\sigma)} \left[ \frac{1}{p_{1,\tau_j+1,T}^i} \right]}_{\text{(E)}}, 
\end{align}
where $\mathbb{P}_j$ is the distribution of the sample mean of the first $j$ samples collected from arm $1$, namely $\overline{\hat{{\mu}}}_{1,\tau_j+1,T}$, while $\mathbb{Q}_j(y)$ is the distribution of the sample mean of $j$ samples collected from \emph{any} $\lambda^2$-subgaussian distribution with fixed mean $\overline{\mu}_1(\sigma)$.
By definition, $\delta_{TV}(P,Q) \coloneqq \sup_{A\in\mathcal{F}}{\big|P(A)-Q(A)\big|}$ is the total variation between the probability measures $P$ and $Q$ (assuming they are defined over a measurable space $(\Omega,\mathcal{F})$), having observed that, using the notation of Lemma~\ref{lemma:changeMeasure} (as the environment cannot generate rewards smaller than zero):
\begin{align}
    & b= \max_{s \ge 0} \frac{1}{\Prob\left(\mathcal{N}\left(s, \frac{1}{\gamma N_{i,t}}\right)>y_i\right)} \leq\frac{1}{\Prob\left(\mathcal{N}\left(0, \frac{1}{\gamma N_{i,t}}\right)\ge \overline{\mu}_1(\sigma)\right)} =  \frac{2}{\mathrm{erfc}(\sqrt{\frac{\gamma j}{2}}\overline{\mu}_1(\sigma))}, \label{eq:nonNegChange} \\
    & a= 1.  
\end{align}
Our interest is to find if there is a minimum number of trials $j$ such that we will have $\text{(D)}\ge\text{(E)}$ without the need to add any term.
Given $\mathcal{F}_{\tau_j}$, let $\Theta_j$ denote a $\mathcal{N}\left(\overline{\hat{\mu}}_1\left(\tau_j+1\right), \frac{1}{\gamma j}\right)$ distributed Gaussian random variable. Let $G_j$ be the geometric random variable denoting the number of consecutive independent trials until and including the trial where a sample of $\Theta_j$ becomes greater than $y_i$. Then observe that $p_{1, \tau_j+1,T}^i=\operatorname{Pr}\left(\Theta_j>y_i \mid \mathcal{F}_{\tau_j}\right)$ and
\begin{align}\label{eq:rif}
    \mathbb{E}\left[\frac{1}{p_{1, \tau_j+1,T}^i}\right]=\mathbb{E}\left[\mathbb{E}\left[G_j \mid \mathcal{F}_{\tau_j}\right]\right]=\mathbb{E}\left[G_j\right].
\end{align}

We compute first the expected value for the real process. We will consider first $j$ such that $\overline{\mu}_1(j)\ge \overline{\mu}_1(\sigma)$ , we will bound the expected value of $G_j$ by a constant for all $j$ defined as earlier.
Consider any integer $r \geq 1$. Let $z=\sqrt{\log r}$ and let random variable MAX $_r$ denote the maximum of $r$ independent samples of $\Theta_j$. We abbreviate $\overline{\hat{\mu}}_1\left(\tau_j+1\right)$ to $\overline{\hat{\mu}}_1$ and we will abbreviate $\overline{\mu}_1(\sigma)$ as $\mu_1$ and $\overline{\Delta}_i(\sigma,T)$ as $\Delta_i$ in the following. Then for any integer $r\ge 1$:
\begin{align} \Prob\left(G_j \leq r\right) & \geq \Prob\left(\operatorname{MAX}_r>y_i\right) \\ & \geq \Prob\left(\operatorname{MAX}_r>\overline{\hat{\mu}}_1+\frac{z}{\sqrt{\gamma j}} \geq y_i\right) \\ & =\mathbb{E}\left[\mathbb{E}\left[\left.\mathds{1}\left(\operatorname{MAX}_r>\overline{\hat{\mu}}_1+\frac{z}{\sqrt{\gamma j}} \geq y_i\right) \right\rvert\, \mathcal{F}_{\tau_j}\right]\right] \\ & =\mathbb{E}\left[\mathds{1}\left(\overline{\hat{\mu}}_1+\frac{z}{\sqrt{\gamma j}} \geq y_i\right) \Prob\left(\left.\operatorname{MAX}_r>\overline{\hat{\mu}}_1+\frac{z}{\sqrt{\gamma j}} \right\rvert\, \mathcal{F}_{\tau_j}\right)\right].
\end{align}
For any instantiation $F_{\tau_j}$ of $\mathcal{F}_{\tau_j}$, since $\Theta_j$ is Gaussian $\mathcal{N}\left(\hat{\mu}_1, \frac{1}{\gamma j}\right)$ distributed r.v., this gives using \ref{lemma:Abramowitz}:
\begin{align}
    \Prob\left(\left.\operatorname{MAX}_r>\overline{\hat{\mu}}_1+\frac{z}{\sqrt{\gamma j}} \right\rvert\, \mathcal{F}_{\tau_j}=F_{\tau_j}\right) & \geq 1-\left(1-\frac{1}{\sqrt{2 \pi}} \frac{z}{\left(z^2+1\right)} e^{-z^2 / 2}\right)^r \\ & =1-\left(1-\frac{1}{\sqrt{2 \pi}} \frac{\sqrt{\log r}}{(\log r+1)} \frac{1}{\sqrt{r}}\right)^r \\ & \geq 1-e^{-\frac{r}{\sqrt{4 \pi r \log r}}}.
\end{align}
For $r \ge e^{12}$:
\begin{align}
   \Prob\left(\left.\operatorname{MAX}_r>\overline{\hat{\mu}}_1+\frac{z}{\sqrt{\gamma j}} \right\rvert\, \mathcal{F}_{\tau_j}=F_{\tau_j}\right) \geq 1-\frac{1}{r^2}. 
\end{align}
Substituting we obtain:
\begin{align}
    \Prob\left(G_j \leq r\right) & \geq \mathbb{E}\left[\mathds{1}\left(\overline{\hat{\mu}}_1+\frac{z}{\sqrt{\gamma j}} \geq y_i\right)\left(1-\frac{1}{r^2}\right)\right] \\ & =\left(1-\frac{1}{r^2}\right) \Prob\left(\overline{\hat{\mu}}_1+\frac{z}{\sqrt{\gamma j}} \geq y_i\right).
\end{align}
Applying Lemma~\ref{lemma:Subg} to the second term, we can write, since $\mathbb{E}[\overline{\hat{\mu}}_1]\ge\mu_1$:
\begin{align}
  \Prob\left(\overline{\hat{\mu}}_1+\frac{z}{\sqrt{\gamma j}} \geq \mu_1\right) \geq 1-e^{-\frac{z^2}{2\gamma \lambda^2}} \geq 1-\frac{1}{r^2},  
\end{align}
being $\gamma\leq \frac{1}{4\lambda^2}$. Using, $y_i \leq \mu_1$, this gives
\begin{equation}
   \Prob\left(\overline{\hat{\mu}}_1+\frac{z}{\sqrt{\gamma j}} \geq y_i\right) \geq 1-\frac{1}{r^2} . 
\end{equation}
Substituting all back we obtain:
\begin{align}
    \mathbb{E}\left[G_j\right] & =\sum_{r=0}^{+\infty} \Prob\left(G_j \geq r\right) \\ & =1+\sum_{r=1}^{\infty} \Prob\left(G_j \geq r\right) \\ & \leq 1+e^{12}+\sum_{r \geq 1}\left(\frac{1}{r^2}+\frac{1}{r^{2}}\right) \\ & \leq 1+e^{12}+2+2.
\end{align}
This shows a constant bound of $\mathbb{E}\left[\frac{1}{p_{1, \tau_j+1,T}^i}-1\right]=\mathbb{E}\left[G_j\right]-1 \leq e^{12}+5$ for all $j\ge \sigma$.
We derive a  bound for large $j$. Consider $j\ge\omega$ (and still $j\ge\sigma$). Given any $r \geq 1$, define $G_j, \operatorname{MAX}_r$, and $z=\sqrt{\log r}$ as defined earlier. Then,
\begin{align} \Prob\left(G_j \leq r\right) & \geq \Prob\left(\operatorname{MAX}_r>y_i\right) \\ & \geq \Prob\left(\operatorname{MAX}_r>\overline{\hat{\mu}}_1+\frac{z}{\sqrt{\gamma j}}-\frac{\Delta_i}{6} \geq y_i\right) \\ & =\mathbb{E}\left[\mathbb{E}\left[\left.\mathds{1}\left(\operatorname{MAX}_r>\overline{\hat{\mu}}_1+\frac{z}{\sqrt{\gamma j}}-\frac{\Delta_i}{6} \geq y_i\right) \right\rvert\, \mathcal{F}_{\tau_j}\right]\right] \\ & =\mathbb{E}\left[\mathds{1}\left(\overline{\hat{\mu}}_1+\frac{z}{\sqrt{\gamma j}}+\frac{\Delta_i}{6} \geq \mu_1\right) \Prob\left(\left.\operatorname{MAX}_r>\overline{\hat{\mu}}_1+\frac{z}{\sqrt{\gamma j}}-\frac{\Delta_i}{6} \right\rvert\, \mathcal{F}_{\tau_j}\right)\right] .\end{align}
where we used that $y_i=\mu_1-\frac{\Delta_i}{3}$. Now, since $j \geq \omega=\frac{288 \log \left(T \Delta_i^2+e^{6}\right)}{\gamma\Delta_i^2}$,
\begin{align}
    2 \frac{\sqrt{2 \log \left(T \Delta_i^2+e^{6}\right)}}{\sqrt{\gamma j}} \leq \frac{\Delta_i}{6}.
\end{align}

Therefore, for $r \leq\left(T \Delta_i^2+e^{6}\right)^2$,
\begin{align}
 \frac{z}{\sqrt{\gamma j}}-\frac{\Delta_i}{6}=\frac{\sqrt{\log (r)}}{\sqrt{\gamma j}}-\frac{\Delta_i}{6} \leq-\frac{\Delta_i}{12}.
\end{align}

Then, since $\Theta_j$ is $\mathcal{N}\left(\overline{\hat{\mu}}_1\left(\tau_j+1\right), \frac{1}{\gamma j}\right)$ distributed random variable, using the upper bound in Lemma \ref{lemma:Abramowitz2}, we obtain for any instantiation $F_{\tau_j}$ of history $\mathbb{F}_{\tau_j}$,
\begin{align}
    \Prob\left(\left.\Theta_j>\overline{\hat{\mu}}_1\left(\tau_j+1\right)-\frac{\Delta_i}{12} \right\rvert\, \mathcal{F}_{\tau_j}=F_{\tau_j}\right) \geq 1-\frac{1}{2} e^{-\gamma j \frac{\Delta_i^2}{288}} \geq 1-\frac{1}{2\left(T \Delta_i^2+e^{6}\right)}.
\end{align}
being $j\geq \omega$. This implies:
\begin{align}
    \Prob\left(\left.\operatorname{MAX}_r>\hat{\mu}_1\left(\tau_j+1\right)+\frac{z}{\sqrt{\gamma j}}-\frac{\Delta_i}{6} \right\rvert\, \mathcal{F}_{\tau_j}=F_{\tau_j}\right) \geq 1-\frac{1}{2^r\left(T \Delta_i^2+e^{6}\right)^r}.
\end{align}

Also, for any $t \geq \tau_j+1$, we have $k_1(t) \geq j$, and using Lemma~\ref{lemma:Subg}, as $\mathbb{E}[\overline{\hat{\mu}}_1]\ge\mu_1$ we get:
\begin{align}
\Prob\left(\overline{\hat{\mu}}_1+\frac{z}{\sqrt{\gamma j}}-\frac{\Delta_i}{6} \geq y_i\right) \geq \Prob\left(\overline{\hat{\mu}}_1 \geq \mu_1-\frac{\Delta_i}{6}\right) \geq 1-e^{- k_1(t) \Delta_i^2 / 72\lambda^2} \geq 1-\frac{1}{\left(T \Delta_i^2+e^{6}\right)^{16}}.
\end{align}

Let $T^{\prime}=\left(T \Delta_i^2+e^{6}\right)^2$. Therefore, for $1 \leq r \leq T^{\prime}$, we have:
\begin{align}
    \Prob\left(G_j \leq r\right) \geq 1-\frac{1}{2^r\left(T^{\prime}\right)^{r / 2}}-\frac{1}{\left(T^{\prime}\right)^8}.
\end{align}

When $r \geq T^{\prime} \geq e^{12}$, we obtain:
\begin{align}
    \Prob\left(G_j \leq r\right) \geq 1-\frac{1}{r^2}-\frac{1}{r^{2}}.
\end{align}

Combining all the bounds we have derived:
    \begin{align} \mathbb{E}\left[G_j\right] & \leq \sum_{r=0}^{\infty} \Prob\left(G_j \geq r\right) \\ & \leq 1+\sum_{r=1}^{T^{\prime}} \Prob\left(G_j \geq r\right)+\sum_{r=T^{\prime}}^{\infty} \Prob\left(G_j \geq r\right) \\ & \leq 1+\sum_{r=1}^{T^{\prime}} \frac{1}{\left(2 \sqrt{T^{\prime}}\right)^r}+\frac{1}{\left(T^{\prime}\right)^7}+\sum_{r=T^{\prime}}^{\infty} \frac{1}{r^2}+\frac{1}{r^{1.5}} \\ & \leq 1+\frac{1}{\sqrt{T^{\prime}}}+\frac{1}{\left(T^{\prime}\right)^7}+\frac{2}{T^{\prime}}+\frac{3}{\sqrt{T^{\prime}}} \\ & \leq 1+\frac{5}{T \Delta_i^2+e^{6}} .\end{align}

So we have proved that:
\begin{align}
    \mathbb{E}\left[\frac{1}{p_{1, \tau_j+1,T}^i}-1\right]\le\begin{cases}
                 \frac{2\delta_{TV}(\mathbb{P}_j(\overline{\mu}_{1}(j)),\mathbb{Q}_j(\overline{\mu}_{1}(\sigma)))}{\mathrm{erfc}(\sqrt{\frac{\gamma j}{2}}\overline{\mu}_1(\sigma))}+\mathbb{E}_{\overline{\mu}_1(\sigma)} \left[ \frac{1}{p_{1,\tau_j+1,T}^i} -1\right].  & \text{if } 0 \le j<\sigma \\
                (e^{12}+5) & \text{if } j\ge\sigma \\
                \frac{5}{T\overline{\Delta}_i(\sigma,T)^2} &\textit{if } j\ge \omega \textit{ and } j\ge\sigma 
        \end{cases}
\end{align}
Furthermore, it also holds:
\begin{align}
    \mathbb{E}_{\overline{\mu}_1(\sigma)}\left[\frac{1}{p_{1, \tau_j+1,T}^i}-1\right]\leq\begin{cases}
                (e^{12}+5) & \text{if } j\leq\omega \\
                \frac{5}{T\overline{\Delta}_i(\sigma,T)^2} &\text{if } j\ge \omega  
        \end{cases},
\end{align}
as it is the expected value for the term in a stationary process with fixed mean for the reward equal to $\overline{\mu}_1(\sigma)$ (Lemma 6 \citet{pmlr-v31-agrawal13a}, but can be also retrieved by making the same calculation that we have just performed).
So, we can write:
\begin{align}
    &\text{(B)}\leq\sum_{j=\Gamma}^{T-1}\mathbb{E}\left[\frac{1}{p_{1, \tau_j+1,T}^i}-1\right]\nonumber\\&\leq  \begin{cases}
        \sum_{j=\Gamma}^{\omega}(e^{12}+5)+\sum_{j=\Gamma}^{T-1}\frac{1}{T\overline{\Delta}_i(\sigma,T)^2} &\text{if } \Gamma\ge\sigma\\
        \\
        \sum_{j=\Gamma}^{\sigma-1}\frac{2\delta_{TV}(\mathbb{P}_j(\overline{\mu}_{1}(j)),\mathbb{Q}_j(\overline{\mu}_{1}(\sigma)))}{\mathrm{erfc}(\sqrt{\frac{\gamma j}{2}}\overline{\mu}_1(\sigma))}+\sum_{j=\Gamma}^{\omega}(e^{12}+5)+\sum_{j=\Gamma}^{T-1}\frac{1}{T\overline{\Delta}_i(\sigma,T)^2} &\text{if } \Gamma<\sigma\\
    \end{cases}\\
    &\leq \begin{cases}
        \omega(e^{12}+5)+\frac{1}{\overline{\Delta}_i(\sigma,T)^2} &\text{if } \Gamma\ge\sigma\\
        \\
        \sum_{j=\Gamma}^{\sigma-1}\frac{2\delta_{TV}(\mathbb{P}_j(\overline{\mu}_{1}(j)),\mathbb{Q}_j(\overline{\mu}_{1}(\sigma)))}{\mathrm{erfc}(\sqrt{\frac{\gamma j}{2}}\overline{\mu}_1(\sigma))}+{\omega}(e^{12}+5)+\frac{1}{\overline{\Delta}_i(\sigma,T)^2} &\text{if } \Gamma<\sigma\\
    \end{cases},
\end{align}
summing all the term follows the statement.
\end{proof}}

\cor*
\begin{proof}
    The proof for \texttt{Beta-TS} follows from the proof of Theorem \ref{thm:ts}, setting $\omega=\frac{\log(T)}{2(x_i-y_i)^2}$ $x_i=\overline{\mu}_i(T)+\frac{\overline{\Delta}_i(\sigma,T)}{3}$ and $y_i=\overline{\mu}_{1}(\sigma)-\frac{\overline{\Delta}_i(\sigma,T)}{3}$. Rewriting Equation \eqref{eq:tsexp}:
    \begin{align}
        \sum_{j=1}^{T}\exp(-jd(x_i,y_i))\le \sum_{j=1}^T \frac{1}{j(x_i-y_i)^2}\leq\frac{9\ln(T)}{\overline{\Delta}_i(\sigma,T)^2}.
    \end{align}
Finally we can rewrite Equation \eqref{eq:lemma4}:
\begin{align}
	\sum_{s=0}^{j} \frac{f_{j,\overline{\mu}_{1}(\sigma)}(s)}{F_{j+1, y_i}(s)}-1 &\leq
	\begin{cases}
		\frac{3}{\Delta_i'} & \text{if } j<\frac{8}{\Delta_i'}\\
		\Theta\left(e^{-\frac{\Delta_i^{'2}j}{2}}+\frac{e^{-D_i j}}{(j+1)\Delta_i'^{2}}+\frac{1}{e^{\Delta_i'^{2}\frac{j}{4}}-1}\right) & \text{if } j\ge\frac{8}{\Delta_i'}
	\end{cases},\\
    &\leq \begin{cases}
		\frac{3}{\Delta_i'} & \text{if } j<\frac{8}{\Delta_i'}\\
		\Theta\left(\frac{2}{{\Delta_{i}'}^2j}+\frac{1}{(j+1)\Delta_i'^{2}}+\frac{1}{{\Delta_i'^{2}\frac{j}{4}}}\right) & \text{if } j\ge\frac{8}{\Delta_i'}
	\end{cases}
\end{align}
where $\Delta_i'= \overline{\mu}_{1}(\sigma)-y_i=\frac{\overline{\Delta}_i(\sigma,T)}{3}$ and $D_i=d(y_i,\overline{\mu}_{1}(\sigma)) =y_i\log{\frac{y_i}{\overline{\mu}_{1}(\sigma)}}+(1-y_i)\log{\frac{1-y_i}{1-\overline{\mu}_{1}(\sigma)}}$. Summing over all the terms we obtain:
\begin{align}
   \sum_{j=1}^T \left(\sum_{s=0}^{j} \frac{f_{j,\overline{\mu}_{1}(\sigma)}(s)}{F_{j+1, y_i}(s)}-1\right) &\leq O\left(\frac{1}{\overline{\Delta}_i(\sigma,T)^2} +\frac{\log(T)}{\overline{\Delta}_i(\sigma,T)^2}\right).
\end{align}
By doing all the steps from Equation \eqref{line:line} to Equation \eqref{eq:tscondfin} it is easy to see that the condition from Equation \eqref{eq:tscond} still holds. The result follows by summing all the terms and noticing that for all the instances in $\mathcal{M}_{\sigma}^{\mathcal{S}}$, whenever each arm is pulled at least $\sigma$ times the sum of the dissimilarity terms vanishes. The result for  \texttt{$\gamma$-GTS} follows trivially from the statement of Theorem \ref{thm:gts}, noting again that by definition of the class of instances all the dissimilarity terms vanish for $\Gamma\ge\sigma$.
\end{proof}

\clearpage
\lb*
\begin{proof}
    First of all, we build two instances, defined for $n\in\dsb{T}$ and $2\sigma +1 \le T \implies \sigma \le (T-1)/2$:
    \begin{equation}
       \bm{\mu} = \begin{cases}
           \mu_1(n)=\min\left\{\frac{n-1}{2\sigma}, \frac{1}{2}\right\}\\
           \\
           \mu_i(n)=\min\left\{\frac{n-1}{2\sigma}, \frac{1}{4}\right\} & \text{ if } i\in \dsb{K}\setminus\{1\}
       \end{cases},
    \end{equation}
    We define $\tau_{i,\sigma}$ as the round in which the $i$-th arm gets to $\frac{\sigma}{2}$ pulls and $i^*\in \arg\max_{i\in \dsb{K}\setminus\{1\}} \left\{\mathbb{E}_{\bm{\mu}}[\tau_{i,\sigma}]\right\}$, i.e., the arm whose expected value for the round in which it is played for the $\frac{\sigma}{2}$-th time under policy $\mathfrak{A}$ in environment $\bm{\mu}$ is the largest in the set of the suboptimal arms. We then introduce the second environment as:
   \begin{equation}
       \bm{\mu'} = \begin{cases}
           \mu_1'(n)=\min\left\{\frac{n-1}{2\sigma}, \frac{1}{2}\right\}\\
           \\
           \mu_i'(n)=\min\left\{\frac{n-1}{2\sigma}, \frac{1}{4}\right\}  & \text{ if } i\in \dsb{K}\setminus\{1,i^*\}\\
           \\
           \mu_{i^*}'(n)=\min\left\{\frac{n-1}{2\sigma}, 1\right\}
       \end{cases}.
    \end{equation} 
     Intuitively, the two environments are indistinguishable as long as an algorithm does not pull arm $i^*$ at least $\frac{\sigma}{2}$ times.
    The instances are depicted in Figure \ref{fig:instance2}.

\begin{figure}[t]
\begin{center}
\resizebox{\textwidth}{!}{
\includegraphics{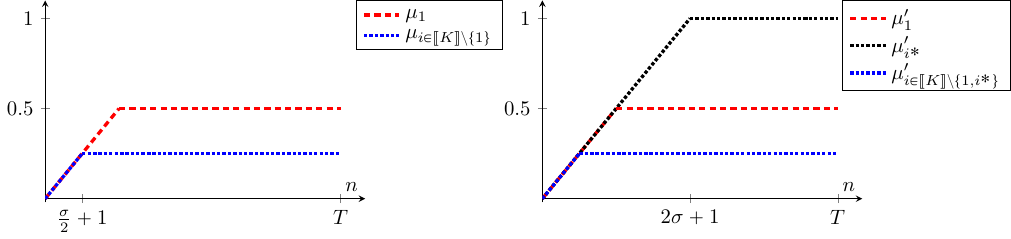}
}
\end{center}
\caption{The two instances $\bm{\mu}$ and $\bm{\mu'}$.}\label{fig:instance2}
\end{figure}
By performing some simple calculations, the following equalities also hold:
\begin{align*}
& \overline{\mu}_i(T)= \overline{\mu}_{i^*}(T) = \overline{\mu}'_i(T)= \frac{1}{4}-\frac{\frac{\sigma}{2}+1}{8T}, \\
    & \overline{\mu}_{i^*}'(T)=  1 - \frac{2\sigma+1}{2T}, & \text{if } i \not\in \dsb{K} \setminus \{1,i^*\}\\
    &\overline{\mu}_1(T)=\overline{\mu}'_1(T)=\frac{1}{2}-\frac{\sigma+1}{4T}
\end{align*}
Consequently, we have for all $i \in \dsb{K}\setminus\{1\}$:
\begin{align}
    \overline{\Delta}_{i,\bm{\mu}}(T,T) = \overline{\mu}_1(T)- \overline{\mu}_i(T) =   \frac{1}{4} -  \frac{\frac{3\sigma}{2}+1}{8T} \ge \frac{\frac{5}{4}T+1}{8T} \ge \frac{5}{32},
\end{align}
and again for all $i \in \dsb{K}\setminus\{i^*\}$:
\begin{align}
    \overline{\Delta}_{i,\bm{\mu'}}(T,T) \ge \overline{\Delta}_{1,\bm{\mu'}}(T,T)\ge\overline{\mu}'_{i^*}(T)- \overline{\mu}'_1(T) =   \frac{1}{2} -  \frac{3\sigma+1}{4T} \ge \frac{T+1}{8T} \ge \frac{1}{8},
\end{align}
having used $\sigma \le (T-1)/2$. Furthermore, notice that by definitions of the instances both $\bm{\mu}$ and $\bm{\mu'}$ will belong to $\mathcal{M}_{2+2\sigma}$.
Let $\mathfrak{A}$ be an algorithm, using Lemma \ref{lem:lbmf}, we can lower bound the regret in both the environments $\bm{\mu}$ and $\bm{\mu'}$:
\begin{align*}
    R_{\bm{\mu}}(\mathfrak{A},T)\ge \sum_{i\neq 1}\overline{\Delta}_{i,\bm{\mu}}(T,T)\mathbb{E}_{\bm{\mu}}[N_{i,T}]\ge\frac{5}{32} \sum_{i\neq 1}\mathbb{E}_{\bm{\mu}}[N_{i,T}],
\end{align*}
where $R_{\bm{\mu}}(\mathfrak{A},T)$ is the expected cumulative regret at the time horizon $T$ in the environment $\bm{\mu}$ and $\mathbb{E}_{\bm{\mu}}[N_{i,T}]$ is the expected value of the total number of pulls of the suboptimal arm $i$ at the time horizon $T$, for a fixed policy $\mathfrak{A}$. Similarly for the environment $\bm{\mu}'$:
\begin{equation*}
  R_{\bm{\mu'}}(\mathfrak{A},T)\ge \sum_{i\neq i^*}\overline{\Delta}_{i,\bm{\mu'}}(T,T)\mathbb{E}_{\bm{\mu'}}[N_{i,T}]\ge\frac{1}{8} \sum_{i\neq i^*}\mathbb{E}_{\bm{\mu'}}[N_{i,T}],
\end{equation*}
where $R_{\bm{\mu'}}(\mathfrak{A},T)$ is the expected cumulative regret at the time horizon $T$ in the environment $\bm{\mu'}$ and $\mathbb{E}_{\bm{\mu'}}[N_{i,T}]$ is the expected value of the total number of pulls of the suboptimal arm $i$ at the time horizon $T$, for a fixed policy $\mathfrak{A}$. For ease of notation we define $N_{i\neq 1,\tau}=\sum_{i\neq 1} N_{i,\tau}$ and $N_{i\neq i^*,\tau}=\sum_{i\neq i^*} N_{i,\tau}$, for all possible stopping times $\tau \in \dsb{T}$. Notice that the following holds:
\begin{align}
    \sup_{\bm{\mu''}\in \mathcal{M}_{2\sigma+2}} R_{\bm{\mu''}}(\mathfrak{A},T)& \ge \max\{R_{\bm{\mu}}(\mathfrak{A},T),R_{\bm{\mu'}}(\mathfrak{A},T)\}\\
    &\ge \frac{R_{\bm{\mu}}(\mathfrak{A},T)+R_{\bm{\mu'}}(\mathfrak{A},T)}{2}\\
    &\ge\frac{\min\{\overline{\Delta}_{i,\bm{\mu}}(T,T),\overline{\Delta}_{1,\bm{\mu'}}(T,T)\}}{2}(\mathbb{E}_{\bm{\mu}}[N_{i\neq 1,T}]+\mathbb{E}_{\bm{\mu'}}[N_{i\neq i^*,T}])\label{eq:riff}
\end{align}
Now we select as a stopping time for the adapted filtration $\tau_{i^*,\sigma}$, where $\tau_{i^*,\sigma}$ is defined as above, so we can write $\mathbb{E}_{\bm{\mu}}[N_{i\neq i^*,\tau_{i^*,\sigma}}]= \mathbb{E}_{\bm{\mu'}}[N_{i\neq i^*,\tau_{i^*,\sigma}}]$, since by definition $\tau_{i^*,\sigma}$ is the r.v. denoting the time at which the $i^*$-th arm gets played for the $\frac{\sigma}{2}$-th, i.e., up to which the two instances are indistinguishable. This yield to: 
\begin{align}
\mathbb{E}_{\bm{\mu}}[N_{i\neq1,\tau_{i^*,\sigma}}] + \mathbb{E}_{\bm{\mu'}}[N_{i\neq i^*,\tau_{i^*,\sigma}}]&=\mathbb{E}_{\bm{\mu}}[N_{i\neq1,\tau_{i^*,\sigma}}]+ \mathbb{E}_{\bm{\mu}}[N_{i\neq i^*,\tau_{i^*,\sigma}}]\\
&= \mathbb{E}_{\bm{\mu}}[N_{i\neq 1,\tau_{i^*,\sigma}}+N_{i\neq i^*,\tau_{i^*,\sigma}}]\\
&\ge \mathbb{E}_{\bm{\mu}}[\tau_{i^*,\sigma}],
\end{align}
where for the last inequality we used the fact that $N_{i\neq 1,\tau_{i^*,\sigma}}+N_{i\neq i^*,\tau_{i^*,\sigma}}=\sum_{t=1}^{\tau_{i^*,\sigma}}\mathds{1}\{I_t\neq 1\}+\mathds{1}\{I_t\neq i^*\}\ge \sum_{t=1}^{\tau_{i^*,\sigma}}1$. Now, in order to find the lower bound we shall find the bound for $\mathbb{E}_{\bm{\mu}}[\tau_{i^*,\sigma}]$. We recall that by definition $i^*\in \arg\max_{i\in \dsb{K}\setminus\{1\}}\mathbb{E}_{\bm{\mu}}[\tau_{i,\sigma}]$, so the following holds:
\begin{align}
    (K-1)\mathbb{E}_{\bm{\mu}}[\tau_{i^*,\sigma}]\ge\sum_{i=2}^K \mathbb{E}_{\bm{\mu}}[\tau_{i,\sigma}]=\mathbb{E}_{\bm{\mu}}\left[\sum_{i=2}^K \tau_{i,\sigma}\right]\label{eq:lbcomb}.
\end{align}
Let us sort the arms in ascending order based on the round in which they get to $\frac{\sigma}{2}$ pulls, i.e., according to the value of $\tau_{\cdot,\sigma}$. In particular, we use the notation:
\begin{align}
    \tau_{[1],\sigma} <  \tau_{[2],\sigma} < \dots < \tau_{[K],\sigma}.
\end{align}
Thus, we can rewrite:
\begin{align}
   \sum_{i=2}^K \tau_{i,\sigma}\ge \sum_{j=1}^{K-1} \tau_{[j],\sigma},
\end{align}
by a pigeonhole-like argument it is easy to infer that $t_{[j],\sigma}\ge j\frac{\sigma}{2}$, then:
\begin{align}
     \sum_{i=2}^K \tau_{i,\sigma}\ge \sum_{j=1}^{K-1} \tau_{[j],\sigma} \ge \frac{\sigma}{2}\sum_{j=1}^{K-1} j\ge\frac{\sigma(K-1)K}{4}.
\end{align}
Substituting in \eqref{eq:lbcomb} we obtain:
\begin{equation}
    \mathbb{E}_{\bm{\mu}}[\tau_{i^*,\sigma}]\ge\frac{K\sigma}{4}
\end{equation}
Putting all together, from Equation \eqref{eq:riff}:
\begin{align}
    \sup_{\bm{\mu''}\in \mathcal{M}_{2\sigma+2}} R_{\bm{\mu''}}(\mathfrak{A},T)& \ge\frac{1}{8}\mathbb{E}_{\bm{\mu}}[\tau_{i^*,\sigma}]\ge \frac{K}{32}\sigma,
\end{align}
the final result follows substituting $\sigma\leftarrow\frac{\sigma-2}{2}$.
\end{proof}
\clearpage
\subsection{Proofs of Section~\ref{sec:SWapproach}}
\swbeta*

\begin{proof}
For ease of notation, we set $\sigma'(T;\tau)=\sigma'(\tau), \overline{\mu}_{i^*(T)}(\sigma'(T;\tau);\tau)=\overline{\mu}_{i^*(T)}(\sigma'(\tau))$ and $\Delta_i'(T;\tau)=\Delta_i$. Using Lemma \ref{lem:RD}, we can bound the expected value of the number of pulls at the time horizon as:
\begin{align*}
    \mathbb{E}_{\tau}[N_{i,T}] & \leq\Gamma+1+\frac{\omega T}{\tau}+\underbrace{\mathbb{E}\left[\sum_{t=K\Gamma+1}^{T} \mathds{1}\left \{ p_{i,t,\tau}^i>\frac{1}{T}, I_t=i, N_{i,t,\tau}\ge \omega\right\}\right]}_{\text{(A)}} \\
    & +\underbrace{\mathbb{E}\left[\sum_{t=K\Gamma+1}^{T}\left(\frac{1}{p_{i^*(t),t,\tau}^i}-1\right)\mathds{1}\left\{I_t=i^*(t) \right\}\right]}_{\text{(B)}}.
\end{align*}
In what follows, we will consider, without loss of generality, the first arm to be the best, i.e., for each round $t \in \dsb{T}$ we have that $i^*(t)=i^*(T)=1$. Furthermore we will set $y_{i,t}=y_i=\overline{\mu}_1(\sigma'(\tau))-\frac{\Delta_i}{3}$, $x_i=\mu_i(T)+\frac{\Delta_i}{3}$, $\omega=\frac{\log{T}}{(y_i-x_i)^2}$.We define $\hat{\mu}_{i,t,\tau}=\frac{S_{i,t,\tau}}{N_{i,t,\tau}}$ We further decompose $\mathcal{A}$ in two contributions, formally:
\begin{align*}
    \text{(A)} & = \underbrace{\mathbb{E}\left[\sum_{t=K\Gamma+1}^{T} \mathds{1}\left \{ p_{i,t,\tau}^i>\frac{1}{T}, \hat{\mu}_{i,t,\tau}\leq x_i, I_t=i, N_{i,t,\tau}\ge \omega\right\}\right]}_{\text{(A1)}}\\ 
    & +\underbrace{\mathbb{E}\left[\sum_{t=K\Gamma+1}^{T} \mathds{1}\left \{ p_{i,t,\tau}^i>\frac{1}{T}, \hat{\mu}_{i,t,\tau}> x_i, I_t=i, N_{i,t,\tau}\ge \omega\right\}\right]}_{\text{(A2)}}.
\end{align*}

\paragraph{Term (A2)}
Focusing on (A2):
\begin{align}
   \text{(A2)} &\leq \mathbb{E}\left[\sum_{t=K\Gamma+1}^T \mathds{1}\left\{\hat{\mu}_{i,t,T}> x_i ,   N_{i,t,\tau}\ge \omega\right\}\right]\\
   &\leq \sum_{K\Gamma+1}^T\Prob(\hat{\mu}_{i,t,\tau}> x_i | N_{i,t,\tau}\ge \omega)\\
   &\leq \sum_{K\Gamma+1}^T\Prob(\hat{\mu}_{i,t,\tau}-\mathbb{E}[\hat{\mu}_{i,t,\tau}]> x_i-\mathbb{E}[\hat{\mu}_{i,t,\tau}] | N_{i,t,\tau}\ge \omega)\\
    &\leq \sum_{K\Gamma+1}^T\Prob(\hat{\mu}_{i,t,\tau}-\mathbb{E}[\hat{\mu}_{i,t,\tau}]> x_i-\mu_i(T) | N_{i,t,\tau}\ge \omega)\\
    &\leq \sum_{t=1}^T \exp{(-2\omega(x_i-\mu_i(T))^2)},
\end{align}
where the last inequality follows from the Chernoff-Hoeffding bound (Lemma \ref{lemma:chernoff}). By definition it holds that $x_i-\mu_i(T)=y_i-x_i$, substituting we obtain:
\begin{align}
     \sum_{t=1}^T \exp{(-2\omega(x_i-\mu_i(T))^2)}= \sum_{t=1}^T \exp{(-2\log(T))}<1.
\end{align}
\paragraph{Term (A1)} Evaluating the term (A1) we can rewrite:
\begin{equation*}
    \text{(A1)}\leq \mathbb{E}\left[\sum_{t=K\Gamma+1}^{T} \mathds{1}\left \{\underbrace{ p_{i,t,\tau}^i>\frac{1}{T}, N_{i,t,\tau}\ge \omega, \hat{\mu}_{i,t,\tau}\leq x_i}_{\text{(C)}}\right\}\right].
\end{equation*}
We wish to evaluate if condition (C) ever occurs.
To this end, let $(\mathcal{F}_{t-1})_{t \in \dsb{T}}$ be the canonical filtration. We have:
\begin{align}
&\Prob(\theta_{i,t,\tau}\hspace{-0.05 cm}>\hspace{-0.05 cm}y_{i}|N_{i,t,\tau}\ge \omega,\hat{\mu}_{i,t,\tau} \leq x_{i}, \mathcal{F}_{t-1})\hspace{-0.05cm}= \nonumber\\ &=\hspace{-0.05cm} \Prob \left( \text{Beta} \left( \hat{\mu}_{i,t,\tau} N_{i,t,\tau}\hspace{-0.1 cm} +\hspace{-0.1 cm} 1, (1 - \hat{\mu}_{i,t,\tau}) N_{i,t,\tau}\hspace{-0.1 cm} +\hspace{-0.1 cm} 1 \right) > y_{i} | \hat{\mu}_{i,t,\tau} \leq x_{i}, N_{i,t,\tau}\ge \omega \right) \\
    & \leq \Prob \left( \text{Beta} \left( x_{i} N_{i,t,\tau} + 1, (1 - x_{i}) N_{i,t,\tau} + 1 \right) > y_{i}| N_{i,t,\tau}\ge \omega\right) \label{lin:linx}\\
    & \leq F^{B}_{N_{i,t,\tau}+1,y_{i}}\big(x_{i}N_{i,t,\tau}| N_{i,t,\tau}\ge \omega\big)\label{line:line2a}\\ &\leq F^{B}_{N_{i,t,\tau},y_{i}}\big(x_{i}N_{i,t,\tau}| N_{i,t,\tau}\ge \omega\big)\label{lin:roch}\\
    &\leq\exp \left( - N_{i,t,\tau} d(x_{i}, y_{i})| N_{i,t,\tau}\ge \omega \right)\\
    &\leq\exp{(-2\omega(y_i-x_i)^2)},
\end{align}
where the last-but-one inequality follows from the generalized Chernoff-Hoeffding bounds (Lemma~\ref{lemma:chernoff}). Equation~\eqref{line:line2a} was derived and the Beta-Binomial identity (Fact \ref{lem:betabin}), \eqref{lin:linx} by exploiting that a sample from $\text{Beta} \left( x_{i} T_{i,t,\tau} + 1, (1 - x_{i}) T_{i,t,\tau} + 1 \right) $ is likely to be as large as a sample from $\text{Beta}( \hat{\mu}_{i,t} T_{i,t,\tau}(t) + 1, (1 - \hat{\mu}_{i,t,\tau})T_{i,t,\tau} + 1 )$ (reported formally in Lemma \ref{lem:nbord}) and finally \eqref{lin:roch} follows from Lemma \ref{lem:bborder}.
Therefore, for $\omega =\frac{\log(T)}{(y_i-x_i)^2}$, we have:
\begin{equation}
    \Prob(\theta_{i,t,\tau}\hspace{-0.05 cm}>\hspace{-0.05 cm}y_{i},N_{i,t,\tau}\ge \omega,\hat{\mu}_{i,t,\tau} \leq x_{i}, \mathcal{F}_{t-1})\leq\frac{1}{T},
\end{equation}
in contradiction with condition (C), so (A1)$=0$.
\paragraph{Term (B)}
We now tackle term (B):
\begin{align}
    \text{(B)}=\mathbb{E}\left[\sum_{t=K\Gamma+1}^{T}\left(\frac{1}{p_{1,t,\tau}^i}-1\right)\mathds{1}\left\{I_t=1 \right\}\right]
\end{align}
We rewrite the last equality as the sum of two contributions. We introduce the number of times we pull the best arm after the exploration phase ad $N_{1,t}^\Gamma=\sum_{t=K\Gamma+1}^T\mathds{1}\{I_t=1\}$. We distinguish between $N_{1,t}^\Gamma\ge\sigma'(\tau)-\Gamma$ and $N_{1,t}^\Gamma\leq\sigma'(\tau)-\Gamma$. This way, we obtain the following:
\begin{align}
    \text{(B)} \leq &\underbrace{\sum_{t=K\Gamma+1}^{T}\mathbb{E}\bigg[\frac{1-p_{1,t,\tau}^i}{p_{1,t,\tau}^i}\mathds{1}\{I_t=1,N_{1,t}^\Gamma\leq \sigma'(\tau)-\Gamma\}\bigg]}_{\text{(B1)}}+ \nonumber\\
&+\underbrace{\sum_{t=K\Gamma+1}^{T}\mathbb{E}\bigg[\frac{1-p_{1,t,\tau}^i}{p_{1,t,\tau}^i}\mathds{1}\{I_t=1, N_{1,t}^\Gamma\ge \sigma'(\tau)-\Gamma\}\bigg]}_{\text{(B2)}}.
\end{align}
We further decompose the term (B1) in other terms:
\begin{align}
    \text{(B1)} &\leq\underbrace{\sum_{t=K\Gamma+1}^{T}\mathbb{E}\left[\frac{1-p_{1,t,\tau}^i}{p_{1,t,\tau}^i}\mathds{1}\left\{\overbrace{I_t=1,N_{1,t}^\Gamma\leq \sigma'(\tau)-\Gamma,N_{1,t,\tau}\leq \frac{8\log(T)}{(\overline{\mu}_1(\sigma'(\tau))-y_i)^2}}^{\mathcal{C}_1}\right\}\right]}_{\text{(B1.1)}} +\nonumber\\&+\underbrace{\sum_{t=K\Gamma+1}^{T}\mathbb{E}\left[\frac{1-p_{1,t,\tau}^i}{p_{1,t,\tau}^i}\mathds{1}\left\{\overbrace{I_t=1,N_{1,t}^\Gamma\leq \sigma'(\tau)-\Gamma,N_{1,t,\tau}\ge \frac{8\log(T)}{(\overline{\mu}_1(\sigma'(\tau))-y_i)^2}}^{\mathcal{C}_2}\right\}\right]}_{\text{(B1.2)}}.
\end{align}
As $\mathbb{E}\left[XY\right] = \mathbb{E}\left[X\mathbb{E}\left[Y\mid X\right]\right]$, we can bound the term (B1.1) as follows:
\begin{align}
    \text{(B1.1)}&=\sum_{t=K\Gamma+1}^{T}\mathbb{E}\left[\mathds{1}\{\mathcal{C}_1\}\mathbb{E}\left[\frac{1-p_{1,t,\tau}^i}{p_{1,t,\tau}^i}\mid \mathcal{C}_1\right]\right]\\
    &=\mathbb{E}\left[\sum_{t=K\Gamma+1}^{T}\mathds{1}\{\mathcal{C}_1\}\mathbb{E}\left[\frac{1-p_{1,t,\tau}^i}{p_{1,t,\tau}^i}\mid \mathcal{C}_1\right]\right]\\
    &=\mathbb{E}\left[\sum_{t=K\Gamma+1}^{T}\mathds{1}\{\mathcal{C}_1\}\mathbb{E}\left[\frac{1-p_{1,t,\tau}^i}{p_{1,t,\tau}^i}\mid \mathcal{C}_1\right]\right]\label{eq:c1}\\
    &\leq \mathbb{E}\left[\sum_{t=K\Gamma+1}^{T}\mathds{1}\{\mathcal{C}_1\}\left(\frac{\delta_{TV}(P_{t\mid \mathcal{C}_1},Q_{t\mid \mathcal{C}_1})}{(1-\overline{\mu}_1(\sigma'(\tau)))^{\tau+1}}+\underbrace{\mathbb{E}_{\overline{\mu}_1(\sigma'(\tau))}\left[\frac{1-p_{1,t,\tau}^i}{p_{1,t,\tau}^i}\mid \mathcal{C}_1 \right]}_{\text{(D1.1)}}\right)\right],
\end{align}
where in line \eqref{eq:c1} we exploited the fact that the only summands that will contribute to the some are those for which condition $\mathcal{C}_1$ occurs, $P_{t\mid \mathcal{C}_1}$ is the distribution of the sample mean of the samples used at round $t$ to compute the sample mean of the expected reward of arm $1$, $Q_{t\mid \mathcal{C}_1}$ is a binomial distribution with $n$ equal to the number of samples used  used at round $t$ to compute the sample mean of the expected reward of arm $1$ and $p$ equal to   $\overline{\mu}_1(\sigma'(\tau))$.

Now consider an arbitrary instantiation $N_{1,t,\tau}'$ of $N_{1,t,\tau}$ (i.e., an arbitrary number of pulls of the optimal arm within the time window $\tau$) in which  $\mathcal{C}_1$ holds true, we can rewrite (D1.1) as:
\begin{align}
    \text{(D1.1)}=\mathbb{E}\left[\frac{1-p_{1,t,\tau}^i}{p_{1,t,\tau}^i}\mid \mathcal{C}_1\right]=\mathbb{E}_{N_{1,t,\tau}'}\left[\underbrace{\mathbb{E}\left[\frac{1-p_{1,t,\tau}^i}{p_{1,t,\tau}^i}\mid\mathcal{C}_1,N_{1,t,\tau}=N_{1,t,\tau}'\right]}_{(\text{E1.1})}\right].
\end{align}
We can bound (E1.1) using Lemma 4 by~\citet{agrawal2012analysis}:
\begin{align}
    &\text{(E.1)} = \sum_{s=0}^{N_{1,t,\tau}'}\frac{f_{N_{1,t,\tau}',\overline{\mu}_1(\sigma'(\tau))}(s)}{F_{N_{1,t,\tau}'+1,y_i}(s)}-1\nonumber\\
    &\leq \begin{cases}
    \frac{3}{\Delta_i'} & \textit{if  }  N_{1,t,\tau}'< \frac{8}{\Delta_i'}\\  
    \\
    {O}\left(e^{-\frac{\Delta_i'^2 N_{1,t,\tau}'}{2}}+\frac{e^{-D_iN_{1,t,\tau}'}}{N_{1,t,\tau}'\Delta_i'^2}+\frac{1}{e^{\Delta_i'^2\frac{N_{1,t,\tau}'}{4}}-1}\right) &\textit{if  } \frac{8}{\Delta'} \leq N_{1,t,\tau}' \leq \frac{8\log(T)}{\Delta_i'^2} \\ 
    \end{cases},
\end{align}
where $\Delta_i'=\overline{\mu}_1(\sigma'(\tau))-y_i$ and $D_i=y_i\log{\frac{y_i}{\overline{\mu}_{1}(\sigma'(\tau))}}+(1-y_i)\log{\frac{1-y_i}{1-\overline{\mu}_{1}(\sigma'(\tau))}}$.
We notice that for $N_{i,t,\tau}'<\frac{8}{\Delta_i'}$ the bound is independent from the number of pulls within the window while when $\frac{8}{\Delta_i}\leq N_{i,t,\tau}'\leq \frac{8\log(T)}{\Delta_i'^2}$ the bound on the expected value is decreasing for $N_{i,t,\tau}'$ increasing, then taking the worst case scenario the following will hold:
\begin{align}
    \text{(D1.1)} \leq \begin{cases}
    \frac{3}{\Delta_i'} & \textit{if  }  N_{1,t,\tau}'< \frac{8}{\Delta_i'}\\  
    \\
    {O}\left(e^{-\frac{\Delta_i'^2 \frac{8}{\Delta_i'}}{2}}+\frac{1}{\frac{8}{\Delta_i'}\Delta_i'^2}+\frac{1}{e^{\Delta_i'^2\frac{\frac{8}{\Delta_i'}}{4}}-1}\right) &\textit{if  } \frac{8}{\Delta'} \leq N_{1,t,\tau}' \leq \frac{8\log(T)}{\Delta_i'^2} \\ 
    \end{cases},
\end{align}
by exploiting the fact that for $x\ge 0$ will hold that $e^{-x}\leq \frac{1}{x}$ and that for all $x$ it holds that $e^x \ge 1+x$ we obtain a term which is independent from the number of pulls within a window $\tau$, so that: \begin{align}
      \text{(D1.1)} \leq O \left( \frac{1}{\Delta_i'}\right) =  {O}\left(\frac{1}{\overline{\mu}_1(\sigma'(\tau))-y_i}\right).
 \end{align}
Let us now consider the number of times the event $\mathcal{C}_1$ can occur in $T$ rounds. Notice that by definition of the event both inequalities will hold true:
\begin{align}
    \mathds{1}\{\mathcal{C}_1\}\leq\begin{cases}
       \mathds{1}\left\{I_t=1, N_{i,t,\tau}\leq \frac{8\log{(T)}}{(\overline{\mu}_1(\sigma'(\tau))-y_i)^2}\right\}\\
       \\
       \mathds{1}\left\{I_t=1, N_{i,t}^\Gamma\leq \sigma'(\tau)-\Gamma\right\}
    \end{cases},
\end{align}
then exploiting Lemma~\ref{lemma:window} we have:
\begin{align}
    \sum_{t=1}^T \mathds{1}\{\mathcal{C}1\}\leq  \sum_{t=1}^T\mathds{1}\left\{I_t=1, N_{i,t,\tau}\leq \frac{8\log{(T)}}{(\overline{\mu}_1(\sigma'(\tau))-y_i)^2}\right\}\leq\frac{8T\log(T)}{\tau(\overline{\mu}_1(\sigma'(\tau))-y_i)^2}.
\end{align}
Furthermore, with analogous considerations:
\begin{align}
    \sum_{t=1}^T \mathds{1}\{\mathcal{C}1\}\leq  \sum_{t=1}^T\mathds{1}\left\{I_t=1, N_{i,t}^\Gamma\leq \sigma'(\tau)-\Gamma \right\}\leq \sigma'(\tau)-\Gamma.
\end{align}
Finally, bounding the TV divergence term with $1$, we obtain:
\begin{align}
    \text{(B1.1)} \leq {O}\left(\frac{\sigma'(\tau)-\Gamma}{ (1-\overline{\mu}_1(\sigma'(\tau)))^{\tau+1}}+\frac{T\log(T)}{\tau (\overline{\mu}_1(\sigma'(\tau))-y_i)^3}\right),
\end{align}

Let us upper bound (B1.2):
\begin{align}
     \text{(B1.2)} &=\sum_{t=1}^{T}\mathbb{E}\left[\mathds{1}\{\mathcal{C}_2\}\mathbb{E}\left[\frac{1-p_{1,t,\tau}^i}{p_{1,t,\tau}^i}\mid\mathcal{C}_2\right]\right]   \\
     &=\mathbb{E}\left[\sum_{t=1}^{T}\mathds{1}\{\mathcal{C}_2\}\mathbb{E}\left[\frac{1-p_{1,t,\tau}^i}{p_{1,t,\tau}^i}\mid\mathcal{C}_2\right]\right] \\ 
     &\leq \mathbb{E}\left[\sum_{t=1}^{T}\mathds{1}\{\mathcal{C}_2\}\left(\frac{\delta_{TV}(P_{t\mid \mathcal{C}2},Q_{t\mid \mathcal{C}2})}{(1-\overline{\mu}_1(\sigma'(\tau)))^{\tau+1}}+\underbrace{\mathbb{E}_{\overline{\mu}_1(\sigma'(\tau))}\left[\frac{1-p_{1,t,\tau}^i}{p_{1,t,\tau}^i}\mid \mathcal{C}_2\right]}_{\text{(D1.2)}}\right)\right].
\end{align}
Let us consider an arbitrary instantiation $N_{1,t,\tau}'$ of $N_{1,t,\tau}$ in which $\mathcal{C}2$ holds true, i.e., an arbitrary number of pulls of the optimal arm within the time window $\tau$. We have:
\begin{align}
    \text{(D1.2)}=\mathbb{E}\left[\frac{1-p_{1,t,\tau}^i}{p_{1,t,\tau}^i}\mid  \mathcal{C}_2 \right]=\mathbb{E}_{N_{1,t,\tau}'}\left[\underbrace{\mathbb{E}\left[\frac{1-p_{1,t,\tau}^i}{p_{1,t,\tau}^i}\mid \mathcal{C}_2 ,N_{1,t,\tau}=N_{1,t,\tau}'\right]}_{\text{(E1.2)}}\right],
\end{align}
where we bound the term (E1.2) using Lemma~4 ~\citet{agrawal2012analysis}:
\begin{align}
   &\text{(E1.2)} = \sum_{s=0}^{N_{1,t,\tau}'}\frac{f_{N_{1,t,\tau}',\overline{\mu}_1(\sigma'(\tau))}(s)}{F_{N_{1,t,\tau}'+1,y_i}(s)}-1\nonumber\\ &\leq {O}\left(e^{-\frac{\Delta_i'^2 N_{1,t,\tau}'}{2}}+\frac{e^{-D_iN_{1,t,\tau}'}}{N_{1,t,\tau}'\Delta_i'^2}+\frac{1}{e^{\Delta_i'^2\frac{N_{1,t,\tau}'}{4}}-1}\right) \qquad \text{ for } N_{1,t,\tau}'\ge \frac{8\log(T)}{\Delta_i'^2},
\end{align}
with $\Delta_i'$ define ad before,
We see that the worst case scenario, when $\mathcal{C}_2$ holds true, is when $ N_{i,t,\tau}'= \frac{8\log(T)}{\Delta_i'^2}$. In such a case, we can bound the expected value for $\frac{1-p_{1,t,\tau}^i}{p_{1,t,\tau}^i}$ for every possible realization of $\mathcal{C}_2$ independently from $N_{1,t,\tau}'$ as:
\begin{align}
    \text{(D1.2)} \leq O \left( e^{-4\log(T)} + e^{-D_i (\Delta_i')^{-2} \log(T)} + \frac{1}{e^{2\log(T)} - 1} \right) =  {O}\left( \frac{1}{T}\right),
\end{align}
so that:
\begin{align}
  \text{(B1.2)} \leq {O}\left(\frac{\sigma'(\tau)-\Gamma}{ (1-\overline{\mu}_1(\sigma'(\tau)))^{\tau+1}}\right),
\end{align}
where the latter inequality is a consequence of the fact that:
\begin{align}
    \sum_{t=K\Gamma+1}^{T}\mathds{1}\{\mathcal{C}_2\})\leq \sigma'(\tau)-\Gamma.
\end{align}
Let us bound term (B2). We decompose this term in two contributions:
\begin{align}
    \text{(B2)} =\sum_{t=K\Gamma+1}^{T}\mathbb{E}\bigg[\frac{1-p_{1,t,\tau}^i}{p_{1,t,\tau}^i}\mathds{1}\{I_t=1,N_{1,t}^\Gamma\ge \sigma'(\tau)-\Gamma\}\bigg],
   \end{align}
so that, similarly to what we have done earlier, we have:
\begin{align}
    \text{(B2)} = \underbrace{\sum_{t=K\Gamma+1}^{T}\mathbb{E}\left[\frac{1-p_{1,t,\tau}^i}{p_{1,t,\tau}^i}\mathds{1}\left\{\overbrace{I_t=1,N_{1,t}^\Gamma\ge \sigma'(\tau)-\Gamma,N_{1,t,\tau}\leq \frac{8\log(T)}{(\overline{\mu}_1(\sigma'(\tau))-y_i)^2}}^{\mathcal{C}_1'}\right\}\right]}_{\text{(B2.1)}}+ \nonumber\\
    +\underbrace{\sum_{t=K\Gamma+1}^{T}\mathbb{E}\left[\frac{1-p_{1,t,\tau}^i}{p_{1,t,\tau}^i}\mathds{1}\left\{\overbrace{I_t=1,N_{1,t}^\Gamma\ge \sigma'(\tau)-\Gamma,N_{1,t,\tau}\ge \frac{8\log(T)}{(\overline{\mu}_1(\sigma'(\tau))-y_i)^2}}^{\mathcal{C}_2'}\right\}\right]}_{\text{(B2.2)}}.
\end{align}
Let us deal with (B2.1) first. We have:
\begin{align}
   \text{(B2.1)} = \mathbb{E}\left[\sum_{t=K\Gamma+1}^{T}\mathds{1}\{\mathcal{C}_1'\}\underbrace{\mathbb{E}\left[\frac{1-p_{1,t,\tau}^i}{p_{1,t,\tau}^i}\mid \mathcal{C}_1'\right]}_{\text{(D2.1)}}\right].
\end{align}
Let us analyze (D2.1):
\begin{align}
    \text{(D2.1)}\leq \mathbb{E}_{N_{1,t,\tau}'}\left[\mathbb{E}\underbrace{\left[\frac{1-p_{1,t,\tau}^i}{p_{1,t,\tau}^i}\mid \mathcal{C}_1', N_{1,t,\tau}=N_{1,t,\tau}'\right]}_{\text{(E2.1)}}\right]
\end{align}
Lemma~\ref{lemma:techlemma} applied to (E2.1), states that a bound on  a instance when the sum of successes is distributed as a binomial $\text{Bin}(N_{1,t,\tau}',\mu_1(\sigma'(\tau)))$,  holds also for (E2.1).
It follows, applying Lemma 4 by~\cite{agrawal2012analysis} to (E2.1), that:
$$\text{(E2.1)} \leq {O}\left(\frac{1}{\overline{\mu}_1(\sigma'(\tau))-y_i}\right).$$
Therefore we have by Lemma~\ref{lemma:window}:
\begin{align}
    \sum_{t=1}^{T}\mathds{1}\{\mathcal{C}_1'\}\leq {O}\left(\frac{T\log(T)}{\tau(\overline{\mu}_1(\sigma'(\tau))-y_i)^2}\right).
\end{align}
Finally, we obtain that:
\begin{align}
    \text{(B2.1)} \leq {O}\left(\frac{T\log(T)}{\tau(\overline{\mu}_1(\sigma'(\tau))-y_i)^3}\right).
\end{align}
Let us analyze (B2.2):
\begin{align}
\text{(B2.2)}=\mathbb{E}\left[\sum_{t=K\Gamma+1}^{T}\mathds{1}\{\mathcal{C}_2'\}\underbrace{\mathbb{E}\left[\frac{1-p_{1,t,\tau}^i}{p_{1,t,\tau}^i}\mid\mathcal{C}_2'\right]}_{\text{(D2.2)}}\right].
\end{align}
Let us analyze (D2.2), considering an instantiation $N_{1,t,\tau}=N_{1,t,\tau}'$ for which condition $\mathcal{C}_2'$ holds we have:
\begin{align}
    \mathbb{E}\left[\frac{1-p_{1,t,\tau}^i}{p_{1,t,\tau}^i}\mid \mathcal{C}_2'\right]=\mathbb{E}_{N_{1,t,\tau}'}\left[\underbrace{\mathbb{E}\left[\frac{1-p_{1,t,\tau}^i}{p_{1,t,\tau}^i}\mid \mathcal{C}_2',N_{1,t,\tau}=N_{1,t,\tau}'\right]}_{\text{(E2.2)}}\right]
\end{align}
Similarly to what has been done for term (B2.1), applying Lemma~\ref{lemma:techlemma} to (E2.2), we have that that term can be bounded by the same bound we would have for a process governed by a Binomial distribution $\text{Bin}(N_{i,t,\tau}',\overline{\mu}_1(\sigma'(\tau)))$. Thus, applying Lemma 4 by~\citet{agrawal2012analysis} to such a instantiation :
$$\text{(E2.2)} \leq {O}\left(\frac{1}{T}\right),$$
and, finally:
\begin{align}
    \text{(B2.2)}\leq {O}(1).
\end{align}
Summing all the term concludes the proof.
\end{proof}

\gtssw*

\begin{proof}
For ease of notation we set $\sigma'(T;\tau)=\sigma'(\tau), \overline{\mu}_{i^*(T)}(\sigma(T;\tau);\tau)=\overline{\mu}_{i^*(T)}(\sigma'(\tau))$ and $\Delta_i'(T;\tau)=\Delta_i$. Using Lemma \ref{lem:RDg}, we can bound the expected value of the number pulls of the suboptimal arm at the time horizon as:
\begin{align*}
        \mathbb{E}_{\tau}[N_{i,T}]\leq &\Gamma+\frac{1}{\epsilon_i}+\frac{T}{\tau}+\frac{\omega T}{\tau}\hspace{-0.12 cm}+\hspace{-0.12 cm}\underbrace{\mathbb{E}\left[\sum_{t=K\Gamma+1}^{T} \hspace{-0.3 cm}\mathds{1}\left \{ p_{i,t,\tau}^i>\frac{1}{T\epsilon_i}, I_t=i, N_{i,t,\tau}\ge \omega\right\}\right]}_{\text{(A)}}\hspace{-0.1 cm} \\
        & +\hspace{-0.1 cm}\underbrace{\mathbb{E}\left[\sum_{t=K\Gamma+1}^{T}\left(\frac{1}{p_{i^*(t),t,\tau}^i}-1\right)\mathds{1}\left\{I_t=i^*(t) \right\}\right]}_{\text{(B)}}.
    \end{align*}
    In what follows, we will consider, without loss of generality, the first arm to be the best, i.e., for each round $t \in \dsb{T}$ we have that $i^*(t)=i^*(T)=1$. Furthermore we will set $y_{i,t}=y_i=\overline{\mu}_1(\sigma'(\tau))-\frac{\Delta_i}{3}$, $x_i=\mu_i(T)+\frac{\Delta_i}{3}$, $\omega=\frac{32\log{(T\Delta_i^2+e^6)}}{\gamma(y_i-x_i)^2}$. We define $\overline{\hat{\mu}}_{i,t,\tau}=\frac{S_{i,t,\tau}}{N_{i,t,\tau}}$ We further decompose (A) in two contributions, formally:
    \begin{align}
\text{(A)}& =\underbrace{\mathbb{E}\left[\sum_{t=K\Gamma+1}^{T} \hspace{-0.3 cm}\mathds{1}\left \{ p_{i,t,\tau}^i>\frac{1}{T\epsilon_i},\overline{\hat{\mu}}_{i,t,\tau}\leq x_i, I_t=i, N_{i,t,\tau}\ge \omega\right\}\right]}_{\text{(A1)}}\nonumber\\
& +\underbrace{\mathbb{E}\left[\sum_{t=K\Gamma+1}^{T} \hspace{-0.3 cm}\mathds{1}\left \{ p_{i,t,\tau}^i>\frac{1}{T\epsilon_i},\overline{\hat{\mu}}_{i,t,\tau}>x_i, I_t=i, N_{i,t,\tau}\ge \omega\right\}\right]}_{\text{(A2)}}.
\end{align}
Let us first tackle term (A2).
\paragraph{Term (A2)}
Focusing on (A2):
\begin{align}
   \text{(A2)} &\leq \mathbb{E}\left[\sum_{t=K\Gamma+1}^T \mathds{1}\left\{\hat{\mu}_{i,t,T}> x_i ,   N_{i,t,\tau}\ge \omega\right\}\right]\\
   &\leq \sum_{K\Gamma+1}^T\Prob(\hat{\mu}_{i,t,\tau}> x_i | N_{i,t,\tau}\ge \omega)\\
   &\leq \sum_{K\Gamma+1}^T\Prob(\hat{\mu}_{i,t,\tau}-\mathbb{E}[\hat{\mu}_{i,t,\tau}]> x_i-\mathbb{E}[\hat{\mu}_{i,t,\tau}] | N_{i,t,\tau}\ge \omega)\\
    &\leq \sum_{K\Gamma+1}^T\Prob(\hat{\mu}_{i,t,\tau}-\mathbb{E}[\hat{\mu}_{i,t,\tau}]> x_i-\mu_i(T) | N_{i,t,\tau}\ge \omega)\\
    &\leq \sum_{t=1}^T \exp{(-2\omega(x_i-\mu_i(T))^2)},
\end{align}
where the last inequality follows from the Chernoff-Hoeffding bound (Lemma \ref{lemma:chernoff}). Substituting $\omega$ and noticing that $x_i-\mu_i(T)=y_i-x_i$, we obtain:
\begin{align}
\sum_{t=K\Gamma+1}^T \exp{(-2\omega(x_i-\mu_i(T))^2)}\leq\frac{1}{\Delta_i^2}.
\end{align}
\paragraph{Term (A1)}
We focus now on (A1):
\begin{align}
    \text{(A1)} \leq \mathbb{E}\left[\sum_{t=K\Gamma+1}^{T} \hspace{-0.3 cm}\mathds{1}\left \{ \underbrace{p_{i,t,\tau}^i>\frac{1}{T\epsilon_i}, N_{i,t,\tau}\ge \omega, ,\overline{\hat{\mu}}_{i,t,\tau}\leq x_i}_{\text{(C)}}\right\}\right].
\end{align}
We wish to evaluate if ever condition (C) occurs.
In this setting, $\theta_{i,t,\tau}$ is a Gaussian random variable distributed as $\mathcal{N}\left(\overline{\hat{\mu}}_{i,t,\tau}, \frac{1}{\gamma N_{i,t,\tau}}\right)$. We recall that an $\mathcal{N}\left(m, \sigma^2\right)$ distributed r.v. ~(i.e., a Gaussian random variable with mean $m$ and variance $\sigma^2$ ) is stochastically dominated by $\mathcal{N}\left(m^{\prime}, \sigma^2\right)$ distributed r.v.~if $m^{\prime} \geq m$ (Lemma \ref{lem:nbord}). Therefore, given $\overline{\hat{\mu}}_{i,t,\tau} \leq x_i$, the distribution of $\theta_{i,t,\tau}$ is stochastically dominated by $\mathcal{N}\left(x_i, \frac{1}{\gamma N_{i,t,\tau}}\right)$. Formally:
\begin{equation} \label{eq:domi}
\Prob\left(\theta_{i,t,\tau}>y_i \mid N_{i,t,\tau}>\omega, \overline{\hat{\mu}}_{i,t,\tau} \leq x_i, \mathcal{F}_{t-1}\right) \leq \Prob\left(\left.\mathcal{N}\left(x_i, \frac{1}{\gamma N_{i,t,\tau}}\right)>y_i \right\rvert\, \mathcal{F}_{t-1}, N_{i,t,\tau}>\omega\right) .
\end{equation}

Using Lemma \ref{lemma:Abramowitz2} we have:
\begin{align} \Prob\left(\mathcal{N}\left(x_i, \frac{1}{\gamma N_{i,t,\tau}}\right)>y_i\right) & \leq \frac{1}{2} e^{-\frac{\left(\gamma N_{i,t,\tau}\right)\left(y_i-x_i\right)^2}{2}} \\ & \leq \frac{1}{2} e^{-\frac{\left(\gamma \omega\right)\left(y_i-x_i\right)^2}{2}}\end{align}
which is smaller than $\frac{1}{T \Delta_i^2}$ because $\omega\geq \frac{2 \log \left(T \Delta_i^2\right)}{\gamma \left(y_i-x_i\right)^2}$. Substituting into Equation~\eqref{eq:domi}, we get:
\begin{equation}
   \Prob\left(\theta_{i,t,\tau}>y_i \mid N_{i,t,\tau}\ge\omega, \overline{\hat{\mu}}_{i,t,\tau} \leq x_i, \mathcal{F}_{t-1}\right)  \leq \frac{1}{T\Delta_i^2},
\end{equation}
in contradiction with the condition (C).
\paragraph{Term (B)}
We decompose term (B) into two contributions:
\begin{align}
   \text{(B)} \leq \underbrace{\sum_{t=K\Gamma+1}^T\mathbb{E}\bigg[\frac{1-p_{1,t,\tau}^i}{p_{1,t,\tau}^i}\mathds{1}\{I_t=1, N_{1,t}^\Gamma\leq \sigma'(\tau)-\Gamma\}\bigg]}_{\text{(B1)}}+\nonumber\\+\underbrace{\sum_{t=k\Gamma+1}^T\mathbb{E}\bigg[\frac{1-p_{1,t,\tau}^i}{p_{1,t,\tau}^i}\mathds{1}\{I_t=1,N_{1,t}^\Gamma\geq \sigma'(\tau)-\Gamma\}\bigg]}_{\text{(B2)}},
\end{align}
with $N_{1,t}^\Gamma=\sum_{t=K\Gamma+1}^T\mathds{1}\{I_t=1\}$.
Analyzing term (B1), we further decompose it in two contributions:
\begin{align}
    \text{(B1)} \leq \underbrace{\sum_{t=K\Gamma+1}^T\mathbb{E}\bigg[\frac{1-p_{1,t,\tau}^i}{p_{1,t,\tau}^i}\mathds{1}\left\{\overbrace{I_t=1,N_{1,t,\tau}\leq\omega, N_{1,t}^\Gamma\geq \sigma'(\tau)-\Gamma}^{\mathcal{C}_1}\right\}\bigg]}_{\text{(B1.1)}}+\nonumber\\+\underbrace{\sum_{t=K\Gamma+1}^T\mathbb{E}\bigg[\frac{1-p_{1,t,\tau}^i}{p_{1,t,\tau}^i}\mathds{1}\left\{\overbrace{I_t=1,N_{1,t,\tau}\ge \omega, N_{1,t}^\Gamma\geq \sigma'(\tau)-\Gamma}^{\mathcal{C}_2}\right\}\bigg]}_{\text{(B1.2)}} \label{eq:c1def}
\end{align}
Let us tackle the term (B1.1) by exploiting the fact that $\mathbb{E}[XY]=\mathbb{E}[X\mathbb{E}[Y\mid X]]$. This way, we can rewrite it as:
\begin{align}\label{eq:rippop}
    \text{(B1.1)}=\mathbb{E}\left[\sum_{t=K\Gamma+1}^T\mathds{1}\{\mathcal{C}_1\}\underbrace{\mathbb{E}\left[\frac{1-p_{1,t,\tau}^i}{p_{1,t,\tau}^i}\mid \mathcal{C}_1\right]}_{\text{(D1.1)}}\right].
\end{align}

In the following, we show that whenever condition $\mathcal{C}_1$ holds, (D1.1) is bounded by a constant. Let $\Theta_j$ denote a $\mathcal{N}\left(\overline{\hat{\mu}}_{1,j}, \frac{1}{\gamma j}\right)$ distributed Gaussian random variable, where $\overline{\hat{\mu}}_{1,j}$ is the sample mean of the optimal arm's rewards played $j$ times within a time window $\tau$. Let $G_j$ be a geometric random variable denoting the number of consecutive independent trials up to $j$ included where a sample of $\Theta_j$ is greater than $y_i$. We will show that for any realization of the number of pulls within a time window $\tau$ such that condition $\mathcal{C}_1$ holds, the expected value of $G_j$ is bounded by a constant for all $j$.

Consider an arbitrary realization of $N_{1,t,\tau} = j$ that satisfies condition $\mathcal{C}_1$. 
\begin{align}
    \mathbb{E}\left[\frac{1}{p_{1, t,\tau}^i}\mid \mathcal{C}_1\right]= \mathbb{E}_{j}\left[\mathbb{E}\left[\frac{1}{p_{1, t,\tau}^i}\mid \mathcal{C}_1, N_{1,t,\tau}=j\right]\right]=\mathbb{E}_{j_{\mid \mathcal{C}_1}}\left[\mathbb{E}\left[\mathbb{E}\left[G_j \mid \mathcal{F}_{\tau_j}\right]\right]\right]=\mathbb{E}_{j_{\mid \mathcal{C}_1}}\left[\mathbb{E}\left[G_j\right]\right]. \label{eq:expgj}
\end{align}
Notice that the term $\mathbb{E}\left[G_j\right]$ in Equation~\eqref{eq:expgj} is the same as the one we had in Equation~\eqref{eq:rif} to derive bounds for the $\gamma$-\texttt{ET-GTS} algortihtm. Relying on the same mathematical steps we bound it as follows:
$$\mathbb{E}\left[G_j\right] \leq e^{12}+5.$$
This shows a constant bound independent from $j$ of $\mathbb{E}\left[\frac{1}{p_{1, t,\tau}}-1\right]$ for any $j$ such that condition $\mathcal{C}_1$ holds. Then, using Lemma~\ref{lemma:window}, (B1.2) can be rewritten as:
\begin{align}
    \text{(B1.1)}&\leq (e^{12}+5)\mathbb{E}\left[\sum_{t=1}^T\mathds{1}\{\mathcal{C}1\}\right]\\
    &\leq (e^{12}+5)\frac{288T\log(T\Delta_i^2+e^6)}{\gamma\tau\Delta_i^2}.
\end{align}
Let us tackle {(B1.2)} by exploiting the fact that $\mathbb{E}[XY]=\mathbb{E}[X\mathbb{E}[Y\mid X]]$:
\begin{align}\label{eq:ripopp2}
     \text{(B1.2)}=\mathbb{E}\left[\sum_{t=K\Gamma+1}^T\mathds{1}\{\mathcal{C}2\}\underbrace{\mathbb{E}\left[\frac{1-p_{1,t,\tau}^i}{p_{1,t,\tau}^i}\mid \mathcal{C}_2\right]}_{\text{(D1.2)}}\right].
\end{align}
We derive a  bound for (D1.2) for large $j$ as imposed by condition $\mathcal{C}_2$. Consider then an arbitrary case in which $N_{i,t,\tau} = j \geq\omega$ (as dictated by $\mathcal{C}_2$), we have:
\begin{align}
    \mathbb{E}\left[\frac{1}{p_{1, t,\tau}^i}\mid \mathcal{C}_2\right]= \mathbb{E}_{j}\left[\mathbb{E}\left[\frac{1}{p_{1, t,\tau}^i}\mid \mathcal{C}_2, T_{1,t,\tau}=j\right]\right]=\mathbb{E}_{j_{\mid \mathcal{C}_2}}\left[\mathbb{E}\left[\mathbb{E}\left[G_j \mid \mathbb{F}_{\tau_j}\right]\right]\right]=\mathbb{E}_{j_{\mid \mathcal{C}_2}}\left[\mathbb{E}\left[G_j\right]\right].
\end{align}

Notice that the term $\mathbb{E}\left[G_j\right]$ in the last equation is the same we bounded in Theorem~\ref{thm:gts} for the regret of $\gamma$-\texttt{ET-GTS}. Therefore, using the same proof line it is bounded by $\mathbb{E}\left[G_j\right] \leq \frac{1}{T\Delta_i^2}$.

For term (B2), we further decompose in two contributions:
\begin{equation}
    \text{(B2)} \leq O \left( \frac{\sigma'(T;\tau)-\Gamma}{\mathrm{erfc}(\sqrt{{\gamma\tau}/{2}}(\overline{\mu}_1(\sigma'(T;\tau),\tau))} + (e^{12}+5)\frac{288T\log(T\Delta_i^2+e^6)}{\gamma\tau\Delta_i^2} + \frac{1}{\Delta_i^2} \right),
\end{equation}
by loosely bounding the total variation with $1$ and noticing that both this condition holds true:
\begin{align}
\text{(B2)} \leq \underbrace{\sum_{t=K\Gamma+1}^T\mathbb{E}\left[\frac{1-p_{1,t,\tau}^i}{p_{1,t,\tau}^i}\mathds{1}\left(\overbrace{I_t=1, N_{1,t}^\Gamma\leq \sigma'(\tau)-\Gamma, N_{1,t,\tau}\leq \omega}^{\mathcal{C}_1'}\right)\right]}_{\text{(B2.1)}}+\nonumber\\+\underbrace{\sum_{t=K\Gamma+1}^T\mathbb{E}\left[\frac{1-p_{1,t,\tau}^i}{p_{1,t,\tau}^i}\mathds{1}\left(\overbrace{I_t=1, N_{1,t}^\Gamma\leq \sigma'(\tau)-\Gamma, N_{1,t,\tau}\geq \omega}^{\mathcal{C}_2'}\right)\right]}_{\text{(B2.2)}}      .
\end{align}
Tackling the two terms separately we have:
\begin{align}
    \text{(B2.1)}&\leq \mathbb{E}\left[\sum_{t=K\Gamma+1}^T \mathds{1}\left\{\mathcal{C}_1'\right\}\mathbb{E}\left[\frac{1-p_{1,t,\tau}^i}{p_{1,t,\tau}^i}\bigg|\mathcal{C}_1'\right]\right]\\
    &\leq \mathbb{E}\left[\sum_{t=K\Gamma+1}^{T}\mathds{1}\{\mathcal{C}_1'\}\left(\frac{2\delta_{TV}(P_{t\mid \mathcal{C}_1'},Q_{t\mid \mathcal{C}_1'})}{\mathrm{erfc}(\sqrt{{\gamma \tau}/{2}}\overline{\mu}_1(\sigma'(\tau)))}+\underbrace{\mathbb{E}_{\overline{\mu}_1(\sigma'(\tau))}\left[\frac{1-p_{1,t,\tau}^i}{p_{1,t,\tau}^i}\mid \mathcal{C}1'\right]}_{\text{(D2.1)}}\right)\right].
\end{align}
Notice that in the proof of \texttt{$\gamma$-ET-GTS} we have already proven for (D2.1) the following:
\begin{align}
    \text{(D2.1)} \leq (e^{12}+5).
\end{align}
Furthermore, we note that both the conditions hold true by Lemma \ref{lemma:window}:
\begin{align}
\sum_{t=K\Gamma+1}^T\mathds{1}\left\{\mathcal{C}1'\right\}\leq \begin{cases}
    \sigma'(\tau)-\Gamma\\
    \\
    \omega
\end{cases},
\end{align}
so, by loosely bounding the total variation term with $1$, we obtain:
\begin{align}
    \text{(B2.1)} \leq O\left(\frac{\sigma'(\tau)-\Gamma}{\mathrm{erfc}(\sqrt{{\gamma \tau}/{2}}\overline{\mu}_1(\sigma'(\tau)))}+(e^{12}+5)\omega\right).
\end{align}
Facing the term (B2.2) we have:
\begin{align}
    \text{(B2.2)}&\leq \mathbb{E}\left[\sum_{t=K\Gamma+1}^T\mathds{1}\left\{\mathcal{C}_2'\right\}\mathbb{E}\left[\frac{1-p_{1,t,\tau}^i}{p_{1,t,\tau}^i}\bigg|\mathcal{C}_2'\right]\right]\\
    &\leq \mathbb{E}\left[\sum_{t=K\Gamma+1}^{T}\mathds{1}\{\mathcal{C}_2'\}\left(\frac{2\delta_{TV}(P_{t\mid \mathcal{C}_2'},Q_{t\mid \mathcal{C}_2'})}{\mathrm{erfc}(\sqrt{{\gamma \tau}/{2}}\overline{\mu}_1(\sigma'(\tau)))}+\underbrace{\mathbb{E}_{\overline{\mu}_1(\sigma'(\tau))}\left[\frac{1-p_{1,t,\tau}^i}{p_{1,t,\tau}^i}\mid \mathcal{C}2'\right]}_{\text{(D2.2)}}\right)\right],
\end{align}
where we have already proved for (D2.2) in the proof of \texttt{$\gamma$-ET-GTS} that the following holds:
\begin{align}
    \text{(D2.2)} \leq \frac{1}{T\Delta_i^2}.
\end{align}
The statement follows by noticing that:
\begin{align}
   \sum_{t=K\Gamma+1}^T\mathds{1}\left\{\mathcal{C}_2'\right\}\leq\sigma'(\tau)-\Gamma ,
\end{align}
and summing all the terms.
\end{proof}

\clearpage
\section{Auxiliary Lemmas}
In this section, we report some results that already exist in the bandit literature and have been used to demonstrate our results.

\begin{restatable}[Generalized Chernoff-Hoeffding bound from~\cite{agrawal2017near}]{lemma}{chernoff}\label{lemma:chernoff}
Let $X_1, \ldots , X_n$ be independent Bernoulli random variables with $\mathbb{E}[X_i
] = p_i$, consider the random variable $X = \frac{1}{n}\sum_{i=1}^nX_i$, with $\mu = \mathbb{E}[X]$.
For any $0 < \lambda < 1 - \mu$ we have:
\[
\Prob(X\ge\mu+\lambda)\leq \exp{\big(-nd(\mu+\lambda,\mu)\big)},
\]
and for any $0 < \lambda < \mu$
\[
\Prob(X\leq\mu-\lambda)\leq \exp{\big(-nd(\mu-\lambda,\mu)\big)},
\]
where $d(a, b) \coloneqq a \log{\frac{a}{b}} + (1-a) \log{\frac{1-a}{1-b}}$.
\end{restatable}

\begin{restatable}[Change of Measure Argument from~\cite{lattimore2020bandit}]{lemma}{distance}\label{lemma:changeMeasure}
Let $(\Omega, \mathcal{F})$ be a measurable space, and $P, Q: \mathcal{F}\rightarrow [0,1]$. Let $a < b$ and $X \rightarrow [a, b]$ be a $\mathcal{F}$-measurable random variable, we have:
\begin{equation}
    \left| \int_{\Omega} X(\omega) \,dP(\omega) \ - \int_{\Omega} X(\omega) \,dQ(\omega) \right| \leq (b-a)\delta_{TV}(P,Q).
\end{equation}
\end{restatable}
\begin{restatable}[\cite{lattimore2020bandit}, proposition 2.8]{lemma}{mean}\label{lemma:change}
    For a nonnegative random variable $X$, the expected value $\mathbb{E}[X]$ can be computed as:
\[
\mathbb{E}[X] = \int_{0}^{+\infty} \Prob(X > y) \de y.
\]
\end{restatable}

\begin{restatable}[\cite{roos2001binomialapproximation}, Theorem 2]{lemma}{deltadistancepb} \label{lemma:delta}
Let us define $\underline{\mu}_n \coloneqq (\mu_1, \ldots, \mu_{n})$, $s \in (0, \ldots, n)$ and $\mu \in (0,1)$. We have that the total variation distance between two variables $PB(\underline{\mu}_{n})$ and $B_s(n, \mu)$ is:
\begin{equation}
     \delta_{TV}(PB(\underline{\mu}_{n}), B_s(n,\mu)) \leq
     \begin{cases}
        C_1(s) \theta(\mu,\underline{\mu}_n)^{\frac{s+1}{2}} \frac{\left(1-\frac{s}{s+1}\sqrt{\theta(\mu,\underline{\mu}_n)}\right)}{(1-\sqrt{\theta (\mu,\underline{\mu}_n)})^2} \quad \text{if } \theta(\mu,\underline{\mu}_n)< 1\\
        C_2(s) \eta(\mu,\underline{\mu}_n)^{\frac{s+1}{2}}(1+\sqrt{2\eta(\mu,\underline{\mu}_n)})\exp(2\eta(\mu,\underline{\mu}_n)) &\text{otherwise }
    \end{cases},
\end{equation}
where $\theta(\mu, \underline{\mu}_n) \coloneqq \frac{\eta(\mu, \underline{\mu}_n)}{2 n \mu (1 - \mu)}$, $\eta(\mu, \underline{\mu}_n) \coloneqq 2 \gamma_{2}(\mu, \underline{\mu}_n) + \gamma_{1}(\mu, \underline{\mu}_n)^2$, $\gamma_{k}(\mu, \underline{\mu}_n) \coloneqq \sum_{n'=1}^{n} (\mu - \mu_{n'})^k$, $C_{1}(s) \coloneqq \frac{\sqrt{e}(s+1)^{\frac{1}{4}}}{2}$, $C_2(s) \coloneqq \frac{(2\pi)^{\frac{1}{4}} \exp{(\frac{1}{24(s+1)}})2^{\frac{s-1}{2}}}{\sqrt{s!}(s+1)^{\frac{1}{4}}}$. To obtain the binomial distribution  is sufficient to set $s=0$ in  $B_s(n,\mu)$.
\end{restatable}


\begin{restatable}[Beta-Binomial identity]{lemma}{betabin} \label{lem:betabin}
    For all positive integers $\alpha, \beta \in \mathbb{N}$, the following equality holds:
    \begin{equation}
        F_{\alpha, \beta}^{beta}(y) = 1 - F_{\alpha + \beta - 1, y}^B(\alpha - 1),
    \end{equation}
    where $F_{\alpha, \beta}^{beta}(y)$ is the cumulative distribution function of a beta with parameters $\alpha$ and $\beta$, and $F_{\alpha + \beta - 1, y}^B(\alpha - 1)$ is the cumulative distribution function of a binomial variable with $\alpha + \beta - 1$ trials having each probability $y$.
\end{restatable}

\begin{restatable}
    [\cite{boland2002stochastic}, Theorem 1 (iii)] {lemma}{stochastic} \label{lemma:stochastic}
    Let $Y \sim Bin(n,\lambda)$ and $X=\sum X_i$ where the $X_i \sim Bin(n_i,\lambda_i)$ are independent random variables for $i=1,\ldots,k$ then:
    \begin{align}
        & X \ge_{st}Y \textit{if and only if } \lambda\leq\overline{\lambda}_g, \\
        & X \leq_{st}Y \textit{if and only if } \lambda\ge\overline{\lambda}_{cg}, 
    \end{align}
    where $X\ge_{st}Y$ means that $X$ is greater than $Y$ in the stochastic order, i.e. $\Prob(X> m)\ge \Prob(Y> m) \text{ }\forall m $, and:
    \begin{align}
        & \overline{\lambda}_g =\left(\prod_{i=1}^k\lambda_i^{n_i}\right)^{\frac{1}{n}}, \\
        & \overline{\lambda}_{cg} =1-\left(\prod_{i=1}^k(1-\lambda_i)^{n_i}\right)^{\frac{1}{n}}.
    \end{align}
\end{restatable}

\begin{restatable}[\cite{abramowitz1968handbook} Formula $7.1.13$]{lemma}{Abramowitz}\label{lemma:Abramowitz}
Let $Z$ be a Gaussian random variable with mean $\mu$ and standard deviation $\sigma$, then:
\begin{equation}
    \Prob(Z>\mu+x\sigma)\ge \frac{1}{\sqrt{2\pi}}\frac{x}{x^2+1}e^{-\frac{x^2}{2}}
\end{equation}
\end{restatable}

\begin{restatable}[\cite{abramowitz1968handbook}]{lemma}{Abramowitz2}\label{lemma:Abramowitz2}
Let $Z$ be a Gaussian r.v.~with mean $m$ and standard deviation $\sigma$, then:
    \begin{equation}
        \frac{1}{4 \sqrt{\pi}} e^{-7 z^2 / 2}<\Prob(|Z-m|>z \sigma) \leq \frac{1}{2} e^{-z^2 / 2}.
    \end{equation}
\end{restatable}

\begin{restatable}[\cite{rigollet2023high} Corollary $1.7$]{lemma}{Subg}\label{lemma:Subg}
Let $X_1,\ldots, X_n$ be $n$ independent random variables such that $X_i\sim $ \textsc{Subg}($\sigma^2$), then for any $a \in \mathbb{R}^n$, we have
\begin{equation}
\Prob\left[\sum_{i=1}^n a_i X_i>t\right] \leq \exp \left(-\frac{t^2}{2 \sigma^2|a|_2^2}\right),    
\end{equation}

and
\begin{equation}
    \Prob\left[\sum_{i=1}^n a_i X_i<-t\right] \leq \exp \left(-\frac{t^2}{2 \sigma^2|a|_2^2}\right)
\end{equation}

Of special interest is the case where $a_i=1 / n$ for all $i$ we get that the average $\bar{X}=\frac{1}{n} \sum_{i=1}^n X_i$, satisfies
$$
\Prob(\bar{X}>t) \leq e^{-\frac{n t^2}{2 \sigma^2}} \quad \text { and } \quad \mathbb{P}(\bar{X}<-t) \leq e^{-\frac{n t^2}{2 \sigma^2}}
$$
\end{restatable}
\begin{restatable}[\cite{hoeffding1956trials},\cite{tang2023pb}, Theorem 2.1 (2)]{lemma}{hoeffpomp}\label{lemma:hoeffpomp}
Let $X \sim \operatorname{PB}\left(p_1, \ldots, p_n\right)$, and $\bar{X} \sim \operatorname{Bin}(n, \bar{p})$ with $\bar{p} = \frac{1}{n} \sum_{i=1}^n p_i$, for any convex function $g:[n] \rightarrow \mathbb{R}$ in the sense that $g(k+2)-2 g(k+1)+g(k)>0$ for all $0 \leq k \leq n-2$, we have:
\begin{equation}
    \mathbb{E} g(X) \leq \mathbb{E} g(\bar{X}),
\end{equation}
where the equality holds if and only if $p_1=\cdots=p_n$ of the Poisson-binomial distribution are all equal to $\bar{p}$ of the binomial distribution.
\end{restatable}
\begin{restatable}[\cite{Johnson_2006} Definition 1.2, \cite{HOGGAR1974concavity}] {lemma}{con}\label{lemma:con}
A random variable V taking values in $\mathbb{Z}_+$ is discrete log-concave if its probability
mass function $p_V (i) = P(V = i)$ forms a log-concave sequence. That is, $V$ is log-concave
if for every $i \ge 1$:
\begin{align}
    p_V(i)^2 \ge p_V(i-1) p_V(i+1)
\end{align}
Any Bernoulli random variable (i.e., only taking values in $\{0, 1\}$) is discrete log-concave.
Furthermore, any binomial distribution is discrete log-concave. In fact, any random variable $S = \sum_{i=1}^n X_i$, where $X_i$ are independent (not necessarily identically distributed) Bernoulli variables, is discrete log-concave. Notice then that, by definition $\frac{1}{p_V(i)}$, is discrete log-convex.
\end{restatable}
\begin{restatable}[\cite{Hill1999AdvancesIS}, Theorem 2 p.152, Remark 13 p.153, Remark 1 p.150] {lemma}{logcon} \label{lemma:logcon}
 Let $1 \leq \alpha<r \leq +\infty$ and let $q: \mathbb{Z} \rightarrow[0, +\infty]$ be $r$-concave (Definition $1$ p.150 \cite{Hill1999AdvancesIS}; we highlight that for Remark 1 p.150 \cite{Hill1999AdvancesIS} to be $\infty$-concave is equivalent to be discrete log-concave). Then, $\mathcal{J}^\alpha {q}$ is $(r-\alpha)$-concave, we assume $r-\alpha=+\infty$ when $r=+\infty$ and $r>\alpha$. Where the $\alpha$-fractional (tail) sum of a function $q: \mathbb{Z} \rightarrow [0,\infty]$ is defined for every $\alpha> 0$ by the formula:
 \begin{align}
     \mathcal{J}^\alpha {q}(n)=\sum_{k=0}^{\infty} \binom{\alpha+k-1}{k} q(n+k),
 \end{align}
 so that for a binomial pdf $p_{\textit{bin}}$, being $p_{\textit{bin}}$ discrete log-concave (see Lemma \ref{lemma:con}), follows that $\mathcal{J}^\alpha p_{\textit{bin}}(-n)$ is $\infty$-concave on $\mathbb{Z}$ for $\alpha \ge 1$.
\end{restatable}

\begin{restatable}[\citet{combes2014unimodal}, Lemma D.1]{lemma}{combprot}\label{lemma:window}
     Let $A \subset \mathbb{N}$, and $\tau \in \mathbb{N}$ fixed. Define $a(n)=$ $\sum_{t=n-\tau}^{n-1} \mathds{1}\{t \in A\}$. Then for all $T \in \mathbb{N}$ and $s \in \mathbb{N}$ we have the inequality:
     \begin{align}
       \sum_{n=1}^T \mathds{1}\{n \in A, a(n) \leq s\} \leq s\lceil T / \tau\rceil .  
     \end{align}
\end{restatable}
\begin{restatable}[\cite{fiandri2024rising}, Lemma 4.1]{lemma} {lbmf}\label{lem:lbmf}
   Let $T \in \mathbb{N}$ be the time horizon, $\pi$ a fixed policy, and $\bm{\mu}$ a rising rested bandit. Then, it holds that:
\begin{align}
    R_{\bm{\mu}}(\pi,T) \ge \sum_{i \neq i^{\star}(T)} \overline{\Delta}_i(T,T)\mathbb{E}_{\bm{\mu}}^\pi[N_{i,T}],
\end{align}

\end{restatable}
\begin{restatable}[Bretagnolle-Huber inequality and List of KL-Divergences by \citet{gil2013renyi}]{lemma} {gil}\label{lem:gil}
Let $P$ and $Q$ be two probability distributions on a measurable space $(\Omega, \mathcal{F})$. Recall that the total variation between $P$ and $Q$ is defined by
$$
\delta_{TV}(P, Q)=\sup _{A \in \mathcal{F}}\{|P(A)-Q(A)|\}
$$

The Kullback-Leibler divergence is defined as follows:
$$
D_{\mathrm{KL}}(P \| Q)= \begin{cases}\int_{\Omega} \log \left(\frac{d P}{d Q}\right) d P & \text { if } P \ll Q \\ +\infty & \text { otherwise }\end{cases}
$$

In the above, the notation $P \ll Q$ stands for absolute continuity of $P$ with respect to $Q$, and $\frac{d P}{d Q}$ stands for the Radon-Nikodym derivative of $P$ with respect to $Q$.

The Bretagnolle-Huber inequality states:
$$
\delta_{TV}(P, Q) \leq \sqrt{1-\exp \left(-D_{\mathrm{KL}}(P \| Q)\right)} \leq 1-\frac{1}{2} \exp \left(-D_{\mathrm{KL}}(P \| Q)\right).
$$
A comprehensive list of KL-Divergences is provided in Figure \ref{fig:divergenze}. 
In Figure \ref{fig:divergenze}, $\Gamma(x)$, $\psi(x)$, and $B(x)$ denote the Gamma, Digamma, and the multivariate Beta functions, respectively. Also, $\gamma \approx  0.5772$ is the Euler-Mascheroni constant.
\begin{figure}
\centering
\begin{tabular}{|c|c|}
\hline Name & $D_{KL}\left(f_i \| f_j\right)$ \\
\hline Beta & 
$\begin{aligned}
& \log \frac{B\left(a_j, b_j\right)}{B\left(a_i, b_i\right)}+\psi\left(a_i\right)\left(a_i-a_j\right)+\psi\left(b_i\right)\left(b_i-b_j\right) \\
& +\left[a_j+b_j-\left(a_i+b_i\right)\right] \psi\left(a_i+b_i\right)
\end{aligned}$ \\
\hline Chi & $\frac{1}{2} \psi\left(k_i / 2\right)\left(k_i-k_j\right)+\log \left[\left(\frac{\sigma_j}{\sigma_i}\right)^{k_j} \frac{\Gamma\left(k_j / 2\right)}{\Gamma\left(k_i / 2\right)}\right]+\frac{k_i}{2 \sigma_j^2}\left(\sigma_i^2-\sigma_j^2\right)$ \\
\hline $\chi^2$ & $\log \frac{\Gamma\left(d_j / 2\right)}{\Gamma\left(d_i / 2\right)}+\frac{d_i-d_j}{2} \psi\left(d_i / 2\right)$ \\
\hline Cramér & $\frac{\theta_i+\theta_j}{\theta_i-\theta_j} \log \frac{\theta_i}{\theta_j}-2$ \\

\hline Dirichlet & $\log \frac{B\left(\boldsymbol{a}_j\right)}{B\left(\boldsymbol{a}_i\right)}+\sum_{k=1}^d\left[a_{i_k}-a_{j_k}\right]\left[\psi\left(a_{i_k}\right)-\psi\left(\sum_{k=1}^d a_{i_k}\right)\right]$ \\
\hline Exponential & $\log \frac{\lambda_i}{\lambda_j}+\frac{\lambda_j-\lambda_i}{\lambda_i}$ \\
\hline Gamma & $\left(\frac{\theta_i-\theta_j}{\theta_j}\right) k_i+\log \left(\frac{\Gamma\left(k_j\right) \theta_j^{k_j}}{\Gamma\left(k_i\right) \theta_i^{k_i}}\right)+\left(k_i-k_j\right)\left(\log \theta_i+\psi\left(k_i\right)\right)$ \\
\hline \begin{tabular}{l}
Multivariate \\
Gaussian
\end{tabular} & $\frac{1}{2}\left(\log \frac{\left|\Sigma_j\right|}{\left|\Sigma_i\right|}+\operatorname{tr}\left(\Sigma_j^{-1} \Sigma_i\right)\right)+\frac{1}{2}\left[\left(\boldsymbol{\mu}_i-\boldsymbol{\mu}_j\right)^{\prime} \Sigma_j^{-1}\left(\boldsymbol{\mu}_i-\boldsymbol{\mu}_j\right)-n\right]$ \\
\hline  Univariate Gaussian & $\frac{1}{2 \sigma_j^2}\left[\left(\mu_i-\mu_j\right)^2+\sigma_i^2-\sigma_j^2\right]+\log \frac{\sigma_j}{\sigma_i}$ \\
\hline
Special Bivariate Gaussian & $\frac{1}{2} \log \left(\frac{1-\rho_j^2}{1-\rho_i^2}\right) + \frac{\rho_j^2 - \rho_j \rho_i}{1-\rho_j^2}$ \\
\hline
General Gumbel & $\log \frac{\beta_j}{\beta_i} + \gamma\left(\frac{\beta_i}{\beta_j} - 1\right) + e^{\left(\mu_j - \mu_i\right) / \beta_j} \Gamma\left(\frac{\beta_i}{\beta_j} + 1\right) - 1$ \\
\hline
Half-Normal & $\log \left(\frac{\sigma_j}{\sigma_i}\right) + \frac{\sigma_i^2 - \sigma_j^2}{2 \sigma_j^2}$ \\
\hline
Inverse Gaussian & $\frac{1}{2}\left(\frac{\lambda_j}{\lambda_i} + \log \left(\frac{\lambda_i}{\lambda_j}\right) + \frac{\lambda_j\left(\mu_i - \mu_j\right)^2}{\mu_i \mu_j^2} - 1\right)$ \\
\hline
Laplace & $\log \frac{\lambda_j}{\lambda_i} + \frac{\left|\theta_i - \theta_j\right|}{\lambda_j} + \frac{\lambda_i}{\lambda_j} \exp \left(-\frac{\left|\theta_i - \theta_j\right|}{\lambda_i}\right) - 1$ \\
\hline
\begin{tabular}{l}
Lévy \\ Equal \\ Supports \\ ($\mu_i = \mu_j$)
\end{tabular} & $\frac{1}{2} \log \left(\frac{c_i}{c_j}\right) + \frac{c_j - c_i}{2 c_i}$ \\
\hline
Log-normal & $\frac{1}{2 \sigma_j^2}\left[\left(\mu_i - \mu_j\right)^2 + \sigma_i^2 - \sigma_j^2\right] + \log \frac{\sigma_j}{\sigma_i}$ \\
\hline
Maxwell Boltzmann & $3 \log \left(\frac{\sigma_j}{\sigma_i}\right) + \frac{3\left(\sigma_i^2 - \sigma_j^2\right)}{2 \sigma_j^2}$ \\
\hline
Pareto & $\log \left(\frac{m_i}{m_j}\right)^{a_j} + \log \frac{a_i}{a_j} + \frac{a_j - a_i}{a_i}$, for $m_i \geq m_j$ and $\infty$ otherwise. \\
\hline
Rayleigh & $2 \log \left(\frac{\sigma_j}{\sigma_i}\right) + \frac{\sigma_i^2 - \sigma_j^2}{\sigma_j^2}$ \\
\hline
Uniform & $\log \frac{b_j - a_j}{b_i - a_i}$ for $\left(a_i, b_i\right) \subseteq \left(a_j, b_j\right)$ and $\infty$ otherwise. \\
\hline
\begin{tabular}{l}
General \\ Weibull
\end{tabular} & $\log \left(\frac{k_i}{k_j}\left[\frac{\lambda_j}{\lambda_i}\right]^{k_j}\right) + \gamma \frac{k_j - k_i}{k_i} + \left(\frac{\lambda_i}{\lambda_j}\right)^{k_j} \Gamma\left(1 + \frac{k_j}{k_i}\right) - 1$ \\
\hline
\end{tabular}
    \caption{KL Divergences}
    \label{fig:divergenze}
\end{figure}
\end{restatable}
\begin{restatable}[Theorem $4.2.3$, Example $4.2.4$ \citet{roch2024modern}]{lemma}{bborder}\label{lem:bborder}
   Let $F_{n,p}$ be the CDF of a $Bin(n,p)$ distributed Random Variable, then holds for $m\leq n$ and $q\leq p$:
\begin{align}
   F_{n,p}(x)\leq F_{m,q}(x)
\end{align}
for all $x$.
\end{restatable}
\begin{restatable}[Beta and Normal Ordering]{lemma}{nbord}\label{lem:nbord}
\phantom{A}
\begin{enumerate}[leftmargin=*]
    \item A $\mathcal{N}\left(m, \sigma^2\right)$ distributed r.v. ~(i.e., a Gaussian random variable with mean $m$ and variance $\sigma^2$ ) is stochastically dominated by $\mathcal{N}\left(m^{\prime}, \sigma^2\right)$ distributed r.v.~if $m^{\prime} \geq m$.
    \item A $Beta(\alpha, \beta)$ random variable
is stochastically dominated by $Beta(\alpha', \beta')$ if $\alpha'\ge \alpha$ and $\beta'\leq \beta$.
\end{enumerate} 
\end{restatable}
These lemmas are informally stated in \citet{agrawal2017near}, and though intuitive we had difficulties to find formal proofs, so we provide one here.
\begin{proof}
    Let us consider the ratio of the pdfs of the normal distributed Random Variables:
    \begin{align}
        \frac{f_{m',\sigma}(x)}{f_{m,\sigma}(x)}=\exp{\left(\frac{-(x-m')^2+(x-m)^2}{2\sigma^2}\right)}=\exp{\left(\frac{2x(m'-m)-m'^2+m^2}{2\sigma^2}\right)},
    \end{align}
   that is increasing in $x\in\mathbb{R}$. Similarly let us consider the ratio of the pdfs of two beta distributed random variables:
   \begin{align}
       \frac{f_{\alpha',\beta'}(x)}{f_{\alpha,\beta}(x)}=c_{\alpha,\alpha',\beta,\beta'}x^{(\alpha'-\alpha)}(1-x)^{\beta'-\beta},
   \end{align}
   that is increasing in $x\in (0,1)$, as $c_{\alpha,\alpha',\beta,\beta'}$ is a constant independent from $x$. So for $x_0\leq x_1$ for the ratios hold:
   \begin{align}
   \begin{cases}
       \frac{f_{m',\sigma}(x_0)}{f_{m,\sigma}(x_0)}\leq\frac{f_{m',\sigma}(x_1)}{f_{m,\sigma}(x_1)}\\
       \\
      f_{m',\sigma}(x_0) f_{m,\sigma}(x_1)\leq f_{m,\sigma}(x_0)f_{m',\sigma}(x_1)
   \end{cases}  ,
   \end{align}
   and
   \begin{align}
   \begin{cases}
      \frac{f_{\alpha',\beta'}(x_0)}{f_{\alpha,\beta}(x_0)}\leq\frac{f_{\alpha',\beta'}(x_1)}{f_{\alpha,\beta}(x_1)}.\\
      \\
      f_{\alpha',\beta'}(x_0)f_{\alpha,\beta}(x_1)\leq f_{\alpha,\beta}(x_0)f_{\alpha',\beta'}(x_1)
   \end{cases}   .
   \end{align}
   Let us first tackle the normal distribute random variables.
   Integrating with respect to $x_0$ up to $x_1$ we obtain:
   \begin{align}
   f_{m,\sigma}(x_1)\int_{-\infty}^{x_1}f_{m',\sigma}(x_0) dx_{0} &\leq f_{m',\sigma}(x_1)\int_{-\infty}^{x_1}f_{m,\sigma}(x_0)dx_0\\
   f_{m,\sigma}(x_1)F_{m',\sigma}(x_1)&\leq f_{m',\sigma}(x_1)F_{m,\sigma}(x_1),
   \end{align}
   so for any $x\in \mathbb{R}$ it holds:
   \begin{align}
       \frac{F_{m',\sigma}(x)}{F_{m,\sigma}(x)}\leq \frac{f_{m',\sigma}(x)}{f_{m,\sigma}(x)}.
   \end{align}
   Integrating with respect to $x_1$ up to $x_0$ we obtain:
   \begin{align}
       f_{m',\sigma}(x_0) \int_{x_0}^{+\infty} f_{m,\sigma}(x_1)dx_1&\leq f_{m,\sigma}(x_0) \int_{x_0}^{+\infty}f_{m',\sigma}(x_1) dx_1\\
        f_{m',\sigma}(x_0)(1-F_{m,\sigma}(x_0))&\leq f_{m,\sigma}(x_0)(1-F_{m',\sigma}(x_0))
   \end{align}
    so for any $x\in \mathbb{R}$ it holds:
    \begin{align}
        \frac{ f_{m',\sigma}(x)}{f_{m,\sigma}(x)}\leq \frac{1-F_{m',\sigma}(x)}{1-F_{m,\sigma}(x)}.
    \end{align}
    Putting together the two inequalities yields to:
    \begin{align}
        \frac{F_{m',\sigma}(x)}{F_{m,\sigma}(x)}\leq \frac{1-F_{m',\sigma}(x)}{1-F_{m,\sigma}(x)} \Rightarrow F_{m,\sigma}(x)\ge F_{m',\sigma}(x),\hspace{0.2 cm}\forall x
    \end{align}
    Now we tackle Beta distributd rvs. Integrating with respect to $x_0$ up to $x_1$ we obtain:
   \begin{align}
   f_{\alpha,\beta}(x_1)\int_{0}^{x_1}f_{\alpha',\beta'}(x_0) dx_{0} &\leq f_{\alpha',\beta'}(x_1)\int_{0}^{x_1}f_{m,\sigma}(x_0)dx_0\\
   f_{\alpha,\beta}(x_1)F_{\alpha',\beta'}(x_1)&\leq f_{\alpha',\beta'}(x_1)F_{\alpha,\beta}(x_1),
   \end{align}
   so for any $x\in \mathbb{R}$ it holds:
   \begin{align}
       \frac{F_{\alpha',\beta'}(x)}{F_{\alpha,\beta}(x)}\leq \frac{f_{\alpha',\beta'}(x)}{f_{\alpha,\beta}(x)}.
   \end{align}
   Integrating with respect to $x_1$ up to $x_0$ we obtain:
   \begin{align}
       f_{\alpha',\beta'}(x_0) \int_{x_0}^{1} f_{\alpha,\beta}(x_1)dx_1&\leq f_{\alpha,\beta}(x_0) \int_{x_0}^{1}f_{\alpha',\beta'}(x_1) dx_1\\
        f_{\alpha',\beta'}(x_0)(1-F_{\alpha,\beta}(x_0))&\leq f_{\alpha,\beta}(x_0)(1-F_{\alpha',\beta'}(x_0))
   \end{align}
    so for any $x\in \mathbb{R}$ it holds:
    \begin{align}
        \frac{ f_{\alpha',\beta'}(x)}{f_{\alpha,\beta}(x)}\leq \frac{1-F_{\alpha',\beta'}(x)}{1-F_{\alpha,\beta}(x)}.
    \end{align}
    Putting together the two inequalities yields to:
    \begin{align}
        \frac{F_{\alpha',\beta'}(x)}{F_{\alpha,\beta}(x)}\leq \frac{1-F_{\alpha',\beta'}(x)}{1-F_{\alpha,\beta}(x)} \Rightarrow F_{\alpha,\beta}(x)\ge F_{\alpha',\beta'}(x),\hspace{0.2 cm}\forall x.
    \end{align}
    Thus obtaining the statements.
\end{proof}

\clearpage
\section{Why Enforcing Exploration is Generally Beneficial?}\label{apx: explore}
Although the optimal value for the exploration parameter for each instance (i.e., $\sigma_{\bm{\mu}}(T)$) depends on the instance itself\footnote{We highlight that this is common in the forced exploration literature (see, \citet{garivier2016etc}, Chapter 6 of \citet{lattimore2020bandit}), and is not specific of our algorithms.} and the quantities needed to compute it might not be accessible to the learner, the result from Theorem \ref{thm:lb} and Corollary \ref{thm:noregret} , makes clear why, even without the knowledge of these quantities, enforcing exploration is beneficial. In fact, although forced exploration may add a source of regret for some instances, it will ensure enough exploration to be able to have tighter guarantees for a wider range of problems. For example, setting an arbitrary forced exploration parameter $\Gamma=\Lambda$ would guarantee enough exploration for the algorithms to properly face all the instances from $\mathcal{M}_{\Lambda}$. Furthermore, these results also provide a useful criterion for setting the exploration parameter $\Gamma$, i.e., whenever the policymaker thinks the class of problems is hard to learn should set $\Gamma$ high, to ensure that the algorithms can learn a wider range of instances with guarantees on the regret, viceversa  whenever the policymaker expects to face simple problems should set $\Gamma$ low. In Appendix \ref{sec:SWapproach}, we argue that whenever the policy-maker is worried about the underestimation of the exploration parameter $\Gamma$ can employ a sliding-window approach. 
\clearpage
\section{Numerical Simulations Parameters and Reproducibility Details}\label{app:ex}
\subsection{Parameters}
The choices of the parameters are those suggested by the authors:
\begin{itemize}
    \item Rexp3: $V_T=K$ as we've considered bounded rewards within zero and the maximum global variation possible is equal to the number of arms of the bandit; $\gamma=\min \left\{1, \sqrt{\frac{K \log K}{(e-1) \Delta_T}}\right\}, \Delta_T=$ $\left\lceil(K \log K)^{1 / 3}\left(T / V_T\right)^{2 / 3}\right\rceil$  \cite{besbes2014stochastic};
    \item KL-UCB: $c=3$ as required by the theoretical results on the regret provided by \cite{garivier2011kl};
    \item Ser4: according to what suggested by \cite{allesiardo2017nonstationary} we selected $\delta=1 / T, \epsilon=\frac{1}{K T}$, and $\phi=\sqrt{\frac{N}{T K \log (K T)}}$;
    \item SW-UCB: as suggested by \cite{garivier2011kl} we selected the sliding-window $\tau=4 \sqrt{T \log T}$ and the constant $\xi=0.6$;
    \item SW-KL-UCB as suggested by \citet{garivier2011upper} we selected the sliding-window $\tau=\sigma^{-4 / 5}$;
    \item SW-TS: as suggested by \cite{trovo2020sliding} for the smoothly changing environment we set $\beta=1 / 2$ and sliding-window $\tau=T^{1-\beta}=\sqrt{T}$. Even though the paper handle just Bernoulli rewards, we will choose the same order for the window length also for \texttt{$\gamma$-SWGTS}.
    \item R-ed-UCB: the window is set as $h_{i, t}=\left\lfloor\epsilon N_{i, t-1}\right\rfloor$ as suggested by the authors \cite{metelli2022stochastic}, $\epsilon \in (0,\frac{1}{2})$, being $N_{i, t-1}$ the numbers of plays of the $i-th$ arm up to time $t$.
\end{itemize}

\subsection{Environment}
To evaluate the algorithms in the rested setting with $K=15$ arms over a time horizon of $T=100,000$ rounds. The payoff functions $\mu_i(\cdot)$ have been chosen in these families:
\begin{align}
 & F_{\exp } = \left\{f(n) \ | \ f(n)=c\left(1-e^{-a n}\right)\right\},\\
 & F_{\text {poly }} = \left\{f(n) \ | \ f(n)=c\left(1-b\left(n+b^{1 / \rho}\right)^{-\rho}\right)\right\},
\end{align}
where $a, c, \rho \in(0,1]$ and $b \in$ $\mathbb{R}_{\geqslant 0}$ are parameters, whose values have been selected randomly. The complete settings and function selection method, in compliance with what has been presented by~\cite{metelli2022stochastic}, have been provided in the attached code.

\subsection{Experimental Infrastructure} \label{app:reproducibility}

In this section, we provide additional information for the full reproducibility of the experiments provided in the main paper.

The code has been run on an AMD Ryzen $7$ $4800$H CPU with $16$ GiB of system memory.
The operating system was Windows $11$, and the experiments have been run on Python $3.8$.
The libraries used in the experiments, with the corresponding versions, were:
\begin{itemize}
	\item \texttt{matplotlib==3.3.4}
	\item \texttt{tikzplotlib==0.10.1}
	\item \texttt{numpy==1.20.1}
\end{itemize}

On this architecture, the average execution time of each algorithm takes an average of $\approx 20-40$ sec both in the synthetic environment for a time horizon of $T = 100,000$ and int the IMDB environment for $T=50000$. 

\subsection{15-arms Numerical Simulation Results}\label{apx:comparison222}
The results of the numerical simulation presented in Section~\ref{sec:Experiments} are reported in Figure~\ref{fig:15armsfull}. In this case we have performed $20$ runs, in the semi-transparent areas are reported the standard deviations. The results show how the methods that have been designed for the restless case perform worse than the one we presented in our paper. The only exception is the \texttt{Beta-SWTS} that we showed to also have nice theoretical properties in the SRRB setting. Overall, the comparison with such methods do not invalidate the conclusions we drew in the main paper.

\begin{figure}[th!]
    \centering
    \includegraphics{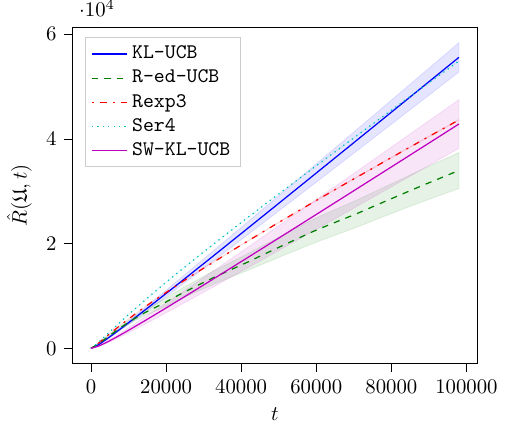}
    \caption{Regret in the $15$-arm environment.}
    \label{fig:15armsfull}
\end{figure}

\subsection{IMDB Experiment}\label{apx:imdb}
SRRBs are a powerful tool to model a lot of real-world challenges. In particular, we focus on concave rising rested bandits that are suitable to describe \emph{recommendation systems}, in which you should be able to model satiation effects \cite{clerici2023linear,xu2023online}. In particular, in this scenario, bandit algorithms serve as meta-algorithms. In fact, we will consider the learning algorithms (such as, logistic regression, neural networks and OGD) as different arms, and the expected reward will be based on the score given by IMDB. We refer the reader to \cite[][Appendix E.2]{metelli2022stochastic} for the details of the setting.
\begin{figure}[t!]
\centering
\includegraphics{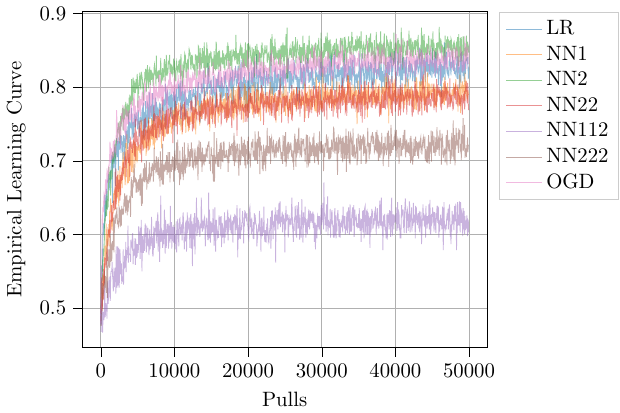}
  \captionof{figure}{Empirical Learning Curve for each algorithm.}
  \label{fig:lc}
\end{figure}
Specifically, we consider a constant learning rate for Logistic Regression ($\lambda_t=1$). Moreover, the NNs use as activation functions the rectified linear unit, i.e., $\text{relu}(x)=\max\{0,x\}$, a constant learning rate $\alpha=0.001$ and the Adam stochastic gradient optimizer for fitting. Two of the chosen nets have only one hidden layer, with 1 and 2 neurons, respectively, the third net has 2 hidden layer, with 2 neurons each, and two nets have 3 layers with 2,2,2 and 1,1,2 neurons, respectively. We refer to a specific NN denoting  the cardinalities of the layers, e.g., the one having 2 layer with 2 neurons each is denoted by NN22. We analyzed their global performance on the IMDB dataset by averaging 1,000 independent runs in which each strategy is
sequentially fed with all the available 50,000 samples.
The learning curves are depicted in Figure \ref{fig:lc}. The experiments will be organized along this line:
\begin{figure}[t!]
\centering
\begin{minipage}{1\textwidth}
  \centering
  \includegraphics{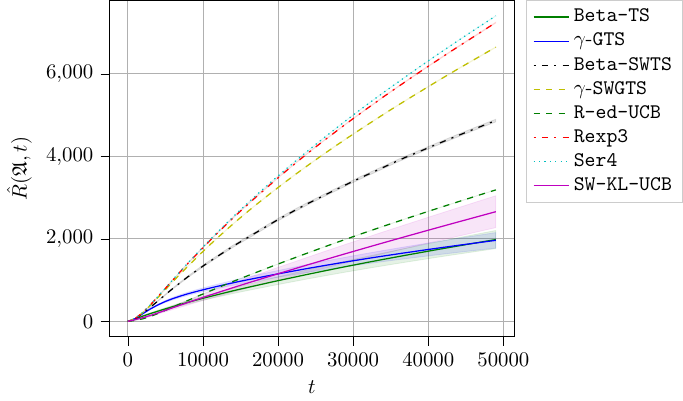}
  \captionof{figure}{Average Cumulative Regret for the first environment (original IMDB experiment).}
  \label{fig:imdb1}
\end{minipage}%
\\
\begin{minipage}{1\textwidth}
  \centering
  \includegraphics{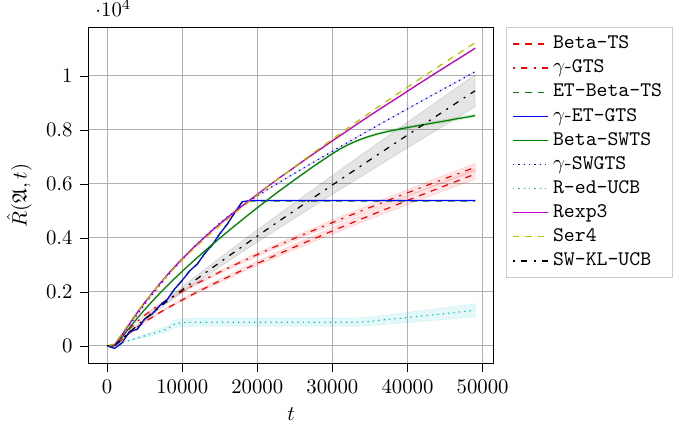}
  \captionof{figure}{Average Cumulative Regret for the second environment (IMDB experiment with additional arms).}
  \label{fig:imdb2}
\end{minipage}
\end{figure}
first, we run all the algorithms against we are competing comparing them with the results of conventional TS algorihms, with $T$ time horizon set to 50000. The experiment shows that the algorithms devised for the standard bandit are competitive, meaning that the environment is not too challenging (coherent with our analysis, i.e., $\sigma$ is low). For the second experiment, we add two artificial arms (the details on the artificial arm are provided in the code) to make the environment more challenging for the algorithms. We provide the average cumulative regret and the standard deviation in the semi-transparent areas.
\paragraph{Experiment 1, Figure \ref{fig:imdb1}}
The experiment shows that, as we have discussed in the paper, conventional TS algorithms can be still effective whenever the complexity induced by the term 
$\sigma$ is low enough. In fact \texttt{Beta-TS} and standard $\gamma$-\texttt{GTS} outperform all the algorithms, even those designed explicitly for Rising Rested bandits, i.e., \texttt{R-ed-UCB}. Compared to the results presented in the main paper, here we report additional baselines.

\paragraph{Experiment 2, Figure \ref{fig:imdb2}}
We add two arms whose evolution challenges the effectiveness of standard bandit algorithms (details are provided in the attached code), i.e., such that a large number of pulls is needed to inferentially estimate the best arm from the other. One of these new two arms is the optimal one. In this experiment we set the forced exploration again equal to
$2000$, and the sliding window set as suggested in Appendix \ref{app:ex}. This experiment shows that the conventional algorithms start to fail, as well as the algorithms devised to face restless settings. Conversely, \texttt{R-ed-UCB} and the forced exploration TS-like algorithms are able to consistently select the best arm. Differently to what happens for the synthetic setting, \texttt{R-ed-UCB} is the best performing algorithm (note that  the presence of the additional arms makes the learning problem different compared to the original IMDB setting), however, the TS-like algorithms with forced exploration, shows a more pronounced flattening of the curve, so that for large enough time horizons $T$ we expect them to cath up.

\clearpage
\section{About the Computational Complexity}\label{apx:cc}
There are three main approaches to handle dynamic environments.

\paragraph{Discounting} (\citet{qi2023discounted,garivier2011upper}): giving different weights to the gathered examples, so that some samples can be more important for the learner. This class of methods need to re-compute the average values at every turn for every arm, by updating the weight of all the samples collected during the learning, so that the complexity would scale as $O(t)$. However, the most common and exploited version of the discounted methods can reach $O(1)$ complexity per round per arm in the update phase, as it updates the mean value by a multiplication, so that the overall complexity per round is $O(K)$, being $K$ the number of arms. 

\paragraph{Change detection methods} (\citet{liu2018change,besson2019generalized}): that at every turn $t$ run a ratio test routine to infer if the distribution is changed, operation that is computationally expensive as it is requested most of the times to compute threshold values that need to be evaluated for every possible subset of samples gathered within $[s,t]$ with $s\in[1,t]$ (see for example Definition 1 of \citet{besson2019generalized}), so that a time complexity of $O(t)$ per round is unavoidable.

\paragraph{Sliding Window approaches} (\citet{trovo2020sliding}): are those approaches in which we need to remove the effect of the oldest observation. This requires mantaining a deque to store rewards per arm, leading to efficient removal in $O(1)$ per arm, and therefore to an overall time complexity per round of order $O(K)$, as the sampling itself for the Thompson sample is also $O(1)$ per arm and finding the maximum is $O(K)$.  We highlight that is the same time complexity per round of the standard Thompson Sampling used for stationary setting.

Thus, Sliding-Window approaches for Thompson Sampling are one of the most competitive approaches in literature even for what regard time complexity. 




\end{document}